\documentclass{article}  
   
\usepackage{microtype}
\usepackage{graphicx}
\usepackage{subfigure}
\usepackage{booktabs} % for professional tables

\usepackage{hyperref}

\usepackage[accepted]{icml2020}

\usepackage[mathscr]{eucal}
\usepackage[cmex10]{amsmath}
\usepackage{epsfig,epsf,psfrag}
\usepackage{amssymb,amsmath,amsthm,amsfonts,latexsym}
\usepackage{amsmath,graphicx,bm,xcolor,overpic}
\usepackage{fixltx2e}%ordering of single and double column floats
\usepackage{array}%array and tabular environments
\usepackage{verbatim}
\usepackage{bm,bbm}
\usepackage{verbatim}
\usepackage{textcomp}
\usepackage{mathrsfs}
\usepackage{epstopdf}

\usepackage{url}
\makeatother

\usepackage{thmtools}
\usepackage{thm-restate}

\numberwithin{equation}{section}

\numberwithin{table}{section}
\numberwithin{figure}{section}
\catcode`~=11 \def\UrlSpecials{\do\~{\kern -.15em\lower .7ex\hbox{~}\kern .04em}} \catcode`~=13 

\usepackage{longtable}
\usepackage{setspace}

\newcommand{\rmKL}{\mathrm{KL}}

\newcommand{\bmw}{{\bm w}}

\newcommand{\bmO}{{\bm O}}
\newcommand{\bmu}{{\bm u}}
\newcommand{\bmv}{{\bm v}}

\newcommand{\calE}{\mathcal{E}}
\newcommand{\calF}{\mathcal{F}}

\newcommand{\calT}{\mathcal{T}}

\newcommand{\rmd}{\mathrm{d}}

\newcommand{\bbE}{\mathbb{E}}

\newcommand{\bbI}{\mathbb{I}}

\newcommand{\bbN}{\mathbb{N}}

\newcommand{\bbP}{\mathbb{P}}

\newcommand{\bbR}{\mathbb{R}}

\newcommand{\bbT}{\mathbb{T}}

\DeclareMathAlphabet{\mathbsf}{OT1}{cmss}{bx}{n}
\DeclareMathAlphabet{\mathssf}{OT1}{cmss}{m}{sl}% slanted sans serif

\DeclareSymbolFont{bsfletters}{OT1}{cmss}{bx}{n}  
\DeclareSymbolFont{ssfletters}{OT1}{cmss}{m}{n}
\DeclareMathSymbol{\bsfGamma}{0}{bsfletters}{'000}
\DeclareMathSymbol{\ssfGamma}{0}{ssfletters}{'000}
\DeclareMathSymbol{\bsfDelta}{0}{bsfletters}{'001}
\DeclareMathSymbol{\ssfDelta}{0}{ssfletters}{'001}
\DeclareMathSymbol{\bsfTheta}{0}{bsfletters}{'002}
\DeclareMathSymbol{\ssfTheta}{0}{ssfletters}{'002}
\DeclareMathSymbol{\bsfLambda}{0}{bsfletters}{'003}
\DeclareMathSymbol{\ssfLambda}{0}{ssfletters}{'003}
\DeclareMathSymbol{\bsfXi}{0}{bsfletters}{'004}
\DeclareMathSymbol{\ssfXi}{0}{ssfletters}{'004}
\DeclareMathSymbol{\bsfPi}{0}{bsfletters}{'005}
\DeclareMathSymbol{\ssfPi}{0}{ssfletters}{'005}
\DeclareMathSymbol{\bsfSigma}{0}{bsfletters}{'006}
\DeclareMathSymbol{\ssfSigma}{0}{ssfletters}{'006}
\DeclareMathSymbol{\bsfUpsilon}{0}{bsfletters}{'007}
\DeclareMathSymbol{\ssfUpsilon}{0}{ssfletters}{'007}
\DeclareMathSymbol{\bsfPhi}{0}{bsfletters}{'010}
\DeclareMathSymbol{\ssfPhi}{0}{ssfletters}{'010}
\DeclareMathSymbol{\bsfPsi}{0}{bsfletters}{'011}
\DeclareMathSymbol{\ssfPsi}{0}{ssfletters}{'011}
\DeclareMathSymbol{\bsfOmega}{0}{bsfletters}{'012}
\DeclareMathSymbol{\ssfOmega}{0}{ssfletters}{'012}

\newcommand{\hati}{\hat{i}}

\newcommand{\hatk}{\hat{k}}
\newcommand{\hatK}{\hat{K}}

\newcommand{\hatw}{\hat{w}}

\newcommand{\barw}{\bar{w}}

\newcommand{\barT}{\bar{T}}

\DeclareMathOperator*{\argmax}{arg\,max}

\newtheorem{theorem}{Theorem}[section]
\newtheorem{lemma}[theorem]{Lemma}

\newtheorem{corollary}[theorem]{Corollary}

\newtheorem{remark}[theorem]{Remark}

\newtheorem{assumption}[theorem]{Assumption}

\usepackage[skip=0.1in]{caption}
\allowdisplaybreaks

\icmltitlerunning{Best Arm Identification for Cascading Bandits in the Fixed Confidence Setting}

\begin{document}

\twocolumn[
\icmltitle{Best Arm Identification for Cascading Bandits in the Fixed Confidence Setting}

% It is OKAY to include author information, even for blind
% submissions: the style file will automatically remove it for you
% unless you've provided the [accepted] option to the icml2020
% package.

% List of affiliations: The first argument should be a (short)
% identifier you will use later to specify author affiliations
% Academic affiliations should list Department, University, City, Region, Country
% Industry affiliations should list Company, City, Region, Country

% You can specify symbols, otherwise they are numbered in order.
% Ideally, you should not use this facility. Affiliations will be numbered
% in order of appearance and this is the preferred way.
\icmlsetsymbol{equal}{*}

%\begin{icmlauthorlist}
%\icmlauthor{Anonymous Authors}{institute}
%\end{icmlauthorlist}
%
%
%\icmlaffiliation{institute}{Anonymous Institution, Anonymous City, Anonymous Region, Anonymous Country}
%\icmlcorrespondingauthor{Anonymous Author}{anon.email@domain.com}

\begin{icmlauthorlist}
\icmlauthor{Zixin Zhong}{math}     %{nus}%
\icmlauthor{Wang Chi Cheung}{ise,iora}  %{nus}%
\icmlauthor{Vincent Y.~F.~Tan}{math,iora,ece}       %{nus}%
\end{icmlauthorlist}

\icmlaffiliation{math}{Department of Mathematics, National University of Singapore, Singapore}
\icmlaffiliation{ise}{Department of Industrial Systems and Management, National University of Singapore, Singapore}
\icmlaffiliation{iora}{Institute of Operations Research and Analytics, National University of Singapore, Singapore}
\icmlaffiliation{ece}{Department of Electrical and Computer Engineering, National University of Singapore, Singapore}
%\icmlaffiliation{nus}{National University of Singapore}
\icmlcorrespondingauthor{Zixin Zhong}{zixin.zhong@u.nus.edu}
\icmlcorrespondingauthor{Wang Chi Cheung}{isecwc@nus.edu.sg}
\icmlcorrespondingauthor{Vincent Y.~F.~Tan}{vtan@nus.edu.sg}

% You may provide any keywords that you
% find helpful for describing your paper; these are used to populate
% the "keywords" metadata in the PDF but will not be shown in the document
\icmlkeywords{Machine Learning, ICML}

\vskip 0.3in
]

% this must go after the closing bracket ] following \twocolumn[ ...

% This command actually creates the footnote in the first column
% listing the affiliations and the copyright notice.
% The command takes one argument, which is text to display at the start of the footnote.
% The \icmlEqualContribution command is standard text for equal contribution.
% Remove it (just {}) if you do not need this facility.

\printAffiliationsAndNotice{}  % leave blank if no need to mention equal contribution
%\printAffiliationsAndNotice{\icmlEqualContribution} % otherwise use the standard text.

\begin{abstract}
% 1
We design and analyze {\sc CascadeBAI}, an algorithm for finding the best set of $K$ items, also called an arm, within the framework of cascading bandits.
% 2
An upper bound on the time complexity of {\sc CascadeBAI} is derived by overcoming a crucial analytical challenge, namely, that of probabilistically estimating the amount of available feedback at each step.
% 3
To do so, we define a new class of random variables (r.v.'s) which we term as left-sided sub-Gaussian r.v.'s; these are r.v.'s whose cumulant generating functions~(CGFs) can be bounded by a quadratic only for non-positive arguments of the CGFs.
% 4
This enables the application of a sufficiently tight Bernstein-type concentration inequality.
% 5
We show, through the derivation of a lower bound on the time complexity, that the performance of {\sc CascadeBAI} is optimal in some practical regimes.
% 6
Finally, extensive numerical simulations corroborate the efficacy of {\sc CascadeBAI} as well as the tightness of our upper bound on its time complexity.
\end{abstract}

\section{Introduction} 
Online recommender systems seek to  recommend a small list of items (such as movies or hotels) to users based on a larger ground set $[L] := \{1, \ldots, L\}$ of items. In this paper, we consider the {\em cascading bandits} model~\citep{Craswell08,kveton2015cascading}, which is widely used in information retrieval and online advertising.
Upon seeing the chosen list, the user looks at the items sequentially. She {\em clicks} on an item if she is {\em attracted} by it and skips to the next one otherwise. 
This process stops when she clicks on one item in the list or 
if no item is clicked, it is deemed that she is {\em not attracted} by {\em any} of the items.
The items that are in the ground set but not in the chosen list and those in the list that come after the attractive one are {\em unobserved}.

Each item $i\in [L]$, with a certain {\em click probability} $w(i)\in [0, 1]$ which is {\em unknown} to the learning agent, attracts the user independently of other items. Under this assumption, the optimal solution is the list of items with largest $w(i)$'s. 
Based on the chosen lists and obtained feedback in previous steps, the agent tries to learn the click probabilities (explore the combinatorial space) in order to find the optimal list with high probability in as few time steps as possible.

\textbf{Main Contributions.} 
Given $\delta >0$, a learning agent aims to find a list of  optimal items of size $K$ with probability at least $1-\delta$ in minimal time steps. 
To achieve a greater generality, we provide results for identifying a list of near-optimal items~\citep{even2002pac,mannor2004sample,kalyanakrishnan2012pac}, where the notion of near-optimality is precisely defined in Section~\ref{sec:prob_setup}.
First, we design {\sc CascadeBAI($\epsilon,\delta,K$)} and derive an upper bound on its time complexity.  
Second, we establish a lower bound on the time complexity of \emph{any} best arm identification~(BAI) algorithm in cascading bandits,
which implies that the performance of {\sc CascadeBAI($\epsilon,\delta,K$)} is optimal in some regimes.
Finally, our extensive numerical results corroborate the efficacy of {\sc CascadeBAI($\epsilon,\delta,K$)} and the tightness of our upper bound on its time complexity.

Different from combinatorial semi-bandit settings, the amount of feedback in cascading bandits is, in general, random. 
The analysis of cascading bandits involves the unique challenge in adapting to the variation of the amount of feedback across time.
To this end, we define a random variable~(r.v.) that describes the feedback from the user at a step and bound its expectation.
We define a novel class of r.v.'s, known as \emph{left-sided sub-Gaussian}~(LSG) r.v.'s,
and apply a concentration inequality to  quantify the variation of the amount of feedback.

Bernstein-type concentration inequalities are applied in many stochastic bandit problems and indicate that sub-Gaussian~(SG) distributions possess light tails~\citep{audibert2010best}. 
Since it turns out that we only need to analyze a one-sided tail in this work, it suffices to consider a one-sided SG condition, which motivates the definition of LSG. We also provide a general estimate of a certain corresponding parameter in Theorem~\ref{thm:sec_moment_to_left_sub_gauss}, which is crucial for the utilization of the inequality.
This may be of independent interest. Summary and future work are deferred to Appendix~\ref{appdix:future_work}.%

\textbf{Literature review.}  
In a stochastic combinatorial bandit (SCB) model, an arm corresponds to a list of items in the ground set, and each item is associated with an r.v. at each time step. The corresponding reward depends on the constituent items' realizations.
We first review the related works on the BAI problem, in which a learning agent aims to identify an {\em optimal arm}, i.e., a list of optimal items.
(i) Given 
$\delta>0$, a learning agent aims to identify an optimal arm with probability $1-\delta$ in minimal time steps~\citep{jamieson2014best,kalyanakrishnan2012pac}. 
(ii) Given
$B>0$, an agent aims to maximize the probability of identifying an optimal arm in $B$ steps~\citep{auer2002finite,audibert2010best,carpentier2016tight}.
These two settings are known as the \emph{fixed-confidence} and
\emph{fixed-budget} setting respectively.
Under the fixed-confidence setting, 
the early works aim to identify only one optimal item~\citep{audibert2010best} and the later ones aim to find an optimal arm~\citep{NIPS2014_5433,rejwan2019combinatorial}. 
Besides, \citet{mannor2004sample,kaufmann2016complexity,pmlr-v65-agarwal17c} provide problem-dependent lower bounds on the time complexity
when \citet{kalyanakrishnan2012pac} establishes a problem-independent one.
All these existing works are under the {\em semi-bandit feedback} setting, where an agent observes realizations of all pulled items.

Secondly, we review the relevant works on the \emph{regret minimization}~(RM) problem,
in which an agent aims to maximize his overall reward, or equivalently to minimize the so-called \emph{cumulative regret}.
Under the semi-bandit feedback setting, this problem 
has been extensively studied by~\citet{lai1985asymptotically, AnantharamVW87,KvetonWAEE14,LiChuLangford10,QinCZ14}.
Moreover, motivated by numerous applications in clinical analysis and online advertisement, some researchers consider SCB models with \emph{partial feedback}, where an agent observes realizations of only a portion of pulled items.
One prime model that incorporates the partial feedback is cascading bandits~{\citep{Craswell08,kveton2015cascading}. 
Recently,  \citet{KvetonWAS15b,LiWZC16,ZongNSNWK16,wang2017improving,
cheung2019thompson} studied this model and derived various regret bounds.% in various settings.

When the RM problem is studied with both semi-bandit and partial feedback, the BAI problem has only been studied in the semi-bandit feedback setting thus far.
Despite existing works, analysis of the BAI problem in the more challenging case of partial feedback is yet to be done.
Our work fills in this gap in the literature by studying the fixed-confidence setting in cascading bandits, and our analysis provides tools for handling the statistical dependence between the amount of feedback and that of time steps in the cascading bandit setting.

\section{Problem Setup}
\label{sec:prob_setup}
	For brevity, we denote the set $\{1,\ldots,n\}$ by $[n]$ for any $n\in \bbN$, and the set of all $m$-permutations of $[n]$, i.e., all ordered $m$-subsets of $[n]$, by $[n]^{(m)}$ for any $m\le n $.
    Let there be $L\in\mathbb{N}$ ground items, contained in $[L]$. 
    Each item $i\in [L]$ is associated with a {\em weight} $w(i)\in [0, 1]$, signifying the item's click probability.
We define an {\em arm} as a list of $K\le L$ items in $[L]^{(K)} $. 
At each time step $t $, the agent pulls an arm $S_t: = (i^t_1, \ldots, i^t_K)\in [L]^{(K)} $. 
Then the user examines the items from $i_1^t$ to $i_K^t$ one at a time, until one item is clicked or all items are examined. 
For each item $i\in [L]$, $W_t(i) \sim \text{Bern}(w(i) )$ are i.i.d. across time. The agent observes $W_t(i)=1$ iff the user clicks on $i$.
    The {\em feedback} $\bmO_t$ from the user is defined as a vector in $\{0, 1, \star\}^K$, where $0, 1, \star$ represents observing no click, observing a click and no observation respectively. 
For example, if $K=4$ and the user clicks on the third item at time step $2$, we have $\bmO_2 = \{ 0,0,1,\star \}$. 
Clearly, there is a one-to-one mapping from $\bmO_t$ to the integer
{\setlength\abovedisplayskip{.5em}
    \setlength\belowdisplayskip{.4em}
 \begin{align*}  
    \tilde{k}_t := \min\{1\leq k\leq K: W_t(i^t_k) = 1 \},   
\end{align*}
where we assume $\min \emptyset = \infty$.
If $\tilde{k}_t < \infty$~(i.e., $\bmO_t $ is not the all-zero vector), the agent observes $W_t(i^t_k) = 0$ for $1\leq k < \tilde{k}_t$, and also observes $W_t(i^t_{\tilde{k}_t} ) = 1$, but does not observe $W_t(i^t_k)$ for $k> \tilde{k}_t$. Otherwise, we have $\tilde{k}_t = \infty$~(i.e., $\bmO_t $ is the all-zero vector), then the agent observes $W_t(i^t_k) = 0$ for $1\leq k \leq K$.
We denote $\barw(i) =1-w(i)$, 
    $\bmw =( w(1), \ldots, w(L) )$, and the probability law (resp. the expectation) of the process $( \{ W_t(i) \}_{i,t} )$ by $\bbP_\bmw$ (resp. $\bbE_\bmw$).

    Without loss of generality, we assume 
    $w^*: = w(1) \ge w(2) \ge \ldots \geq w(L) := w'$.
We say item $i$ is {\em optimal} if $w(i)\ge w(K)$. We assume $w(K) \! > \! w(K \! + \! 1)$ to ensure there are exactly $K$ optimal items. 
Next, we say item $i$ is {\em $\epsilon$-optimal}~($\epsilon\ge 0$) if $w(i)\ge w(K)-\epsilon$ and set $K'_\epsilon:=\max\{i\in[L]: w(i)\ge w(K)-\epsilon \}$. 
Then $[K'_\epsilon]$ is the set of all $\epsilon$-optimal items,
$[K]^{(K)}$ is the set of all {\em optimal arms} $S^*$ (up to permutation), and $[K'_\epsilon]^{(K)}$ is the set of all {\em $\epsilon$-optimal arms}.

To identify an $\epsilon$-optimal arm, an agent uses an \emph{algorithm} $\pi$ that decides which arms to pull, when to stop pulling, and which arm $\hat{S}^\pi$ to choose eventually. 
A deterministic and non-anticipatory online algorithm consists in a triple $\pi: = ( (\pi_t )_t , \calT^\pi,  \phi^\pi ) $ in which:
\\ $\bullet$ the \emph{sampling rule} $\pi_t$ determines, based on the observation history, the arm $S_t^\pi$ to pull at time step $t$; in other words, $S_t^\pi$ is $\calF_{t-1}$-measurable, with $\calF_t: = \sigma( S_1^\pi, \bmO_1^\pi, \ldots , S_t^\pi, \bmO_t^\pi  )$;
\\ $\bullet$ the \emph{stopping rule} determines the termination of the algorithm, which leads to a \emph{stopping time} $\calT^\pi$ with respect to $(\calF_t)_{t\in \bbN}$ satisfying $\bbP(\calT^\pi < +\infty) = 1$;
\\ $\bullet$ the \emph{recommendation rule} $\phi^\pi$ chooses an arm $\hat{S}^\pi$, which is $\calF_{\calT^\pi}$-measurable.% random subset of $[L]$ of size $K$.
\\
We define the \emph{time complexity} of $\pi$ as $ \calT^\pi$.
Under the fixed-confidence setting, a risk parameter~(failure probability) $\delta\in (0,1)$ is fixed. We say an algorithm $\pi $ is {\em$(\epsilon,\delta,K)$-PAC~(probably approximately correct)} if 
$\bbP_\bmw ( \hat{S}^\pi \!  \subset \!  [K'_\epsilon] ) \ge 1-\delta$. 
The goal is to obtain an $(\epsilon,\delta,K)$-PAC algorithm $\pi$ such that $\bbE_\bmw \calT^\pi$ is small and $\calT^\pi$ is small with high probability. 
We also define the {\em optimal expected time complexity} over all $(\epsilon,\delta,K)$-PAC algorithms as
\begin{align*}
	\bbT^*(\bmw,\epsilon,\delta,K): = \inf \{  \bbE_\bmw \mathcal{T}^\pi: \pi \text{ is } (\epsilon,\delta,K)\text{-PAC}  \}.
\end{align*}
This measures the hardness of the problem. 
We abbreviate $( 0 , \delta , K)$-PAC as $(\delta,K)$-PAC,
$\bbE_\bmw$ as $\bbE$, $\bbP_\bmw$ as $\bbP$, $K'_\epsilon$ as $K'$, $\calT^\pi $ as $\calT$, $\bbT^*(\bmw,\epsilon,\delta,K) $ as $\bbT^*$ when there is no ambiguity.}

\section{Algorithm}
\label{sec:alg}
\setlength{\textfloatsep}{.5cm}
    \begin{algorithm}[ht]
	  \caption{ \sc CascadeBAI($\epsilon,\delta,K$) }\label{alg:bai_cas_conf}
	    \begin{algorithmic}[1]
		  \STATE Input: risk $\delta $, tolerance $\epsilon $, size of arm $K$.
		  \STATE Initialize $t = 1, D_{1} =[L], A_{1} = \emptyset, R_{1} = \emptyset, T_{0}(i) = 0, \hat{w}_0(i) = 0,  \forall i$.
		  \WHILE{$D_t \neq \emptyset$, $|A_t|<K$ and $|R_t|<L-K$ }
		  	\STATE Sort the items in $D_t$ according to the number of previous observations: $T_{t-1}(i_1^t) \le \ldots \le T_{t-1}(i_{|D_t|}^t)$.%
		  	\IF{$|D_t| \ge K$} 
		  		\STATE {pull arm $S_t = (i^t_1, i^t_2, \ldots, i^t_K)$. } 
		  	  \ELSE 
		  	    \STATE{pull arm $S_t = (i^t_1, i^t_2, \ldots, i^t_{|D_t|}, S'_t)$, where $S'_t$ is any $(K-|D_t|)$-subset of $A_t \bigcup R_t$. } 
		  	\ENDIF
			\STATE Observe click $\tilde{k}_t \in \{ 1, \ldots,K, \infty \}$.
			\STATE For each $i\in D_t$, if $W_t(i)$ is observed, set\hphantom{a a a a a a } 
				$\hphantom{a}
				 \hat{w}_{t}(i) = \frac{T_{t-1}(i)\hat{w}_{t-1}(i) + W_t(i)}{T_{t-1}(i) + 1},\
				  T_{t}(i) = T_{t-1}(i)  +1$. \hspace{3em}
				Otherwise, $\hat{w}_{t}(i) = \hat{w}_{t-1}(i), T_{t}(i) = T_{t-1}(i)$.
		  	\STATE $k_t = K - |A_t|$.
		  	\STATE Calculate UCBs and LCBs for each $i   \in   D_t$: \\
		  	$ \hspace{3em} U_t(i, \delta) = \hat{w}_t(i) + C_t( i, \delta )  ,$\\
		  	$ \hspace{3em} L_t(i, \delta) = \hat{w}_t(i) - C_t( i, \delta )  $. \vspace{.3em}		  	
		  	\STATE $j^* =   \argmax _{j \in D_{t}}^{ (k_{t}+1 )} \hatw_{t} , j' = \argmax _{j \in D_{t}}^{ (k_{t} )} \hatw_{t}  $. \vspace{.3em}
		  	\STATE $A_{t+1} = A_{t} \cup \{i \in D_{t} ~|~ L_{t}(i, \delta)>U_{t}(j^*, \delta) - \epsilon  \}$. \vspace{.3em}
		  	\STATE $R_{t+1} = R_{t} \cup \{i \in D_{t} ~|~ U_{t}(i, \delta)< L_{t}(j', \delta) - \epsilon \}$. \vspace{.3em}
		  	\STATE $ D_{t+1} = D_{t} / (R_{t+1} \bigcup A_{t+1} )$.
		  	\STATE $t=t+1$. 
		  \ENDWHILE
		  \STATE If $|A_t|=K$, output $A_t$; otherwise, output the first $K$ items that entered $A_t$.
		\end{algorithmic}
	\end{algorithm}

Our algorithm {\sc CascadeBAI($\epsilon,\delta,K$)} is presented in Algorithm~\ref{alg:bai_cas_conf}. 
Intuitively, to identify an $\epsilon$-optimal arm, an agent needs to learn the true weights $w(i)$ of a number of items in $[L]$ by exploring the combinatorial arm space.

At each step $t$, we classify an item as {\em surviving}, {\em accepted} or {\em rejected}.
Initially, all items are surviving and belong to the {\em survival set} $D_t$. %
Over time, an item may be eliminated from  $D_t$, in which case we say that it is {\em identified}. 
Once an item is identified, it can be moved to either the {\em accept set} $A_t$ if it is deemed to be $\epsilon$-optimal, or the {\em reject set} $R_t$ otherwise. %
(i) At step $1$, all items are in $D_1$. 
(ii) At each step $t$, the agent selects $\min\{K,|D_t|\}$ surviving items with the least number of previous observations, $T_t(i)$'s, pulls them in {\em ascending} order of the $T_t(i)$, and gets cascading feedback from the user in the form of the $\tilde{k}_t$'s.
Similarly to a Racing algorithm~\citep{even2002pac,maron1994hoeffding,heidrich2009hoeffding,jun2016top},
 this design of $S_t$ increases the $T_t(i)$'s of all surviving items almost uniformly and avoids the wastage of time steps. % 
(iii) Next, 
we maintain upper and lower confidence bounds~(UCB, LCB) across time to facilitate the identification of items as in Lines 13--17.}
The confidence radius is defined as follows:%
{\setlength\abovedisplayskip{.5em}
    \setlength\belowdisplayskip{.4em}
\begin{align} %\label{eq:def_conv_rad}
	& C_t(i, \delta): = 4\sqrt{ \frac{ \log  ( \log _{2}[2 T_t(i) ] / \rho(\delta) ) }{ T_t(i) } } ,\ 
	\rho(\delta): = \sqrt{ \frac{\delta}{12L} }. \nonumber %\sqrt{\delta/(12L)}. \nonumber
\end{align}
We set $C_t(i, \delta)=+\infty$ when $T_t(i)=0$.
(iv) Lastly, the algorithm stops once 
$D_t=\emptyset$, $|A_t|\ge K$ or $|R_t|\ge L-K$.

\section{Main results}
We develop an upper bound on the time complexity of {\sc CascadeBAI($\epsilon,\delta,K$)} and a lower bound on the expected time complexity of any $(\delta,K)$-PAC algorithm. We also discuss the gap between the bounds.
We use $c_1,c_2,\ldots$ to denote finite and positive universal constants whose values may vary from line to line.
The proofs are sketched in Section~\ref{sec:proof_sketch} and more details are provided in Appendix~\ref{appdix:pf_main_results}.

\subsection{Upper bound}
\label{sec:main_result_up}
The gaps between the click probabilities determine the hardness to identify the items. The gaps are defined as
{\setlength\abovedisplayskip{.5em}
    \setlength\belowdisplayskip{.4em}
\begin{align*}
 	&
 	{\Delta}_i: = \left\{
               \begin{array}{ll}
                    w(i)-w(K+1)                       & 1 \le i \le K \\
                    w(K) - w(i)                          &  K < i \le L 
               \end{array}
               \right.
    ,\\ &
               \bar{\Delta}_i: = \left\{
               \begin{array}{ll}
                    \Delta_i + \epsilon                        & 1 \le i \le K \\
                    \Delta_K - \Delta_i + \epsilon    & K < i \le K' \\
                    \Delta_i - \epsilon                         &  K' <i\le L 
               \end{array}
               \right. .
\end{align*}
Here $\bar{\Delta}_i$ is a slight variation of ${\Delta}_i$ that takes into account the $\epsilon$-optimality of items. 
Moreover, 
to correctly identify item $i$ with probability at least $1-\delta/2$, our algorithm needs to observe it at least
\begin{align*}
	\barT_{i,\delta}
	&  := 
	1 + \Bigg \lfloor \frac{216}{ \vphantom{\hat{\Delta}} \bar{\Delta}_i^2 }  \log \Bigg( \frac{2}{ \rho(\delta) } \log_2 \bigg( \frac{ 648 }{ \vphantom{\hat{\Delta}} \rho(\delta) \bar{\Delta}_i^2 } \bigg) \Bigg)  \Bigg\rfloor
	\\ &
	=
	 O \bigg( \bar{\Delta}_{i}^{-2} \log \bigg[ \frac{ L}{\delta } \log \bigg( \frac{ L }{ \vphantom{\hat{\Delta}} \delta \bar{\Delta}_{ i }^{2} } \bigg) \bigg] 
	\bigg)
\end{align*}
times. 
Similarly to existing works~\citep{even2002pac,mannor2004sample}, we derive the upper bound with $\bar{\Delta}_i$'s and $\bar{T}_{i,\delta}$'s.
A larger $\bar{\Delta}_{i}$ leads to a smaller $\barT_{i,\delta}$, implying that it requires fewer observations to identify item $i$ correctly.}
The permutation $\sigma$ defines the ordering of $\bar{\Delta}$:
$
	\bar{\Delta}_{\sigma(1)} \ge \ldots \ge  \bar{\Delta}_{\sigma(L)} %\ge  \bar{\Delta}_{\sigma(2)}
$.
At step $t$, we set $\hat{k}_t$ as the number of surviving items in $S_t$,  
and $X_{\hatk_t;t}$ as the number of observations of them.
Note that $\hat{k}_t$ is an r.v.
We lower bound $\bbE X_{k;t}$ with 
\vspace{-.5em}
\[
\mu(k,w) \! := \! \sum_{i=1}^{k-1}  i \cdot w(i) \prod_{ j=1 }^{i-1} \barw(j)  + k \prod_{j=1}^{k-1} \barw(j) 
	    \! \ge \! \min \! \bigg\{ \frac{k}{2}, \frac{1}{2w^*} \! \bigg\},
\]
\vspace{-1.1em}

and upper bound $\bbE X_{k;t}^2$ with $ v(k,w):=  \min\{ k, \sqrt{2}/w' \}$.
We abbreviate $\barT_{i,\delta}$\vphantom{$\hat{T}$} as $\barT_{i}$, $\rho(\delta)$ as $\rho$, $\mu(k,w)$ as $\mu_k$, $v(k,w)$ as $v_k$ when there is no ambiguity.
In anticipation of Theorem~\ref{thm:fix_conf_up_bd}, we define three more notations:
\begin{align*}
	& K_1 :=  \max \{~ K'-K  , \min\{\lfloor 1/w^*\rfloor, K-1 \}  ~\}, \\ 
	&		    K_2 := \max \{   K'-K , 1 \}, \quad
	  M_k :=  \frac{  K+1-k  }{\mu_{K+1-k}} - \frac{  K-k  }{\mu_{K-k}}.
\end{align*}}
\vspace{-1em}

\begin{restatable}{theorem}{thmFixConfUpBd} \label{thm:fix_conf_up_bd}
	Assume $K'<2K-1$. 
	With probability at least $1-\delta$, Algorithm~\ref{alg:bai_cas_conf} outputs an $\epsilon$-optimal arm 
	after at most $(c_1 N_1 + c_2 N_2 + c_3 N_3)$ steps, where
	{\setlength\abovedisplayskip{.5em}
    \setlength\belowdisplayskip{.4em}
\begin{align}
    & N_1 = \sum_{k=1}^{ K-K_2 }  
         \frac{ v_{K-k+1}^2   \vphantom{\hat{v}}  }{\mu_{K-k+1}^2} 
         \log \Bigg( \frac{ 1 }{\delta} \sum_{j=1}^{ K-K_2 } 
          \frac{ v_{K-j+1}^2 }{\mu_{K-j+1}^2} 
          \Bigg), \nonumber     \\
    & N_2 =  \frac{1}{\mu_K} \sum_{i=1}^{L-K} \bar{T}_{\sigma(i)}  
             , \nonumber    \\
    & N_3 = 
    	 \sum_{k=2}^{2K-K' }
    	 \frac{ K-k+1 }{\mu_{K-k+1}}
    	  \big[ \bar{T}_{\sigma(L-K+k)} - \bar{T}_{\sigma(L-K+k-1)} \big] \label{eq:N_three_main_result}
         \\
                    &  = 
            \sum_{k=1}^{K-K_1-1 }
                 M_k
          		\bar{T}_{ \sigma(L-K+k) }
           +   \bigg(  \frac{ K_1+1  }{ \mu_{K_1+1} } - 2 \bigg) \cdot
           \bar{T}_{ \sigma(L-K_1 ) }
         \nonumber    \\
         & \hspace{2em}
          + 2 \bar{T}_{ \sigma(L-K_2) }
           .    \label{eq:N_three_expand_main_result}
\end{align}}
\end{restatable}

When $\epsilon=0$, $\bar{\Delta}_i = \Delta_i $ for all $i\in[L]$ and $K'=K$. 
We note that it is a waste to pull identified items. This occurs only when $K'<2K-1$~(see Lemma~\ref{le:stop_eliminate_no}) and this scenario is more complicated to analyze. The scenario $K'\ge 2K-1$ is relatively easier to analyze and the result is deferred to Proposition~\ref{prop:fix_conf_up_bd_many_eps_opt}~(see Appendix~\ref{appdix:influen_epsilon}).

\textbf{Interpretation of the bound.} 
The first term 
$N_1$ in the bound is unique to the cascading model, which results from the gap between $X_{\hatk_t;t}$ and $\bbE X_{\hatk_t;t} $. 
We can bound $N_1$ in terms of the maximum and minimum weights, $w^*$ and $w'$. 
\begin{restatable}{proposition}{propTermCas}
\label{prop:term_by_cas}
	Assume $0<w'<w^*\le 1 $.  
	We have
	{\setlength\abovedisplayskip{.5em}
    \setlength\belowdisplayskip{.4em}
		\begin{align*}
		N_1 \le 
          \left\{
            \begin{array}{ll}
            	4K \log \big( \frac{ 4K }{\delta}  		
            	 \big) & 0< w^* \le 1/K,  \vphantom{ \Big( } \\
                 \vphantom{ \Big( }
            	\frac{ 8Kw^{*2} }{w'^2 }  \log \big(
            	\frac{ 8Kw^{*2} }{\delta w'^2 } 
            	  \big)  & 1/K < w^* \le 1.
            \end{array}
          \right.
	\end{align*}}
\end{restatable}
Next,  
recall that we say that an item is identified by time $t$ if it is put into $A_\tau$ or $R_\tau$ for some $\tau \leq t$.
In the worst-case scenario, 
the agent identifies items in descending order of $\bar{\Delta}_{i}$'s.
With probability at least $1-\delta$, it costs at most $c_2 N_2$ steps to identify 
items $\sigma(1),\ldots,\sigma(L-K)$
and 
$c_3 N_3$ is for identifying the remaining ones. 
More precisely, 
after item $\sigma(L\!-\!K\!-\!k\!-\!1)$ is identified,
the number of steps required for identifying
item $\sigma(L\!-\!K\!-\!k)$
 is 
$(c_3 /{\mu_{K-k+1}}) \cdot { (K-k+1) } 
    	  [ \bar{T}_{\sigma(L-K+k)} - \bar{T}_{\sigma(L-K+k-1)} ]$;
we sum these steps up to obtain \eqref{eq:N_three_main_result}.
Since the results in many existing works~\citep{even2002pac,jun2016top} mainly involve $\bar{T}_i$'s, we show the dependence of $N_3$ on $\bar{T}_i$'s more concretely in \eqref{eq:N_three_expand_main_result}.

\textbf{Technique.}
The crucial analytical challenge to derive our bound, especially to establish $\mu_k,v_k,N_1$, is to quantify the impact of partial feedback that results from the cascading model.
Firstly, we bound $\bbE X_{\hatk_t;t}$ by exploiting some properties of the cascading feedback.
Next, to bound the gap between $\sum_{t=1}^n X_{\hatk_t;t} $ and $\sum_{t=1}^n \bbE X_{\hatk_t;t} $ for some $n\in \bbN$, we propose a novel class of r.v.'s, known as LSG  r.v.'s, provide an estimate of a certain LSG parameter, and utilize a Berstein-type concentration inequality to bound the tail probability of a certain LSG r.v.. 
Details are in Section~\ref{sec:main_pf_part_feedback}.

To facilitate the remaining discussion in Section~\ref{sec:main_result_up}, we specialize our analysis and results henceforth to the case of $\epsilon \! = \! 0$, in which $\bar{\Delta}_i \! = \! \Delta_i$ and the agent aims to find $S^*$. 
The remaining results in Section~\ref{sec:main_result_up} can be directly generalized to the scenario of $\epsilon \! > \! 0$ by replacing  $\Delta_i$'s with $\bar{\Delta}_i$'s.

\textbf{Comparison to the semi-bandit problem.}
A related algorithm in the setting of semi-bandit feedback and $\epsilon=0$ is the {\sc BatchRacing} Algorithm, which was proposed by~\citet{jun2016top}. This algorithm has three paramters $k$, $r$ and $b$ which respectively represent the number of optimal items, the maximum number of pulls of one item at one step and the size of a pulled arm. When $r=1 $ and $b=k$, we denote it as {\sc BatRac($k$)}. 
The fact that our algorithm 
observes between $1$ and $K$ items per step
motivates a comparison among {\sc CascadeBAI($0,\delta,K$)}, {\sc BatRac($K$)} and {\sc BatRac($1$)}.

\begin{restatable}{corollary}{corCompFullFeedback}
    \label{cor:compare_full_feedback}
   (i) If all $w(i)$'s are at most $1/K$, 	with probability at least $1-\delta$, Algorithm~\ref{alg:bai_cas_conf} outputs $S^*$ after at most 
   {\setlength\abovedisplayskip{.5em}
    \setlength\belowdisplayskip{.4em}
   \begin{align*}
   		& O \bigg(
    \frac{1}{ K} \sum_{i=1}^{L-K} \bar{T}_{\sigma(i)} 
        + \bar{T}_{\sigma(L-1 )}
    \bigg)
    \\&
    = O \bigg(
    \frac{1}{ K}
    \sum_{i=1}^{L-K} 
    \Delta_{\sigma(i)}^{-2} \log
		\bigg[
		\frac{L}{\delta} \log \bigg( \frac{L}{\delta \Delta_{\sigma(i)}^{2} } \bigg)
		\bigg]
%	  \right. 
	  \\* & \hspace{3em} %\left.
	+
	\Delta_{\sigma(L-1)}^{-2} \log
		\bigg[
		\frac{L}{\delta} \log \bigg( \frac{L}{\delta \Delta_{\sigma(L-1)}^{2} } \bigg)
		\bigg]
   \bigg)
\end{align*}
steps;   
   (ii) if all $w(i)$'s are at least $1/2$, with probability at least $1-\delta$, Algorithm~\ref{alg:bai_cas_conf} outputs $S^*$ after at most 
   \begin{align*}
   		O\bigg(
    \sum_{i=1}^{L-1} \bar{T}_{\sigma(i)} 
    \bigg)
    = O \bigg(
    \sum_{i=1}^{L-1} 
    \Delta_{\sigma(i)}^{-2} \log
		\bigg[
		\frac{L}{\delta} \log \bigg( \frac{L}{\delta \Delta_{\sigma(i)}^{2} } \bigg)
		\bigg]
    \bigg)
\end{align*}
steps.}
\end{restatable}

The results of Corollary~\ref{cor:compare_full_feedback} are intuitive: 
(i) if all $w(i)$'s are close to $0$~(i.e., at most $1/K$), the bound on the time complexity of {\sc CascadeBAI($0,\delta,K$)} is of the same order as that of {\sc BatRac($K$)}; 
(ii) if all $w(i)$'s are close to $1$~(i.e., at least $1/2$), the bound corresponds with that of {\sc BatRac($1$)} \citep{jun2016top}. 
We further upper bound the expected time complexity of our algorithm~(denoted by $\pi_1$) in these cases.

\begin{restatable}{proposition}{propConfCompFullFeedbackExp}
\label{prop:conf_comp_full_feedback_exp}
      (i) If all $w(i)$'s are at most $1/K$, 
      {\setlength\abovedisplayskip{.5em}
    \setlength\belowdisplayskip{.4em}
   \begin{align*}
   	\bbE \calT^{\pi_1} \! \le \!
   		c_1 \log \bigg( \frac{1}{\delta} \bigg) \cdot 
    \bigg\{
         \frac{1}{ K}
    \sum_{i=1}^{L-K} 
    \Delta_{\sigma(i)}^{-2} \log
		\bigg[
		L \log \bigg( \frac{L}{  \Delta_{\sigma(i)}^{2} } \bigg)
		\bigg]
	   \\ \hspace{2em} %\left.
	+
	\Delta_{\sigma(L-1)}^{-2} \log
		\bigg[
		L\log \bigg( \frac{L}{ \Delta_{\sigma(L-1)}^{2} } \bigg)
		\bigg]
		\bigg\}
    ;
\end{align*}
   (ii) if all $w(i)$'s are at least $1/2$, %an upper bound is 
   \begin{align*}
   	\bbE \calT^{\pi_1} \le 
   		c_2
    \sum_{i=1}^{L-1} 
    \Delta_{\sigma(i)}^{-2} \log
		\bigg[
		L \log \bigg( \frac{L}{  \Delta_{\sigma(i)}^{2} } \bigg)
		\bigg]
		\log \bigg( \frac{1}{ \delta  } \bigg)
    .
\end{align*}}
\end{restatable}
According to the definition of $\bbT^*$ in Section~\ref{sec:prob_setup}, $\bbT^* \le \bbE \calT^{\pi_1}$ and hence also satisfies the above bounds.
Corollary~\ref{cor:compare_full_feedback} and Proposition~\ref{prop:conf_comp_full_feedback_exp} indicate that the high probability upper bound on $\calT^{\pi_1}$ and the upper bound on $\bbE \calT^{\pi_1}$ are of the same order in the sense that (i) if all $w(i)$'s are at most $1/K$, both upper bounds are 
$\tilde{O} \big( 
    ({1}/{ K})\cdot
    \sum_{i=1}^{L-K} 
    \Delta_{\sigma(i)}^{-2}  
	+
	\Delta_{\sigma(L-1)}^{-2}  
    \big)
$;
(ii) if all $w(i)$'s are at least $1/2$, both are 
$\tilde{O} \big(
    \sum_{i=1}^{L-1} 
    \Delta_{\sigma(i)}^{-2}  
    \big)
$.

\textbf{Specialization to the case of two click probabilities}.
We consider a simplified scenario with the following assumption; this allows us to present the upper bound on the time complexity with greater clarity.
\begin{assumption}
\label{assump:two_prob_cond}
	With $0<w'<w^*\le 1$, the $K$ optimal and $L-K$ suboptimal items have click probabilities $w^*$ and $w'$ respectively. 
\end{assumption}

\begin{restatable}{proposition}{propConfUpbdTwoprob}
\label{prop:conf_upbd_two_prob}
	Under Assumption~\ref{assump:two_prob_cond},
(i) if $0<w^* \le 1/K $,
with probability at least $1-\delta$, Algorithm~\ref{alg:bai_cas_conf} outputs $S^*$ after at most 	
\begin{align*} 
	O \bigg(
         \frac{  L }{   K(w^*-w')^2 }  \log \bigg[ \frac{L}{ \delta } \log \bigg( \frac{ L }{ \delta (w^*-w')^2 } \bigg) \bigg]
        \bigg) 
\end{align*}
steps;
(ii) if $1/K < w^* \le 1$, with probability at least $1-\delta$, Algorithm~\ref{alg:bai_cas_conf} outputs $S^*$ after at most 	
\begin{align*}
	O \bigg(
         \frac{ w^*L }{   (w^* \! - \! w')^2 }  \! \log \! \bigg[ \frac{L}{ \delta } \! \log \!  \bigg( \frac{ L }{ \delta (w^* \! - \! w')^2 } \bigg) \bigg]
         \! +  \! 
		\frac{w^{*2}  }{ w'^2 }  \! \log \! \bigg( \frac{1}{\delta} \bigg)
        \bigg)
\end{align*}
steps.
\end{restatable}
In the second case, if $L \ge { w^* ( w{^*}-w' )^2}/{w'^2}$, the first term dominates the bound. For instance, $w'\ge 1/\sqrt{L}$ satisfies this condition.

\begin{restatable}{proposition}{propConfUpbdTwoprobExp}
\label{prop:conf_upbd_two_prob_exp}
	Under Assumption~\ref{assump:two_prob_cond},
     (i) if $0 \! < \! w^* \! \le \! 1 \! / \! K\!,\!$%
     {\setlength\abovedisplayskip{.5em}
    \setlength\belowdisplayskip{.4em}
     \begin{align*}
   	\bbE \calT^{\pi_1} \le 
         \frac{ c_1 L }{ K (w^*-w')^2 }  \log \! \bigg[ L \log \! \bigg( \frac{  L }{   (w^*-w')^2 } \bigg) \bigg] \! \log \! \bigg(\frac{1}{\delta}  \bigg);
     \end{align*}
     (ii) if $w' \ge 1/2$ or $w^*/w'\le 2$, 
     \begin{align*}
   	\bbE \calT^{\pi_1} \le 
         \frac{ c_2 w^*L }{  (w^*-w')^2 }  \log  \! \bigg[ L \log  \!  \bigg( \frac{  L }{  (w^*-w')^2 } \bigg) \bigg]\!  \log \! \bigg(  \frac{1}{\delta} \bigg).  
     \end{align*}}
\end{restatable}
Proposition~\ref{prop:conf_upbd_two_prob_exp} also upper bounds $\bbT^*$ since $\bbT^* \le \bbE \calT^{\pi_1}$. 
It, together with Proposition~\ref{prop:conf_upbd_two_prob} implies that
the high probability bound on $\calT^{\pi_1}$ and the bound on $\bbE \calT^{\pi_1}$ are of the same order in these cases.

\subsection{Lower bound}
We set $\epsilon=0$, in which scenario the agent aims to find an optimal arm $S^*$.
We also upper bound the expected number of observations of items per step by $\tilde{\mu}(K,w)$ where
{\setlength\abovedisplayskip{.5em}
    \setlength\belowdisplayskip{.4em}
\begin{align*}
	\tilde{\mu}(k,w)
	& =
	 \! \sum_{i=1}^{k-1}  i \! \cdot w
	 \resizebox{.07\textwidth}{!}{$
	 (L \! + \! 1 \! - \! i) 
	 $}
	  \prod_{ j=1 }^{i-1}  \! \barw
	 \resizebox{.07\textwidth}{!}{$
	 (L \! + \! 1 \! - \! j) 
	 $}
	  \! +  \! k \!\! \prod_{j=1}^{k-1} \! \barw
	 \resizebox{.07\textwidth}{!}{$
	 (L \! + \! 1 \! - \! j) 
	 $}
	\\ & \le 
	\min\{ {1}/{w'} ,k\}.
\end{align*}
We write $\tilde{\mu}(k,w)=\tilde{\mu}_k, d(m,n) = \rmKL( w (m), w (n) )$,
where $\rmKL(p,q)= p \log \frac{p}{q} + (1-p) \log \frac{1-p}{1-q} $ .}

\begin{theorem}\label{thm:lb_fix_conf_dep}
	We have
	{\setlength\abovedisplayskip{.5em}
    \setlength\belowdisplayskip{.4em}
	\begin{align*}
		\bbT^* \! \ge \!
		 \frac{ \log (1/2.4\delta) }{\tilde{\mu}_K}
		 \!\cdot\!
		 \Bigg[
	\sum_{i=1}^K \frac{1}{ d ( i, K+1  ) }
	+\sum_{j=K+1}^L \frac{1}{ d ( j , K  ) }
	\Bigg].
	\end{align*} }
\end{theorem}

\textbf{Comparison to the semi-bandit problem.} 
First, 
if $w'$ is close to $1$~(i.e., $w'\ge 1/2$), $\tilde{\mu}_K=1/w'\le 2$, i.e., at one step, the agent observes at most $2$ items in expectation.
We can recover the lower bound in~\citet{kaufmann2016complexity} by replacing $\tilde{\mu}_K$ with $1$, which is of the same order as our bound in this case.
Next, 
if $w'$ is close to $0$~(i.e., $w'\le 1/K$), the agent observes $\tilde{\mu}_K=K$ items in expectation.
Then the bound is the same as that incurred by pulling $K$ items per step and getting semi-bandit feedback, which is
{\setlength\abovedisplayskip{.5em}
    \setlength\belowdisplayskip{.4em}
\begin{align*}
	\frac{ \log (1/2.4\delta) }{K}
		 \cdot
		 \Bigg[
	\sum_{i=1}^K \frac{1}{ d ( i, K+1  ) }
	+\sum_{j=K+1}^L \frac{1}{ d  ( j , K  ) }
	\Bigg].
\end{align*}}

\textbf{Specialization to the case of two click probabilities}.
\begin{restatable}{corollary}{corConfLowbdTwoprob}
\label{cor:conf_lowbd_two_prob} 
	Under Assumption~\ref{assump:two_prob_cond}, we have %for any $1 < \delta <  1$,
	{\setlength\abovedisplayskip{.5em}
    \setlength\belowdisplayskip{.4em}
	\begin{align*}
		\bbT^* & \ge %- \frac{\log \delta'}{\tilde{\mu}_K} +
		 \frac{\mathrm{KL}(1-\delta,\delta)}{\tilde{\mu}_K}
		 \cdot
		 \bigg[
		   \frac{K}{ \mathrm{KL} ( w^*, w' ) }
		   + \frac{L-K}{ \mathrm{KL} (  w', w^* ) }
	\bigg] 		
	  \\&	
		=
		\Omega \Big( 
	\min \{ w' ,1-w^* \} \cdot 
	\frac{L w' }{ ( w^*-w' )^2 } 
	\log  \Big[ \frac{1}{ \delta}  \Big] 
	 \Big)
	.
	\end{align*}
where 
$
	\tilde{\mu}_K = [ 1-(1-w')^K ] /{w'} \le {1}/{w'}. % \frac{ 1-(1-w')^K }{w'} \le \min\{ \frac{1}{w'}, K}.
$ }
\end{restatable}

\subsection{Comparison of the upper and lower bounds}
To see whether the upper and lower bounds on $\bbT^*$ match, we set $\epsilon=0$ and consider the following simplified cases. 
\begin{corollary}
\label{cor:comp_conf_up_low_two_prob}
	Set $\epsilon=0$. 
    (i) If $ 0<w^* \le 1/K$, 
    {\setlength\abovedisplayskip{.5em}
    \setlength\belowdisplayskip{.4em}
    \begin{align*}
        \bbT^* \in
        \tilde{\Omega} \bigg( 
	\frac{L w'^2 }{ ( w^*-w' )^2 } 
	\bigg)
		\bigcap
	\tilde{O}
        \bigg(
         \frac{   L }{   K(w^*-w')^2  } \bigg)   
	;
    \end{align*}
    (ii) if $1>  w'\ge 1/2$,
    \begin{align*}
    	\bbT^* \in
        \tilde{\Omega} \bigg( 
	\frac{L w' (1-w^*) }{ ( w^*-w' )^2 } 
	\bigg)
		\bigcap
	\tilde{O}
        \bigg(
         \frac{   w^* L }{  (w^*-w')^2  } \bigg)
         .
    \end{align*}
    The upper bounds above are achieved by Algorithm~\ref{alg:bai_cas_conf}. }
\end{corollary}

In the first case, the gap between the upper and lower bounds is manifested in the terms $1/K$ and $w'^2$. In the second case, the gap is manifested in $w^*$ and $w'(1-w^*)$.

\section{Proof sketch}
\label{sec:proof_sketch}

\subsection{Analysis of partial feedback for cascading bandits} \label{sec:main_pf_part_feedback}
At a high level, the time complexity $\calT$ can be established by analyzing $\sum_{t=1}^\calT X_{\hatk_t;t}$ and $X_{\hatk_t;t}$.  %
The first term is determined by $\bar{T}_{i,\delta}$'s, the number of observations that guarantees the correct identification of items with high probability.
These $\bar{T}_{i,\delta}$'s  
are invariant to the scenario whether the agent receives semi-bandit or partial feedback from the user.
The second term $X_{\hatk_t;t}$ equals to $\hat{k}_t$ in the semi-bandit feedback setting while it is an r.v. in the partial feedback setting.
Since $\bar{T}_{i,\delta}$'s have already been studied by a number of works on the semi-bandit feedback~\citep{even2002pac,jun2016top},
the crucial challenge of analyzing cascading bandits is to estimate $X_{\hatk_t;t}$ probabilitistically.

According to  
Algorithm~\ref{alg:bai_cas_conf}, $\hatk_t = \min\{ K, |D_t|\}$. 
When $\hatk_t=K\le |D_t|$, the agent pulls $ K$ surviving (i.e., not identified) items.
Otherwise, the agent pulls all surviving items first and then complements $S_t$ with some identified items. 
In the cascading bandit setting, the agent observes only one item when the first item $i_1^t$ is clicked, and the corresponding probability is $w(i^t_1)$; the agent observes two items when % the first item 
$i_1^t$ is not clicked but  
$i_2^t$ is clicked, and the probability is $[1-w(i_1^t)] w(i_2^t)$; and so on. 
Therefore,  
{\setlength\abovedisplayskip{.5em}
    \setlength\belowdisplayskip{.4em}
\begin{align*}
	& \bbE X_{\hatk_t;t}   =  
	\sum_{i=1}^{\hat{k}_t-1} i \cdot \bigg[ \prod_{j=1}^{i-1} \barw(i^t_j) w(i^t_i) \bigg]
	+ \hat{k}_t \prod_{j=1}^{\hat{k}_t-1} \barw(i^t_j) .
\end{align*}
Since $\bbE X_{\hatk_t ;t}$ depends only on $S_t$~(the pulled arm at step $t$) and $S_t$ is learnt online, it is difficult to estimate $\bbE X_{\hatk_t ;t}$ for each step separately.
Therefore, the second best thing one can do is to bound $\bbE X_{\hatk_t ;t}$ as a function of $\hatk_t $ and $\bmw$.
We now present some properties of $\bbE X_{\hatk_t;t}$.}

\begin{restatable}{theorem}{thmObsProp}
\label{thm:prop_obs}
	Consider a set of items with weights $\bmu = (u_1,\ldots,u_k)$ such that $u_1 \ge \ldots\ge u_k$, and let $\mu_k(\bmu,I)$ be the expected number of observations when items are placed with order $I$. Let $I_{ \textrm{dec} }=(1,\ldots,k)$, 
	 $I_{ \textrm{inc} }=(k,\ldots,1)$,  
	 and $I$ be any order, then
\\ (i) boundedness: ${\mu}_k(\bmu,I_{ \textrm{dec} })\!  \le {\mu}_k(\bmu,I) \le {\mu}_k(\bmu,I_{ \textrm{inc} })$;
\\ (ii) monotonicity: let $\bmv \! = \! (v_1,\ldots,v_k)$ be another vector of weights, then $ {\mu}_k( \bmu, I) \! \ge \! {\mu}_k( \bmv, I)$ if $u_i \! \le \! v_i$ for all $i \! \in \! [k]$.
\end{restatable}
Theorem~\ref{thm:prop_obs} implies that when $\bmw$ is fixed, $\bbE X_{k;t}$ attains its minimum when the agent pulls items $1,2,\ldots,k$ in this order and attains its maximum when the agent pulls $L,L-1,\ldots, L -   K  +  1$ in this order.
Moreover, if $w(i)=w^*$ for all $i\in [k]$, $\bbE X_{k;t}$ is even smaller;
 if $w(j)=w'$ for all $j \in \{ L  -  k  +  1,\ldots, L \}$, $\bbE X_{k;t}$ is even larger.
This observation inspires Lemma~\ref{lemma:bound_obs}.

\begin{restatable}{lemma}{lemmaBoundObs}
\label{lemma:bound_obs}
For any $k,t$, 
{\setlength\abovedisplayskip{.5em}
    \setlength\belowdisplayskip{.4em}
\begin{align*}
	\min\Big\{ \frac{k}{2}, \frac{1}{2w^*} \Big\} \le \mu_k \le \bbE X_{k;t} \le \tilde{\mu}_k \le \min \Big\{ \frac{1}{w'} , k \Big\}.
\end{align*}}
\end{restatable}
Next, since $ X_{k;t}$, instead of $ \bbE X_{k;t}$, affects the dynamics,
we examine the gap between 
$\sum_{t=1}^n X_{k;t}$ and $\sum_{t=1}^n \bbE X_{k;t}$.
Clearly, a tight concentration inequality is essential to estimate this gap well.
Since $X_{k;t}$ is a bounded r.v.,  
there are some applicable Bernstein-type inequalities.
For instance, we can apply Azuma's inequality to analyze SG r.v.'s.
However, 
(i) it is challenging to find an SG parameter of $X_{k;t}$ that is good enough for our purpose~(a more detailed explanation is provided after Lemma~\ref{lemma:conc_one_sub_gauss_apply_obs}), and 
(ii) we only require a one-sided concentration inequality.
Hence, we resort to defining a new class of r.v.'s --- known as LSG r.v.'s --- and provide an estimate of the relevant LSG parameter.

\begin{restatable}[LSG]{definition}{defSubGaussLeftside} 
\label{def:sub_gauss_left_sided}
    An r.v. $X$ 
    is $v$-LSG ($v \ge 0$) if
    {\setlength\abovedisplayskip{.5em}
    \setlength\belowdisplayskip{.4em}
    \begin{align}
        \bbE [ \exp [\lambda (X-\bbE X ) ] ~] \le \exp ( {v^2 \lambda^2}/{2}  ), \ \
        \forall \lambda \le 0
        . 
        \nonumber 
    \end{align}}
\end{restatable}
 \begin{restatable}{theorem}{thmSecmoToLeftSubgauss}
\label{thm:sec_moment_to_left_sub_gauss}
	Let $X$ be an almost surely bounded nonnegative r.v.. 
If $\mathbb{E}X^2\le v^2$,  then $X$ is $v$-LSG.	
\end{restatable}

Furthermore, we bound $\bbE X_{k;t}^2$~(Lemma~\ref{lemma:obs_sec_moment})  
and adapt a variation of Azuma's inequality as in Theorem~\ref{thm:conc_sub_gauss_onesided} \citep{shamir2011variant} to evaluate the dependence between the number of observations and the number of time steps. 
\begin{restatable}{lemma}{lemmaObsSecmo}
\label{lemma:obs_sec_moment}
    For any $k,t$, $\bbE X_{k;t}^2 \le v_k^2 = \min\{ k^2,2/{w'^2} ~\}$.
\end{restatable}

\begin{restatable}{lemma}{lemmaOneSubConcObs}
\label{lemma:conc_one_sub_gauss_apply_obs}
For any $k,t,\delta>0$, set
{\setlength\abovedisplayskip{.5em}
    \setlength\belowdisplayskip{.4em}
\begin{align*}
	\calE^* := \bigg\{ \sum_{t=1}^n X_{k;t} \le n \mu_k - \sqrt{2n v_k^2 \log \bigg( \frac{1}{\delta} \bigg) } ~\bigg\},
\end{align*}
then $\Pr( \calE^* ) \le \delta$. 
Further when $\overline{\calE^* }$ holds, for any $T \! > \! 0$, $\sum_{t=1}^n X_{k;t} \! \le \! T$ implies that
$n \! \le \!
	 {2T}/{\mu_k} 
	\! + \!  {2 \log (1/ \delta) v_k^2 }/{\mu_k^2}  $. }
\end{restatable}

Lemma~\ref{lemma:conc_one_sub_gauss_apply_obs}  
implies that with high probability, we can lower bound the amount of observations on the surviving items over the whole horizon.
Subsequently, with probability at least $1-\delta$, the agent would have received sufficiently many observations on the surviving items to return an $\epsilon$-optimal arm after at most $(c_1 N_1 + c_2 N_2 + c_3 N_3)$ time steps (see Theorem~\ref{thm:fix_conf_up_bd}). %
The lemma also indicates that a smaller LSG/SG parameter of $X_{k;t}$  leads to 
a smaller upper bound on the number of time steps.
Since we can show $X_{k;t}$ is $v_k$-LSG but cannot show it is $v_k$-SG~(a detailed discussion is deferred to Appendix~\ref{pf:thm_sec_moment_to_left_sub_gauss}), it is beneficial to consider the class of LSG distributions for our problem.
The class of LSG r.v.'s and the general estimate of the LSG parameter, which is crucial for the utilization of the concentration inequality, may be of independent interest.

\subsection{Proof sketch of Theorem~\ref{thm:fix_conf_up_bd}} 
\label{sec:main_pf_ub}
\textbf{Concentration.} As the algorithm proceeds, the agent moves items from $D_t$ to $A_t$ or $R_t$ according to the confidence bounds
of all surviving items in $D_t$. 
This motivates us to define a ``nice event"
{\setlength\abovedisplayskip{.5em}
    \setlength\belowdisplayskip{.4em}
	$$\calE(i,\delta) := \{ \forall t\ge 1, L_t(i,\delta) \le \hat{w}_t(i) \le U_t(i,\delta) \}.$$
To show that $\bigcap_{i=1}^L \calE(i,\delta)$ holds with high probability, we utilize Theorem~\ref{thm:conc_log}~\citep{jamieson2014lil,jun2016top} and the SG property of $W_t(i)$~(the r.v. that reflects whether item $i$ is clicked at time step $t$). 
}

\begin{restatable}{lemma}{leConcEvent}
\label{le:conc_event}
	For any $\delta \! \in \! [0,1]$,  
	$\bbP \big( \bigcap_{i=1}^L \calE(i,\delta) \big) \ge 1 - { \delta}/{2}.$%
\end{restatable}

\textbf{Sufficient observations.} Next, we  
assume $\bigcap_{i=1}^L \calE(i,\delta)$ holds 
and find the number of observations that guarantees the correct identification of an item. 
To facilitate the analysis of the expected time complexity~(Proposition~\ref{prop:conf_comp_full_feedback_exp}, \ref{prop:conf_upbd_two_prob_exp}), we assume $\bigcap_{i=1}^L \calE(i,\delta')$ holds for a fixed $\delta' \in (0, \delta]$ in Lemma~\ref{le:suff_obs_to_identif},
which generalizes \citet[Lemma 2]{jun2016top}. 

\begin{restatable} 
{lemma}{lemmaSuffObsIdentif}
 \label{le:suff_obs_to_identif}
 {\setlength\abovedisplayskip{.5em}
    \setlength\belowdisplayskip{.4em}
    Fix any $0 \! < \! \delta' \! \le \! \delta$, assume $\bigcap_{i=1}^L \calE(i,\delta')$ holds. 
    Set 
     $T'_t := \min_{i\in D_t} T_t(i)$, 
     then for any time step $t$,
    \begin{align}        
        & 
        \forall i \! \le \! K' ,   T'(t) \! \ge \! \barT_{i,\delta'} 
             \Rightarrow      L_{t}(i , \delta) \! > \! U_{t}(j^* \! , \delta) \! - \! \epsilon 
             \Rightarrow      i \! \in \! A_t,   
        \nonumber \\  
        & 
        \forall i \! > \! K' ,  T'(t) \! \ge \! \barT_{i,\delta'} 
             \Rightarrow    U_{t}(i , \delta) \! <  \! L_{t}( j' \! ,  \delta) \! - \! \epsilon  
             \Rightarrow      i \! \in \! R_t.   
        \nonumber %\label{eq:reject_subopt}
    \end{align}}
\end{restatable}
Lemmas~\ref{le:conc_event} and \ref{le:suff_obs_to_identif} imply that with sufficiently many observations, the agent can correctly identify items 
with probability at least $1 \! - \! { \delta}/{2}$.

\textbf{Time complexity.}  
Subsequently, we observe that our algorithm stops before identifying all items. 
\begin{restatable}{lemma}{leStopElimintateNo}
\label{le:stop_eliminate_no}
	Assume $\bigcap_{i=1}^L \calE(i,\delta)$ holds. Algorithm~\ref{alg:bai_cas_conf} stops after
	identifying at most $L- \max\{ K'-K, 1\}$ items. 
\end{restatable}
Lemma~\ref{le:stop_eliminate_no} indicates that it suffices to count the number of time steps needed to identify at most $L-K'+K$ items.

We consider the worst case in which the agent identifies items in descending order of the $\bar{\Delta}_i$'s. We divide the whole horizon into several phrases according to $|D_t|$, the number of surviving items. 
During each phrase, we upper bound the required number of observations with Lemma~\ref{le:suff_obs_to_identif}; then Lemma~\ref{lemma:conc_one_sub_gauss_apply_obs} helps to upper bound the required number of time steps with high probability. Lastly, we bound the total error probability 
by $\delta/2$ and utilize the Lagrange multipliers to solve the following problem:
{\setlength\abovedisplayskip{.5em}
    \setlength\belowdisplayskip{.4em}
\begin{align*} 
	\max_{\delta_k: 1\le k\le 2K-K' } \! \sum_{k=1}^{2K-K' }  \frac{2 v_{K-k+1}^2 }{\mu_{K-k+1}^2} \log \delta_k  
	 ~\text{ s.t.} 
	\sum_{k=1}^{2K-K'} \!\! \delta_k \le \delta/2. 
\end{align*}
Altogether, we upper bound the time complexity.}

\subsection{Proof sketch of Theorem~\ref{thm:lb_fix_conf_dep}} 
\label{sec:main_pf_lb}

\textbf{Construct instances.} 
To begin, we fix $\alpha>0$ and define a class of $L+1$ instances, indexed by $\ell = 0, 1, \! \ldots \! , L$:  
\\ $\bullet$ under instance $0$, we have $\{w(1), w(2), \ldots, w(L)\}$, 
\\ $\bullet$ under instance $\ell$, we have $\{w(1), w(2), \ldots, w(\ell  - 1), \hspace{1em}$ $w^{(\ell)}(\ell), w(\ell+1), \ldots, w(L)\}$;
\\
where we define $w^{(\ell)}_\ell$'s so that they satisfy 
{\setlength\abovedisplayskip{.5em}
    \setlength\belowdisplayskip{.4em}
\begin{align*}
	& 1\le i \le K: ~
	  w^{(i)}(i) < w(K+1), \\
	& \hspace{2em} \rmKL(w(i), w(K+1)  )<\rmKL(w (i),w^{(i)}(i)  ) ,\\
	& \hspace{2em} \rmKL(w (i),w^{(i)}(i)  ) < \rmKL(w(i), w(K+1)  )+\alpha, \\
	& K< j \le L: ~
	  w^{(j)}(j) > w(K), \\
	& \hspace{2em} \rmKL(w(j), w(K)  )< \rmKL(w (j),w^{(j)}(j)  ),\\
	& \hspace{2em} \rmKL(w (j),w^{(j)}(j)  ) < \rmKL(w(j), w(K)  )+\alpha.
\end{align*}
In particular, $ S^* \in [K]^{(K)}$ is optimal under instance $0$, while suboptimal under instance $1\le \ell \le L$.
Bearing the risk of overloading the notations, under instance $\ell$, we denote $S^{*, \ell} $ as an optimal arm, $S^{\pi, \ell}_t$ as the arm chosen by algorithm $\pi$ at step $t$ and $\bmO^{\pi, \ell}_t$ as the corresponding stochastic outcome~(see its definition in Section~\ref{sec:prob_setup}).

\begin{figure}[ht]
	\centering
	\includegraphics[width=.48\textwidth, trim= 0 20 0 15,clip]{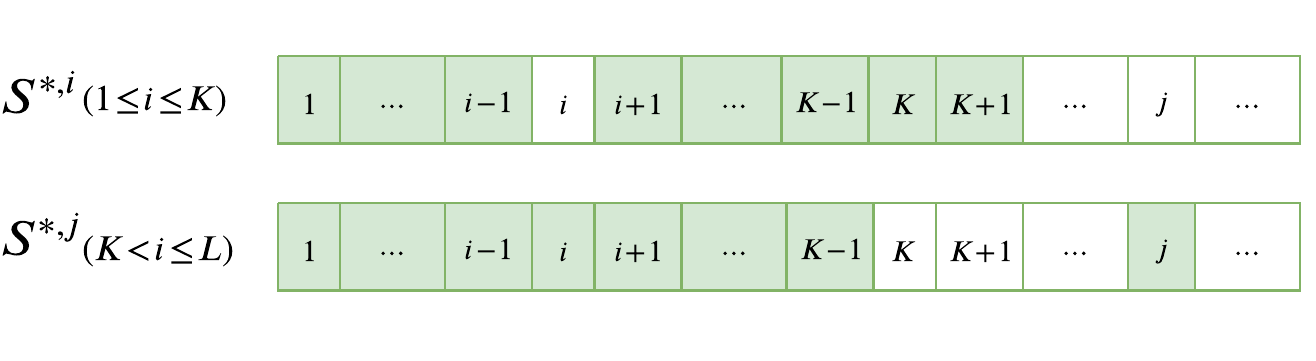}	
	\caption{Optimal set $S^{*,\ell}$ in instance $\ell$ (shaded in green)}
	\label{fig:instance_lb}
\end{figure}}

\textbf{KL divergence.} 
We first apply chain rule to decompose 
$\mathrm{KL}\big(~\{ S^{\pi, 0}_t, \bmO^{\pi, 0}_t \}^\calT_{t=1}, \{ S^{\pi, \ell}_t, \bmO^{\pi, \ell}_t \}^\calT_{t=1}~\big)$.  \vphantom{ $ \bar{  \bmO^{\pi, 0}_t }$}

\begin{restatable}[{\citet[Lemma 6.4]{orSubmit}}]{lemma}{lemmaKLdecompOR}\label{lemma:KLdecompOR}
	For any $1\le \ell \le L$,
	{\setlength\abovedisplayskip{.5em}
    \setlength\belowdisplayskip{.4em}
\begin{align*}
	&
	\mathrm{KL}\big(~ \{ S^{\pi, 0}_t, \bmO^{\pi, 0}_t \}^\calT_{t=1}, \{ S^{\pi, \ell}_t, \bmO^{\pi, \ell}_t \}^\calT_{t=1} ~\big)
	\\ & 
	\resizebox{.5\textwidth}{!}{$
	= \sum\limits_{t=1}^\calT \sum\limits_{s_t\in [L]^{(K)}} \Pr[S^{\pi, 0}_t = s_t]
	\mathrm{KL}  \big(  P_{ \bmO_t^{\pi,0 }  \mid S_t^{\pi,0}}(\cdot \mid s_t) \big\| P_{ \bmO_t^{\pi, \ell }  \mid S_t^{\pi,\ell}}(\cdot \mid s_t)  \big). \nonumber  
	$}
\end{align*}}
\end{restatable}

Next, we lower bound $\bbE [T_\calT(\ell)]$ with the KL divergence by applying a result from \citet{kaufmann2016complexity}. 

\begin{restatable} 
{lemma}{lemmaLbKLdecomp}\label{lemma:lb_KLdecomp}
For any $1\le \ell \le L$,
{\setlength\abovedisplayskip{.5em}
    \setlength\belowdisplayskip{.4em}
\begin{align*}
	&
	\mathrm{KL}\big(~ \{ S^{\pi, 0}_t, \bmO^{\pi, 0}_t \}^\calT_{t=1}, \{ S^{\pi, \ell}_t, \bmO^{\pi, \ell}_t \}^\calT_{t=1} ~\big)
	\\& 
	= \bbE [T_\calT(\ell)] \cdot \mathrm{KL}\big( w (\ell), w^{(\ell)}(\ell) \big)
	  \ge 
		\sup_{\calE \in \calT} 
		\mathrm{KL} \big( \bbP_0(\calE), \bbP_\ell(\calE) \big).
		\nonumber 
\end{align*}}
\end{restatable}
Define the event $\calE := \{ \hat{S}^\pi    \in   [K]^{(K)} \} \in \calF_\calT$. We establish that, for any $(\delta,K)$-PAC algorithm, $\bbP_0(\calE)  \ge 1  -   \delta$ and $\bbP_\ell(\calE)   \le    \delta ~(\forall 1  \le  \ell  \le L)$.
Lastly, by revisiting the definition of $X_{k;t}$ in Section~\ref{sec:main_result_up}, we see that $\tilde{\mu}_K$ also upper bounds the expected number of observations of items at one step for any $(\delta, K)$-PAC algorithm~(Lemma~\ref{lemma:bound_obs}).
This allows us to lower bound $\bbT^*$.

\section{Experiments}
We evaluate the performance of {\sc CascadeBAI$(\epsilon,\delta,K)$} and some related algorithms.
For each choice of algorithm and instance, we run $20$ independent trials. The standard deviations of the time complexities of our algorithm are negligible compared to the averages, and therefore are omitted from the plots.
More details are provided in Appendix~\ref{appdix_exp}.

\subsection{Order of pulled items}

    \begin{figure}[h]
        \centering
        \hspace{-.6em} %-.5em
        \includegraphics[width = .48\textwidth]{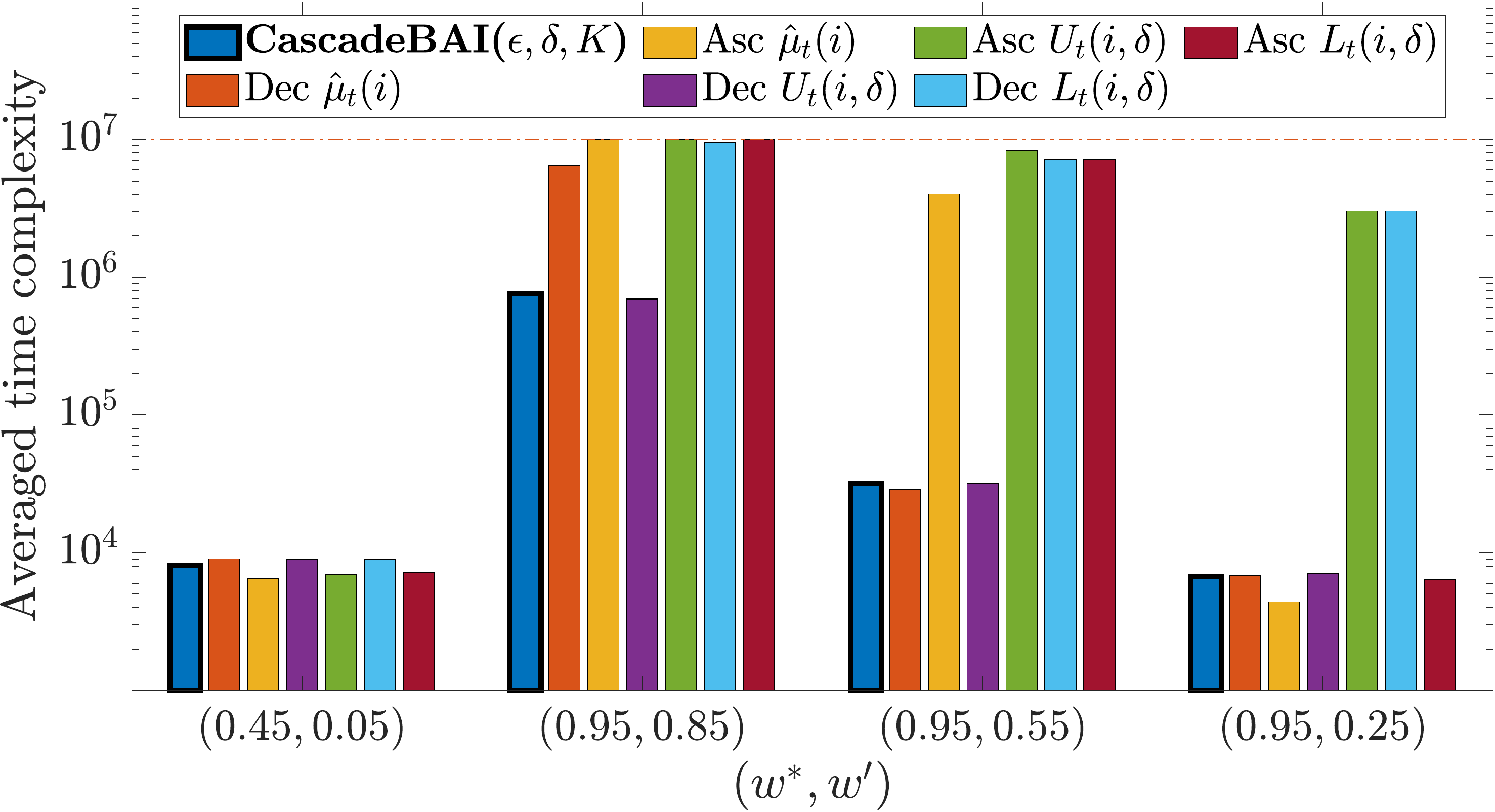} %\\   .45\textwidth
       \caption{Average time complexity incurred by different sorting order of $S_t$: ascending order of $T_i(t)$~(Algorithm~\ref{alg:bai_cas_conf}), ascending/descending order of $\hat{\mu}_t(i)/ U_t(i)/ L_t(i)$ with $L=64$, $K=16$, $\delta=0.1$ and $\epsilon=0$ in the cascading bandits. 
       }\label{pic:compare_order} 
    \end{figure}
    
As shown in Lines 5--9 of Algorithm~\ref{alg:bai_cas_conf}, {\sc CascadeBAI($\epsilon,\delta,K$)} sorts items in $S_t$ based on {\em ascending} order of $T_{t-1}(i)$'s.  
This order is crucial for proving our theoretical results. 
To learn the impact of ordering on the time complexity, we evaluate the empirical performance of sorting $S_t$ in descending or ascending order of $\hatw_t(i)$'s, $U_t(i,\delta)$'s or $L_t(i,\delta)$'s. 
We compare these methods under various problem settings and set the maximum time step as $10^7$.
Figure~\ref{pic:compare_order} shows that {\sc CascadeBAI($\epsilon,\delta,K$)} empirically performs as well as the best among the other heuristics, but {\sc CascadeBAI($\epsilon,\delta,K$)} is the only one with a theoretical guarantee.

\subsection{Comparison to semi-feedback setting}

We compare {\sc CascadeBAI($0,\delta,K$)}, {\sc BatRac($K$)} and {\sc BatRac($1$)}~\citep{jun2016top} empirically.
In Figure~\ref{pic:compare_semi}, if $w^*$, $w'$ are sufficiently small as the parameters shown in subfigure~(a), {\sc CascadeBAI($0,\delta,K$)}  performs similarly to {\sc BatRac($K$)}; 
if $w^*$, $w'$ are large as in subfigure~(b), it behaves similarly to {\sc BatRac($1$)}.
This corroborates the implications of Corollary~\ref{cor:compare_full_feedback}.

\begin{figure}[ht]
  \centering
  \subfigure[$w^* \! = \! \frac{1}{K}, w' \! = \! \frac{1}{K^2}$]{
\label{pic:compare_semi_a}
\includegraphics[width = .23\textwidth]{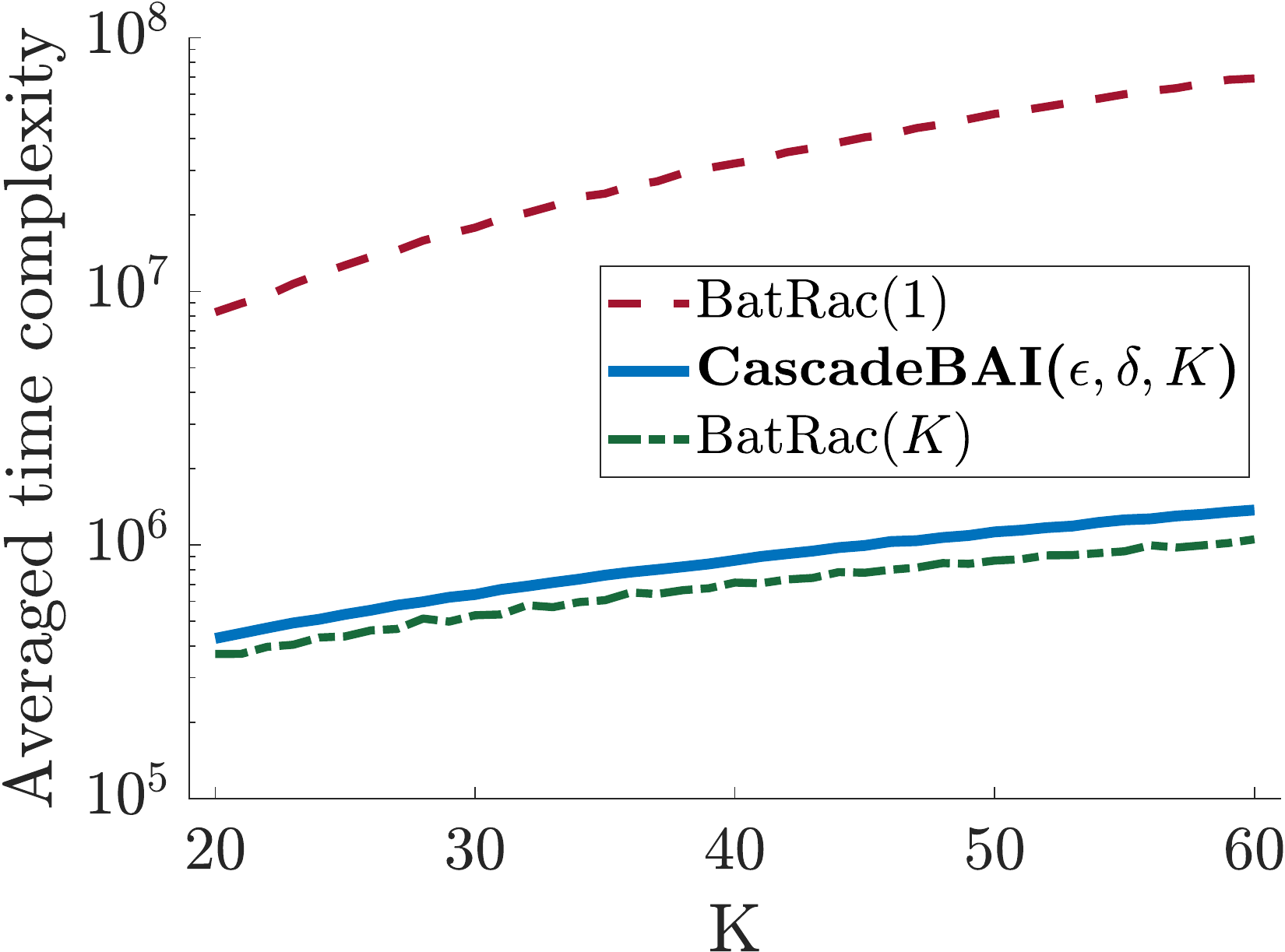}
}
\hspace{-.5em}
\subfigure[$w^* \! = \! 1 \! - \! \frac{1}{K^2}, w' \! = \! 1 \! - \! \frac{1}{K}$]{
\label{pic:compare_semi_b}
\includegraphics[width = .23\textwidth]{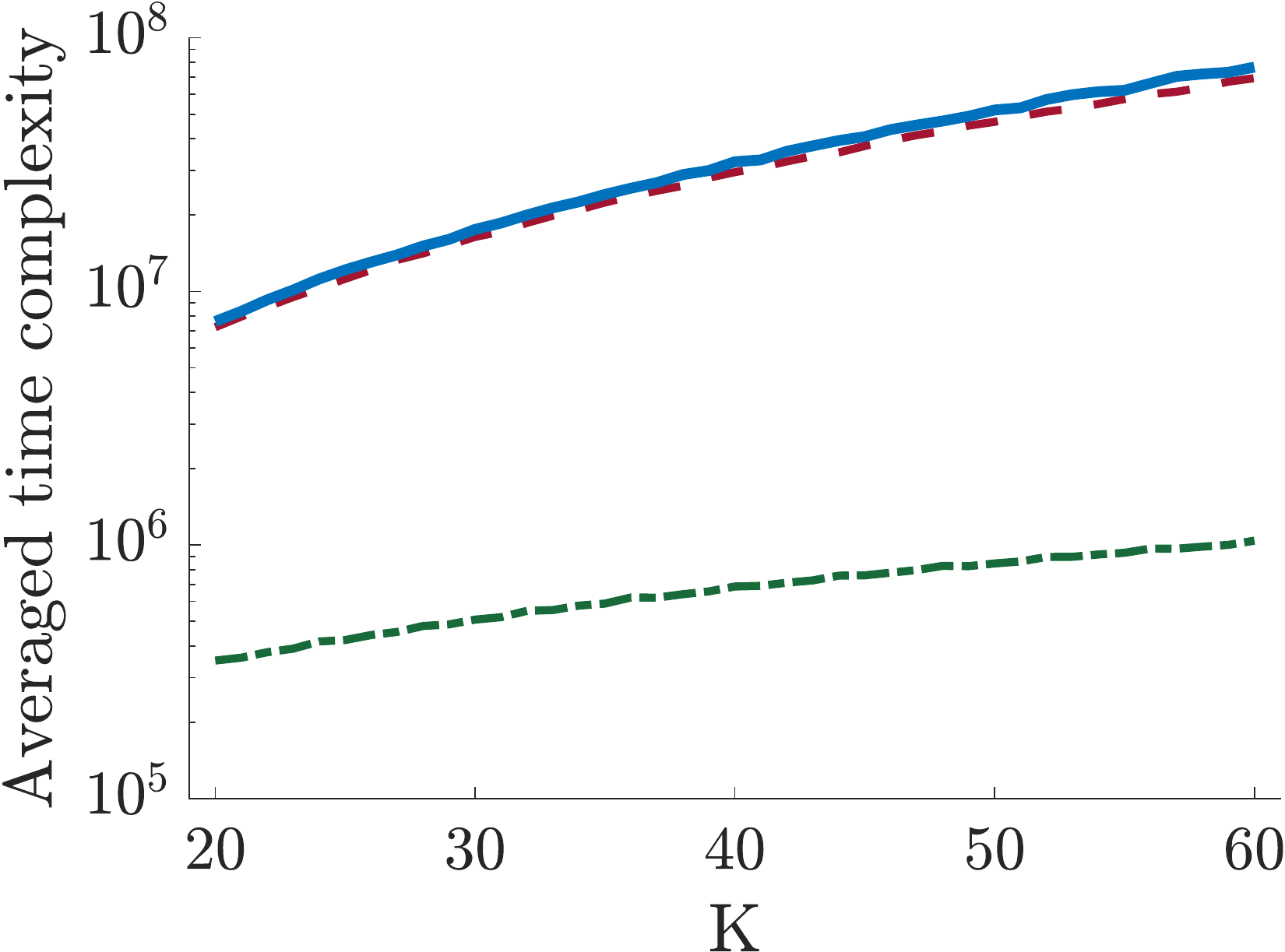} 
}
  \caption{Average time complexity of {\sc CascadeBAI$(\epsilon,\delta,K)$}, {\sc BatRac($1$)}, {\sc BatRac($K$)} with $L=128$, $\delta=0.1$, $\epsilon=0$, $K = 20,\ldots,60$. }
  \label{pic:compare_semi}
\end{figure}

\subsection{ Further empirical evidence }

Our analysis of the cascading feedback involves $v_k, \mu_k$ in the upper bound of the time complexity; these parameters depend strongly on $w^*, w'$ and $K$ (Lemma~\ref{lemma:bound_obs}, \ref{lemma:obs_sec_moment}). 
Hence, to assess the tightness of our analyses,
we consider several simplified cases by choosing $w^*, w'$ as functions of $ K$ and examine whether the dependence of the resultant time complexity (Proposition~\ref{prop:conf_upbd_two_prob}) on $K$ is materialized through numerical experiments.

\begin{table}[ht]
  \centering
  \caption{Upper bounds on the time complexity of Algorithm~\ref{alg:bai_cas_conf} with $L=128$, $K = 20,\ldots,60$, $\delta=0.1$, $\epsilon=0$~(Proposition~\ref{prop:conf_upbd_two_prob}), and their fitting results. 
  }
  \resizebox{.48\textwidth}{!}{ 
	\begin{tabular}{ccccc}
%	  No. & 
	  $w^*$     & $w'$ &  Upper bound    & Fit. model & $R^2$-stat.\\ 
	  \hline \\[-.8em]
%	  1 & 
	  $ {1}/{K}$ & $ {1}/{K^2}$ & $\tilde{O}  (
          K  
         )~$ &  $c_1 K + c_2$  
        & $0.999$	 
        \\[.5em]
      %
%      3 & 
      $1 - \!  {1}/{K^2 }$ & $ 1 - \!{1}/{K }$ & $\tilde{O}  (
         K^2  
         )$ &  $c_1 K^2 + c_2$  
        & $0.999$	 
        \\[.5em]
      %
%      2 & 
      $ {1}/{\sqrt{K} }$ & $ {1}/{K }$ & $\tilde{O}  (
         K  
         )~$ &  $c_1 K + c_2$ 
      & $0.973$	 
        \\[.5em]
      %
%      4 & 
      $1-\!  {1}/{K }$ & $ 1-\!  {1}/{\sqrt{K} }$ & $\tilde{O} (
         K%L %\log \left[ \frac{L}{ \delta } \log \left( \frac{ KL }{ \delta (w^*-w')^2 } \right) \right]
         )~$ &  $c_1 K + c_2$   
        & $1.000$	 
        \\[.5em]
      %
%      5 & 
      $1-\!  {1}/{K }$ & $  {1}/{K }$ & $\tilde{O} (
         K^2  
        )$ &  $c_1 K^2 + c_2$
        & $0.992$	
	    \end{tabular}%
	    }
	  \label{tab:K_set_approx}%
	\end{table}

We fit a model to the averaged time complexity under each setting as stated in Table~\ref{tab:K_set_approx}. 
In each case, the $R^2$-statistic is almost $1$, implying that the variability of the time complexity is almost fully explained by the proposed polynomial model~\citep{glantz1990primer}. 
Therefore, the empirical results show that the dependence of the upper bound on $K$ (Proposition~\ref{prop:conf_upbd_two_prob}) is rather tight, which implies that using the new concept of  LSG r.v.'s, our quantifications of the cascading feedback are also rather tight.

\begin{figure}[ht]
        \centering
  \subfigure[$w^* \! = \! \frac{1}{K}, w' \! = \! \frac{1}{K^2}$]{
\label{pic:K_fit_row_a}
\includegraphics[width = .23\textwidth]{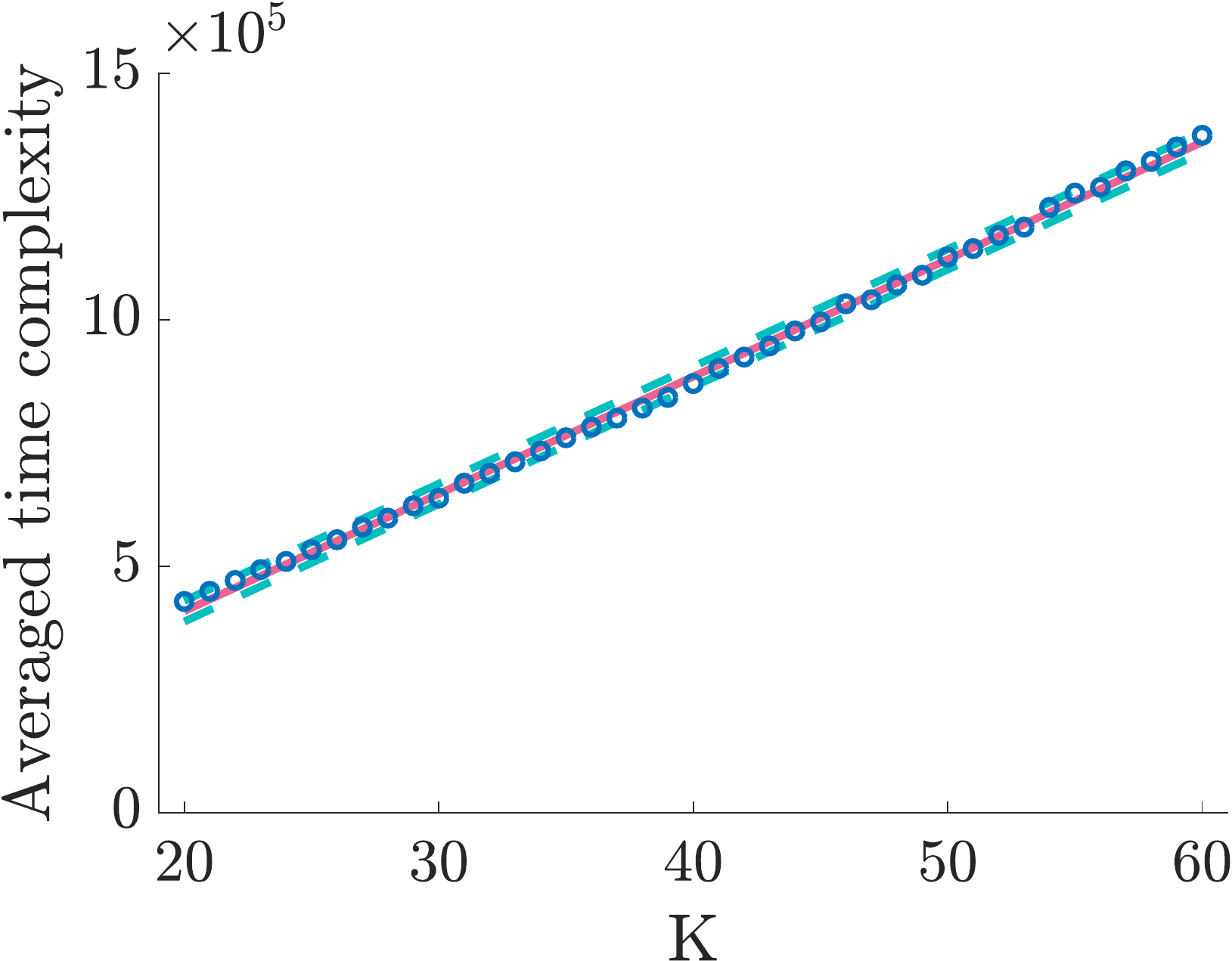} 
}
\hspace{-.5em}
\subfigure[$w^* \! = \! 1 \! - \! \frac{1}{K^2}, w' \! = \! 1 \! - \! \frac{1}{K}$]{
\label{pic:K_fit_row_b}
\includegraphics[width = .23\textwidth]{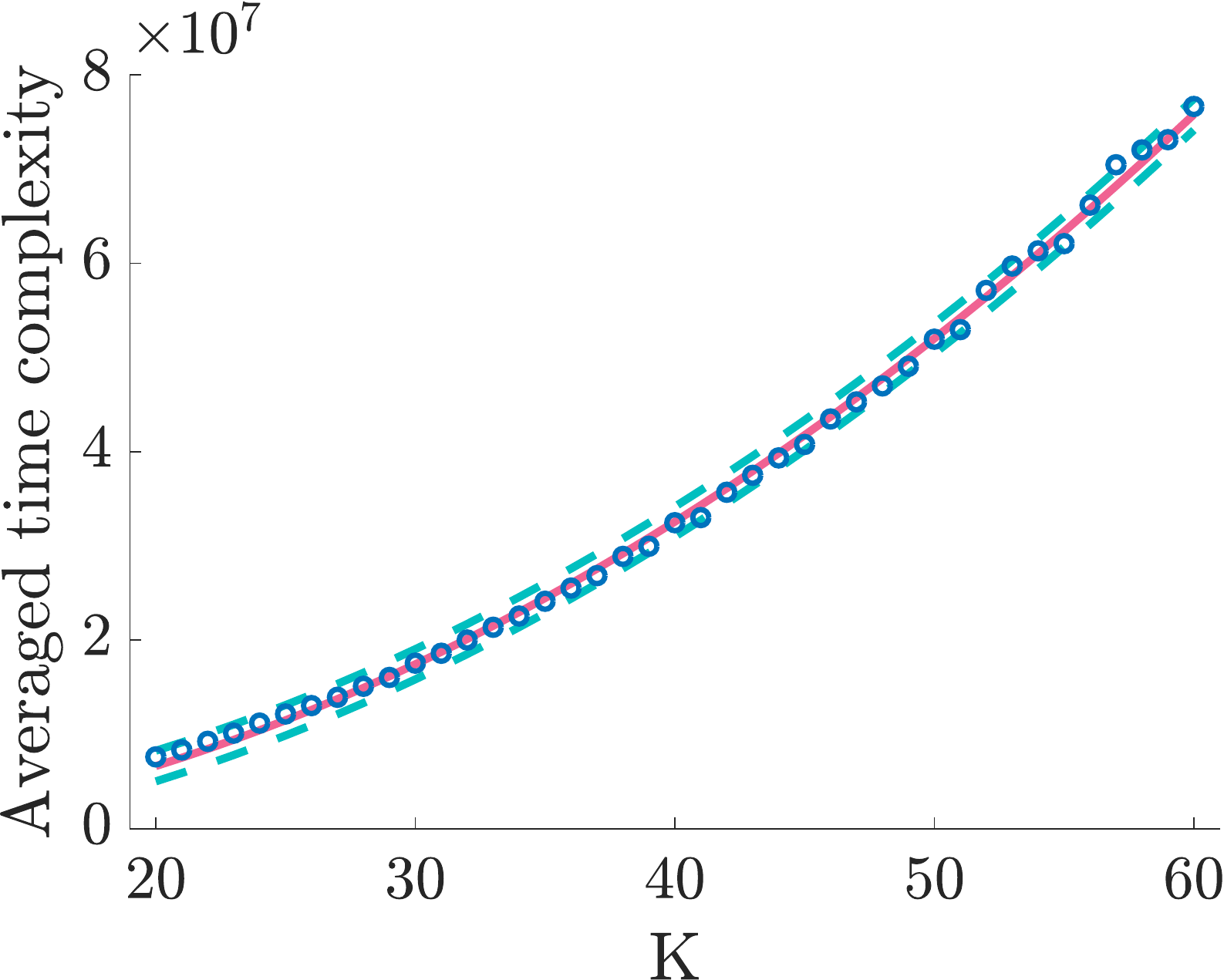} 
}
       \caption{Fit the averaged time complexity with functions of $K$ for two cases. Fix $L=128$, $\delta=0.1$, $\epsilon=0$.  
       Blue dots are the averaged time complexity, red line is the fitted curve, and cyan dashed lines show the $95\%$ confidence interval.}  
\label{pic:K_fit_row}
\end{figure}

\section{Summary and Future work}
\label{appdix:future_work}
This work presents the first theoretical analysis of best arm identification problem with partial feedback. 
We also show that the upper bound for the {\sc CascadeBAI($\epsilon,\delta,K$)} algorithm closely matches the lower bound in some cases. 
Empirical experiments further support the theoretical analysis. 
Moreover, the relation between the second moment and the LSG property of a bounded random variable may be of independent mathematical interest.

The assumption of $w^*<1/K$ (ensuring tightness of the sample complexity in Corollary~\ref{cor:comp_conf_up_low_two_prob}) is relevant in practical applications since CTRs are low in real applications. 
For most applications (e.g. online advertising), $K$ is small ($\approx 5$-$10$), so our assumption is reasonable.
We are cautiously optimistic that, the framework could be enhanced for better bounds in the remaining less practically relevant regime, which is an avenue for future work.

The following are some more avenues for further investigations. 
First, we note that estimating the number of observations per time step is key to establishing a high probability bound on the total number of time steps. 
In this work, we bound the expectation of this term with $w^*$ and $w’$ (Lemma~\ref{lemma:bound_obs}). This bound is tight in some cases (cf.\ Corollary~\ref{cor:comp_conf_up_low_two_prob}). These include the difficult case where all click probabilities $w(i)$'s are close to one another and hence the gaps are small. Nevertheless, a tighter bound for each individual time step may improve the results; this, however, will require a more elaborate and delicate analysis of the stochastic outcomes and their impacts at each time step. This is especially so since the order of selected items also affects the number of observations in the cascading bandit setting.
Secondly, this work focuses on the fixed-confidence setting of the BAI problem. We see that the consideration of the fixed-budget setting for cascading bandits is still not available. It is envisioned that the analysis of the statistical dependence between the number of observations and time steps would be non-trivial.
Thirdly, we envision that the analysis may be generalized to the contextual setting~\citep{soare2014best,pmlr-v80-tao18a}.

\section*{Acknowledgment}
We thank the reviewers for their insightful comments. This work is partially funded by a
National University of Singapore Start-Up Grant (R-266-000-136-133)
and
a Singapore National Research Foundation (NRF) Fellowship (R-263-000-D02-281).

%\newpage

\bibliography{cascadebairef}
\bibliographystyle{icml2020}

%%%%%%%%%%%%%%%%%%%%%%%%%%%%%%%%%%%%%%%%%%%%%%%%%%%%%%%%%%%%%%%%%%%%%%%%%%%%%%%
%%%%%%%%%%%%%%%%%%%%%%%%%%%%%%%%%%%%%%%%%%%%%%%%%%%%%%%%%%%%%%%%%%%%%%%%%%%%%%%
% DELETE THIS PART. DO NOT PLACE CONTENT AFTER THE REFERENCES!
%%%%%%%%%%%%%%%%%%%%%%%%%%%%%%%%%%%%%%%%%%%%%%%%%%%%%%%%%%%%%%%%%%%%%%%%%%%%%%%
%%%%%%%%%%%%%%%%%%%%%%%%%%%%%%%%%%%%%%%%%%%%%%%%%%%%%%%%%%%%%%%%%%%%%%%%%%%%%%%

\newpage

\onecolumn
\appendix
\section{Notations}

\begin{spacing}{1.5}
	\begin{longtable}[!h]{ p{.20\textwidth}  p{.80\textwidth} } 
           %%%%%% overall problem setting		  
           $[n]$ & set $\{1,\cdots,n\}$ for any $n\in \bbN$ \\
           $[n]^{(m)}  $ & set of all $m$-permutations of $[n]$, i.e., all ordered $m$-subset of $[n]$ for any $m\le n$ \\
		   $[L]$ & ground set of size $L$ \\
		   $w(i)$ & click probability of item $i \in [L]$ \\
		   $K$ & size of recommendation list/pulled arm \\
		   $[L]^{(K)}$ & set of all $K$-permutations of $[L]$ \\
		   $S_t$ & recommendation list/pulled arm at time step $t$\\
		   
		   $i_i^t$ & $i$-th pulled item at time step $t$\\
		   $W_t(i)$ & an r.v. that reflects whether the user clicks at item $i$ at time step $t$\\

		   $S_t^{\pi }$ & chosen arm at time step $t$  by algorithm $\pi$ \\
		   $\bmO_t^{\pi  }$ & stochastic outcome by pulling $S_t^{\pi }$ at time step $t$ by algorithm $\pi$ \\
		   $ \tilde{k}_t$ & feedback from the user at time step $t$ \\
		   
		   $\bar{w}(i)$ & equals to $1-w(i)$, i.e., one minus the click probability \\
		   $\bmw$ & the vector of click probabilities $w(i)$'s \\

		   $w^*$ & maximum click probability \\
		   $w'$ & minimum click probability \\
		   $\epsilon$ & tolerance parameter \\
		   $K'_\epsilon$ & number of $\epsilon$-optimal items (abbreviated as $K'$ for brevity)\\
		   $S^{ *}$ & optimal arm in $[K]^{ (K) } $\\

		   $\pi$ & deterministic and non-anticipatory algorithm \\
		   $\hat{S}^\pi$ & output of algorithm $\pi$ \\
		   $\calT^\pi$ & time complexity of algorithm $\pi$\\
		   $\phi^\pi$ & final recommendation rule of algorithm $\pi$\\
		   $\calF_t$ & observation history \\
		   $\delta$ & risk parameter (failure probability) \\
		   
		   $\bbT^*( \bmw, \epsilon,\delta,K )$ & optimal expected time complexity (abbreviated as $\bbT^*$) \\
		   
		   %%%%%% algorithm
		   $D_t$ & survival set in Algorithm~\ref{alg:bai_cas_conf} \\
		   $A_t$ & accept set in Algorithm~\ref{alg:bai_cas_conf} \\
		   $R_t$ & reject set in Algorithm~\ref{alg:bai_cas_conf} \\
		   $T_t(i)$ & number of observations of item $i$ by time step $t$ \\ %before time step $t$ added by $1$ \\
		   $\hat{w}_t(i)$ & empirical mean of item $i$ at time step $t$ in Algorithm~\ref{alg:bai_cas_conf}  \\
		   $k_t$ & number of $\epsilon$-optimal items to identify at time step $t$ in Algorithm~\ref{alg:bai_cas_conf}\\
		   $C_t(i,\delta)$ & confidence radius of item $i$ at time step $t$ in Algorithm~\ref{alg:bai_cas_conf}\\
		   $U_t(i,\delta)$ & upper confidence bound~(UCB) of item $i$ at time step $t$ in Algorithm~\ref{alg:bai_cas_conf}\\
		   $L_t(i,\delta)$ & lower confidence bound~(LCB) of item $i$ at time step $t$ in Algorithm~\ref{alg:bai_cas_conf}\\
		   $j^*$ & empirically $(K+1)$-th optimal item at time step $t$ in Algorithm~\ref{alg:bai_cas_conf}\\
		   $j'$ & empirically $K$-th optimal item at time step $t$ in Algorithm~\ref{alg:bai_cas_conf}\\
		   $\rho(\delta)$ & parameter used to define the confidence radius $C_t(i,\delta)$  \\

		   %%%%%% main result of upper bound
		   $c_1,c_2, \ldots$ & finite and positive universal constants whose values may vary from line to line \\
		   
		   $\Delta_i$ & gap between the click probabilities \\
		   $\bar{\Delta}_i$ & variation of $\Delta_i$ incurred by the tolerance parameter $\epsilon$ \\
		   $ \bar{T}_{i,\delta} $ & number of observations required to identify item $i$ with fixed $\delta$ and $\epsilon$ \\
		   $ \sigma(i) $ & descending order of $\bar{\Delta}_i$ of ground items \\
		   $\hatk_t$ & number of surviving items pulled at time step $t$ during the proceeding of Algorithm~\ref{alg:bai_cas_conf}\\
		   $X_{\hatk_t;t}$ & number of observations of surviving items at time step $t$\\
		   $ \mu(k,w)$ & lower bound on $\bbE X_{k;t}$  (abbreviated as $\mu_k $)\\
		   $v(k,w)$ & upper bound on $\bbE X_{k;t}^2$  (abbreviated as $v_k $)\\
		   $K_1, K_2, M_k$ & parameters used in Theorem~\ref{thm:fix_conf_up_bd} \\
		   $N_1, N_2, N_3$ & constituents in the upper bound established in Theorem~\ref{thm:fix_conf_up_bd} \\
		   %%%%%% main result of upper bound and lower bound
		   $\pi_1$ & represents Algorithm~\ref{alg:bai_cas_conf} for brevity \\
		   $\tilde{\mu}(k,w)$ & upper bound on $\bbE X_{k;t}$ (abbreviated as $\tilde{\mu}_k $) \\
		   %%%%%% proof sketch of 'partial feedback'
		   %%%%%% proof sketch of 'upper bound'
		   $\mathcal{E}(i,\delta)$ & ``nice event'' in the analysis of Algorithm~\ref{alg:bai_cas_conf} \\
		   %%%%%% proof sketch of 'lower bound'
		   $w^{(\ell)}(\ell)$ & click probability of item $\ell$ under instance $\ell~(1\le \ell \le L)$ \\
		   $S_t^{\pi, \ell  }$ & chosen arm at time step $t$  by algorithm $\pi$ under instance $\ell~(0\le \ell \le L)$\\
		   $\bmO_t^{\pi,  \ell  }$ & stochastic outcome by pulling $S_t^{\pi }$ at time step $t$ by algorithm $\pi$ under instance $\ell~(0\le \ell \le L)$ 
	\end{longtable}
\end{spacing}
	\addtocounter{table}{-1}

\section{Useful definitions and theorems}
Here are some basic facts from the literature that we will use:

\begin{restatable}
[Azuma’s Inequality for Martingales with Subgaussian Tails, implied by~\citet{shamir2011variant}]
{theorem}{theoremConcSubGaussOnesided}

\label{thm:conc_sub_gauss_onesided}
Let $\{(D_t,\calF_t)\}^\infty_{t=1}$ be a martingale difference sequence, and suppose that for any $\lambda\le 0$, we have $\bbE[e^{\lambda D_t} | \calF_{t-1}] \le e^{\lambda^2 \omega^2/2}$ almost surely. Then for all $\omega \ge 0$,
	\begin{align*}
		\Pr \left[ \sum^n_{t=1} D_t  \le - \omega \right]
		 \le 
		 \exp\left( -\frac{ \omega^2}{2\sum_{t=1}^{n} v^2_t } \right).
	\end{align*}
\end{restatable}

\begin{theorem}[Non-asymptotic law of the iterated logarithm~\citep{jamieson2014lil,jun2016top}]\label{thm:conc_log}
	Let $X_{1}, X_{2}, \ldots$ be $i . i . d .$ zero-mean sub-Gaussian random variables with scale $\sigma>0 ;$ i.e.
$\mathbb{E} e^{\lambda X_{i}} \leq e^{\frac{\lambda^{2} \sigma^{2}}{2}} .$ Let $\omega \in(0, \sqrt{1 / 6}) .$ Then,
$$\mathbb{P}\left(\forall \tau \geq 1,\left|\sum_{s=1}^{\tau} X_{s}\right| \leq 4 \sigma \sqrt{ \frac{ \log \left(\log _{2}(2 \tau) / \omega\right) }{\tau} }\right) 
\geq 
1-6 \omega^{2}.$$
\end{theorem}

\section{Influence of $\epsilon$} 
\label{appdix:influen_epsilon}

In general, a larger $\epsilon$ indicates a smaller time complexity. 
Here are two explanations. 
	(i) When $\epsilon$ grows, $K'_\epsilon$, the number of $\epsilon$-optimal items also grows. Then it should be easier to identify an $\epsilon$-optimal arm. 
	(ii) If $\epsilon$ is sufficiently large such that $K'_\epsilon \ge 2K-1$, then there are at least $K$ items left in the survival set $D_t$ before the algorithm stops. Otherwise, when $|D_t|<K$, 
	the agent pulls $|D_t|<K$ surviving items at some steps and this results in a wastage in the number of time steps.

\begin{restatable}{proposition}{propFixConfUpBdManyEpsOpt} \label{prop:fix_conf_up_bd_many_eps_opt}
	Assume $K'\ge 2K-1$.
	With probability at least $1-\delta$, Algorithm~\ref{alg:bai_cas_conf} outputs an $\epsilon$-optimal arm after at most $(c_1 N_1' + c_2 N_2' )$ steps where
\begin{align*}
    N_1' & = \frac{ 2v_K^2 }{\mu_K^2}\log \left( \frac{2}{\delta} \right) 
        =O \left(
            \frac{ v_K^2 }{\mu_K^2}\log \left( \frac{2}{\delta} \right)
            \right),\\
    N_2' & = \frac{2}{\mu_K}\left[
		\sum_{i=1}^{L-K'+K-1} \bar{T}_{\sigma(i)} + (K'-K+1) \bar{T}_{\sigma(L-K'+K)} + (K'-K)
	\right] \\
	 &   = O \left(
	        \frac{1}{\mu_K}\left\{
		\sum_{i=1}^{L-K'+K-1} \bar{\Delta}_{\sigma(i)}^{-2} \log \left[ \frac{ L}{\delta } \log \left( \frac{ L }{ \delta \bar{\Delta}_{ \sigma(i) }^{2} } \right)\right]
		+ (K'-K+1) \bar{\Delta}_{\sigma(L-K'+K)}^{-2} \log \left[ \frac{ L}{\delta } \log \left( \frac{ L }{ \delta \bar{\Delta}_{ \sigma(L-K'+K ) }^{2} } \right)
		\right]
		\right\}
	        \right).
\end{align*}
\end{restatable}

\section{Proofs of main results}
\label{appdix:pf_main_results}
In this Section, we provide proofs of 
%%%%%% 
%%%%%% simplifications of main theorem
Proposition
\ref{prop:term_by_cas}, 
Corollary
\ref{cor:compare_full_feedback},
Propositions
\ref{prop:conf_comp_full_feedback_exp},
\ref{prop:conf_upbd_two_prob},
\ref{prop:conf_upbd_two_prob_exp},
%%%%%% 
Corollary
\ref{cor:conf_lowbd_two_prob},
%%%%%% novel: partial feedback
Theorem~\ref{thm:prop_obs}, \ref{thm:sec_moment_to_left_sub_gauss}, 
Lemmas~\ref{lemma:bound_obs} --
%, 
%\ref{lemma:obs_sec_moment},
%\ref{lemma:conc_one_sub_gauss_apply_obs},
%%\ref{lemma:geo_left_sub_gau},
%%\ref{lemma:obs_sub_exp}, \ref{lemma:conc_apply_obs},
%%%%%%% 
%%%%%%% similar results to Jun's paper
%\ref{le:conc_event},
%\ref{le:suff_obs_to_identif},
%\ref{le:stop_eliminate_no},
%%%%%%%
%%%%%%% lower bound
%\ref{lemma:KLdecompOR},
\ref{lemma:lb_KLdecomp},
and complete the proof of 
Theorems~\ref{thm:fix_conf_up_bd}, \ref{thm:lb_fix_conf_dep},
Proposition~\ref{prop:fix_conf_up_bd_many_eps_opt}
in this order.

\subsection{Proof of Proposition~\ref{prop:term_by_cas} }
\label{pf:prop_term_by_cas}
\propTermCas*

\begin{proof}
According to Lemma~\ref{lemma:bound_obs} and Theorem~\ref{thm:sec_moment_to_left_sub_gauss},
\begin{align*}
    \mu_k \ge \min\{ k/2, 1/(2w^*) \},\quad
    v_k = \min \{  k, \sqrt{2}/w' \}.
\end{align*}

We upper bound $v_k/\mu_k$ and $k/\mu_k$ in two cases:

\noindent {(i)}: $0< w^* \le 1/K$: $0< w' < w^* \le 1/K$, $v_k =k $, $\mu_k \ge k/2$.
\begin{align*}
	& \frac{v_k}{ \mu_k } \le \frac{k }{ k/2 }  =2
	~\Rightarrow~
    \sum_{k=1}^{ K-K_2 }  
         \frac{ v_{K-k+1}^2 }{\mu_{K-k+1}^2} 
         \log \left( \frac{1 }{\delta} \sum_{j=1}^{ K-K_2 } 
          \frac{ v_{K-j+1}^2 }{\mu_{K-j+1}^2} 
          \right) 
    \le 
    4K
    \log \left( \frac{ 4K }{\delta}  
          \right)
          .
\end{align*}

\noindent {(ii)}: $1/K < w^* \le 1$:  $v_k \le \sqrt{2}/w'$, $\mu_k \ge 1/(2w^*)$.
\begin{align*}
	& \frac{v_k}{ \mu_k } \le \frac{  \sqrt{2}/w' }{ 1/(2w^*) } = \frac{2\sqrt{2}w^*}{ w' }
	~\Rightarrow~
    \sum_{k=1}^{ K-K_2 }  
         \frac{ v_{K-k+1}^2 }{\mu_{K-k+1}^2} 
         \log \left( \frac{1 }{\delta} \sum_{j=1}^{ K-K_2 } 
          \frac{ v_{K-j+1}^2 }{\mu_{K-j+1}^2} 
          \right) 
    \le 
     \frac{8Kw^{*2} }{w'^2 } 
    \log \left(
          \frac{8Kw^{*2} }{\delta w'^2 } 
          \right).
\end{align*}

\end{proof}
 
\subsection{Proof of Corollary~\ref{cor:compare_full_feedback} }
\label{pf:cor_compare_full_feedback}
\corCompFullFeedback*

\begin{proof}
 According to Lemma~\ref{lemma:bound_obs} and Theorem~ \ref{thm:sec_moment_to_left_sub_gauss},
\begin{align*}
    \mu_k \ge \min\{ k/2, 1/(2w^*) \},\quad
    v_k = \min \{  k, \sqrt{2}/w' \}.
\end{align*}
We first upper bound $v_k/\mu_k$ and $k/\mu_k$ in two cases:

\noindent {(i)}: $0< w^* \le 1/K$: $0< w' < w^* \le 1/K$, $v_k =k $, $\mu_k \ge k/2$.
\begin{align*}
    \frac{v_k}{ \mu_k } \le \frac{k}{k/2} = 2,\quad 
    \frac{k}{\mu_k} \le \frac{k}{k/2} = 2.
\end{align*}

\noindent {(ii)}: $1/K < w^* \le 1$:  $v_k \le \sqrt{2}/w'$, $\mu_k \ge 1/(2w^*)$.
\begin{align*}
    \frac{v_k}{ \mu_k } \le \frac{  \sqrt{2}/w' }{ 1/(2w^*) } = \frac{2\sqrt{2}w^*}{ w' },\quad
    \frac{k}{\mu_k} \le \frac{k}{ 1/(2w^*) } = 2kw^*.
\end{align*}

Next, we separate the upper bound in Theorem~\ref{thm:fix_conf_up_bd} into two parts and bound them separately:
\begin{align*}
    (A) &
          = \sum_{k=1}^{K-1}  
         \frac{  v_{K-k+1}^2 }{\mu_{K-k+1}^2} 
         \log \left( \frac{1 }{\delta} \sum_{j=1}^{K-1 } 
          \frac{  v_{K-j+1}^2 }{\mu_{K-j+1}^2} 
          \right) , \\
     (B) & =
     \frac{1}{\mu_K} \sum_{i=1}^{L-K} \bar{T}_{\sigma(i)} 
        +
           \sum_{k=1}^{ K-K_1-1 }\bar{T}_{\sigma(L-K+k)}
          \left( \frac{  K-k+1  }{ \mu_{K-k+1} } - \frac{  K-k  }{ \mu_{K-k}}  \right)
        + 
         \left( \frac{  K_1+1 }{ \mu_{K_1+1} } -2 \right) \bar{T}_{\sigma(L-K_1 ) }
         +
         2 \bar{T}_{\sigma(L-K_2 )}
\end{align*}
with
$  K_1 = \min\{\lfloor 1/w^*\rfloor, K-1 \} , \ 
			    K_2 = 1$.
			    
\noindent \textbf{\underline{Case 1}: All click probabilities $w(i)$ are at most $1/K$}. $ 1/w^* \ge K$ and $v_k/\mu_k \le 2$,  $K_1 = K-1$.
\begin{align*}
	(A) & \le 4(K-1) \log \left( \frac{4(K-1)}{\delta}  \right)  = O \left( K \log\left( \frac{K}{\delta} \right) \right), \\
	(B) & 
	\le 
	\frac{2}{ K} \sum_{i=1}^{L-K} \bar{T}_{\sigma(i)} 
        + 
         \left( \frac{  2 }{ \mu_{2} } -2 \right) \bar{T}_{\sigma(L-K+1 ) }
         +
         2 \bar{T}_{\sigma(L-1 )}
    = O \left(
    \frac{1}{ K} \sum_{i=1}^{L-K} \bar{T}_{\sigma(i)} 
        + \bar{T}_{\sigma(L-1 )}
    \right).
\end{align*}

\noindent \textbf{\underline{Case 2}: All click probabilities $w(i)$ are at least $1/2$}. $ 1/w^* \le K$ implies $v_k/\mu_k \le 4\sqrt{2}$,  $K_1 \ge 1$.
\begin{align*}
	(A) & \le 32(K-1) \log \left( \frac{32(K-1)}{\delta}  \right)  = O \left( K \log\left( \frac{K}{\delta} \right) \right), \\
	(B) 
	& \le 
	2 \sum_{i=1}^{L-K} \bar{T}_{\sigma(i)} 
        +
           \sum_{k=1}^{ K-2 }\bar{T}_{\sigma(L-K+k)}
          \left( \frac{  K-k+1  }{ \mu_{K-k+1} } - \frac{  K-k  }{ \mu_{K-k}}  \right)
        + 
         \frac{  2 }{ \mu_{2} } \bar{T}_{\sigma(L-1 ) } \\
    & = 
    2 \sum_{i=1}^{L-K} \bar{T}_{\sigma(i)} 
        +
           \sum_{k=0}^{ K-3 }\frac{  K-k   }{ \mu_{K-k } }\bar{T}_{\sigma(L-K+k+1)}
        - 
        	\sum_{k=1}^{ K-2 } \frac{  K-k  }{ \mu_{K-k}} \bar{T}_{\sigma(L-K+k)}
        + 
         \frac{  2 }{ \mu_{2} } \bar{T}_{\sigma(L-1 ) } \\
    & = 
    2 \sum_{i=1}^{L-K} \bar{T}_{\sigma(i)} 
        +
        \frac{K}{\mu_K} \bar{T}_{\sigma(L-K +1)}
        +
           \sum_{k=0}^{ K-2 }\frac{  K-k   }{ \mu_{K-k } }\left[ \bar{T}_{\sigma(L-K+k+1)}
        -   \bar{T}_{\sigma(L-K+k)} \right] \\
    & \le 
    2 \sum_{i=1}^{L-K} \bar{T}_{\sigma(i)} 
        +
        \frac{K}{\mu_K} \bar{T}_{\sigma(L-K +1)}
        +
           \sum_{k=0}^{ K-2 }2(K-k) \cdot \left[ \bar{T}_{\sigma(L-K+k+1)}
        -   \bar{T}_{\sigma(L-K+k)} \right] \\
	& \le 
	2 \sum_{i=1}^{L-K} \bar{T}_{\sigma(i)} 
        +
           \sum_{k=1}^{ K-2 }\bar{T}_{\sigma(L-K+k)}
          \left[ 2( K-k+1)  - 2(K-k)  \right]
        + 
         4 \bar{T}_{\sigma(L-1 ) } \\
    & = O \left(
    \sum_{i=1}^{L-1} \bar{T}_{\sigma(i)} 
    \right).
\end{align*}

Recall that when $\epsilon=0$,
\begin{align*}
	\barT_i = O \left(
		\Delta_i^{-2} \log
		\left(
		\frac{L}{\delta} \log \left( \frac{L}{\delta \Delta_i^{2} } \right)
		\right)
	\right)
\end{align*}
for all $i\in [L]$. Altogether, we complete the proof.

\end{proof}

\subsection{Proof of Proposition~\ref{prop:conf_comp_full_feedback_exp}}
\label{pf:prop_conf_comp_full_feedback_exp}

\propConfCompFullFeedbackExp*

\begin{proof}
     
     \textbf{ (i) Consider all click probabilities $w(i)'$'s are at most $1/K$.} 
For any $0<\delta'\le \delta$, revisit the proof and result of Theorem~\ref{thm:fix_conf_up_bd}. 
First, Lemma~\ref{le:conc_event} implies that $\bbP \left( \bigcap_{i=1}^L \calE(\epsilon,\delta' ) \right) \ge 1-\delta'/2$.
    Assume $ \bigcap_{i=1}^L \calE(\epsilon,\delta' ) $ holds from now on. 
Secondly, Lemma~\ref{le:suff_obs_to_identif} indicates that Algorithm~\ref{alg:bai_cas_conf} can correctly identify item $i$ after $\barT_{i,\delta}$ observations.
Thirdly, similar to the discussion in Section~\ref{sec:main_pf_ub}, we set $\sum_{k=1}^{K-1} \delta_k \le \delta'/2$.
Additionally applying the analysis in Proposition~\ref{prop:term_by_cas} and Corollary~\ref{cor:compare_full_feedback}, with probability at least $1-\delta'$, we can bound the time complexity by
\begin{align*}
   & \sum_{k=1}^{ K-1}  
         \frac{ v_{K-k+1}^2 }{\mu_{K-k+1}^2} 
         \log \left( \frac{ 1 }{\delta‘} \sum_{j=1}^{ K-K_2 } 
          \frac{ v_{K-j+1}^2 }{\mu_{K-j+1}^2} 
          \right)
          + K + \frac{8}{K} \sum_{i=1}^{L-K} \barT_{\sigma(i)} + 8\barT_{ \sigma(L-1) } \\
    & \le 
    4K \log \left( \frac{4K}{\delta'} \right) + K + 
    \frac{8}{ K}
    \sum_{i=1}^{L-K} 
    \frac{217}{ \Delta_{\sigma(i)}^{2} } \log
		\left[
		\frac{24L}{\delta'} \log_2 \left( \frac{ 648\times 12L}{\delta' \Delta_{\sigma(i)}^{2} } \right)
		\right]
	+
	\frac{7136}{ \Delta_{\sigma(L-1)}^{2} } \log
		\left[
		\frac{24L}{\delta'} \log_2 \left( \frac{ 648\times 12 L}{\delta' \Delta_{\sigma(L-1)}^{2} } \right)
		\right] \\
		& \le 
	 5K \log \left( \frac{4K}{\delta'} \right)+ 
    \frac{1}{ K}
    \sum_{i=1}^{L-K} 
    \frac{10600}{ \Delta_{\sigma(i)}^{2} } \log
		\left[
		\frac{ L}{\delta'} \log_2 \left( \frac{  L}{\delta' \Delta_{\sigma(i)}^{2} } \right)
		\right]
	+
	\frac{10600 }{ \Delta_{\sigma(L-1)}^{2} } \log
		\left[
		\frac{ L}{\delta'} \log_2 \left( \frac{   L}{\delta' \Delta_{\sigma(L-1)}^{2} } \right)
		\right] \\
    & \le 
       10610 \log \left( \frac{1}{ \delta'^2 } \right) \cdot
           \left\{
         \frac{1}{ K}
    \sum_{i=1}^{L-K} 
    \Delta_{\sigma(i)}^{-2} \log
		\left[
		L \log \left( \frac{L}{  \Delta_{\sigma(i)}^{2} } \right)
		\right]
	+
	\Delta_{\sigma(L-1)}^{-2} \log
		\left[
		L\log \left( \frac{L}{ \Delta_{\sigma(L-1)}^{2} } \right)
		\right]
		\right\}.
\end{align*}    
In short, set 
\begin{align*}
    A = 21220 
    \left\{
         \frac{1}{ K}
    \sum_{i=1}^{L-K} 
    \Delta_{\sigma(i)}^{-2} \log
		\left[
		L \log \left( \frac{L}{  \Delta_{\sigma(i)}^{2} } \right)
		\right]
	+
	\Delta_{\sigma(L-1)}^{-2} \log
		\left[
		L\log \left( \frac{L}{ \Delta_{\sigma(L-1)}^{2} } \right)
		\right]
		\right\},
\end{align*}
then $  \Pr( \calT >  -A \log \delta'  ) < \delta'$ for any $ 0\le \delta' \le \delta$.

Meanwhile, Tonelli's Theorem implies that
\begin{align*}
    \bbE \calT 
    = \bbE \left[ \int_0^\calT 1 ~\rmd x  \right]
    = \bbE \left[ \int_0^{+\infty} \bbI (\calT>x) ~\rmd x  \right]
    = \int_0^{+\infty}\bbE \left[  \bbI (\calT>x)  \right] ~\rmd x
    = \int_0^{+\infty} \bbP(\calT>x) ~\rmd x.
\end{align*}
Since $x= -A \log \delta$ implies $\delta = e^{-x/A}$ and
    $\int_0^{+\infty} e^{-x/A} ~\rmd x
    =  A e^{-x/A}|_{x=+\infty}^0 = A$,
\begin{align*}
    & \bbE \calT \le 
    \int_0^{ -A \log \delta } 1 ~\rmd x
    +
    \int_{ -A \log \delta }^{+\infty} \bbP(\calT>x) ~\rmd x  
    \le 
    -A \log \delta + 
    \int_{ 0 }^{+\infty} \bbP(\calT>x) ~\rmd x   
    = 
    -A \log \delta + A
    \\
    & \le 
     42440  \log \left( \frac{1}{\delta} \right) \cdot 
    \left\{
         \frac{1}{ K}
    \sum_{i=1}^{L-K} 
    \Delta_{\sigma(i)}^{-2} \log
		\left[
		L \log \left( \frac{L}{  \Delta_{\sigma(i)}^{2} } \right)
		\right]
	+
	\Delta_{\sigma(L-1)}^{-2} \log
		\left[
		L\log \left( \frac{L}{ \Delta_{\sigma(L-1)}^{2} } \right)
		\right]
		\right\}
		.
\end{align*}

\textbf{ (ii) Consider all click probabilities $w(i)$'s are at least $1/2$.} 
The analysis is similar to that in Case (i). With the analysis in Theorem~\ref{thm:fix_conf_up_bd} and results in Proposition~\ref{prop:term_by_cas} and Corollary~\ref{cor:compare_full_feedback}, for any $0<\delta'\le \delta$, with probability at least $1-\delta'$, the time complexity is at most
\begin{align*}
   & \sum_{k=1}^{ K-1}  
         \frac{ v_{K-k+1}^2 }{\mu_{K-k+1}^2} 
         \log \left( \frac{ 1 }{\delta'} \sum_{j=1}^{ K-1 } 
          \frac{ v_{K-j+1}^2 }{\mu_{K-j+1}^2} 
          \right)
          + K +  8\sum_{i=1}^{L-1} \barT_{\sigma(i)}   \\
    & \le 
    \frac{ 8Kw^{*2} }{w'^2 }  \log \left(
            	\frac{ 8Kw^{*2} }{\delta w'^2 } 
            	  \right) + K + 
    \sum_{i=1}^{L-1} 
    \frac{8 \times 217}{ \Delta_{\sigma(i)}^{2} } \log
		\left[
		\frac{24L}{\delta'} \log_2 \left( \frac{ 648\times 12L}{\delta' \Delta_{\sigma(i)}^{2} } \right)
		\right] \\
		& \le 
	 32K  \log \left(
            	\frac{  32K }{\delta  }  \right) + K + 
   \sum_{i=1}^{L-1} 
    \frac{ 10598 }{ \Delta_{\sigma(i)}^{2} } \log
		\left[
		\frac{ L}{\delta'} \log_2 \left( \frac{  L}{\delta' \Delta_{\sigma(i)}^{2} } \right)
		\right] \\
    & \le 
       10630 
    \sum_{i=1}^{L-1} 
    \Delta_{\sigma(i)}^{-2} \log
		\left[
		L \log \left( \frac{L}{  \Delta_{\sigma(i)}^{2} } \right)
		\right]
		\log \left( \frac{1}{ \delta'^2 } \right) 
		.
\end{align*}    
Set $A = 21260 \sum_{i=1}^{L-1} 
    \Delta_{\sigma(i)}^{-2} \log
		\left[
		L \log \left( \frac{L}{  \Delta_{\sigma(i)}^{2} } \right)
		\right]$,
then $  \Pr( \calT >  -A \log \delta'  ) < \delta'$ for any $ 0\le \delta' \le \delta$. Lastly,
\begin{align*}
    \bbE \calT \le 
    42520
    \sum_{i=1}^{L-1} 
    \Delta_{\sigma(i)}^{-2} \log
		\left[
		L \log \left( \frac{L}{  \Delta_{\sigma(i)}^{2} } \right)
		\right]
		\log \left( \frac{1}{ \delta  } \right) .
\end{align*}

\end{proof}

\subsection{Proof of Proposition~\ref{prop:conf_upbd_two_prob}  } 
\label{pf:prop_conf_upbd_two_prob}

\propConfUpbdTwoprob*

\begin{proof}
\textbf{We first remind ourselves how the algorithm proceeds.}
In this instance,  $\epsilon=0$ yields $K^*=1$. 
For any item $i \in [L]$, $\bar{\Delta}_i=\Delta_i=w^*-w'$. 
And according to Lemma~\ref{le:suff_obs_to_identif},  item $i$ will be correctly classified with high probability after $\barT_i$ observations where $\rho = \delta/(12L)$,
	\begin{align*}
        \barT_{i,\delta} = \barT_{(w,\delta)} = \barT_{(w)} 
        &  = 
        1 + \left\lfloor \frac{216}{   (w^*-w')^2 }  \log \left( \frac{2}{ \rho} \log_2 \left( \frac{ 648 }{ \rho (w^*-w')^2 } \right) \right)  \right\rfloor  \\
       &  = O\left(
         \frac{ 1 }{   (w^*-w')^2 }  \log \left[ \frac{L}{ \delta } \log_2 \left( \frac{ L }{ \delta (w^*-w')^2 } \right) \right]
        \right)        
        .
    \end{align*}
This implies that each item requires the same number of observations to be correctly identified. According to the design of algorithm, $ T_t(j)-1 \le T_t(i) \le T_t(j)+1 $ for any remaining items $i \neq j$.
Therefore, the worst case is as follows: 
\begin{itemize}
	\item the agent observes one item for $\barT_{(w)}$ times and the others for $\barT_{(w)}-1$ times after $t'$ steps, and identifies one item per step for the subsequent $L-2$ steps.
\end{itemize} 

Therefore, we now turn to upper bounding the number of steps required to eliminate an item for the first time. 
According to Lemma~\ref{lemma:conc_one_sub_gauss_apply_obs}, we
 set $\delta_0 = \delta/2$, $k=K$, $n=t'$, 
 $\omega_K'= - \sqrt{-2t' v_K^2 \log (\delta/2)}$.
  Then the total number of observations during $t'$ steps should be larger than $t' \mu_K + \omega_K'$ with probability at least $1-\delta/2$. The number of observations can be upper bounded as follows:
\begin{align*}
	t' \mu_K + \omega_K'	
	& \le \bar{T}_{(w)} + (L-1)[ \bar{T}_{(w)} - 1]
	= L\cdot\bar{T}_{(w)} - L + 1.
\end{align*} 
Then with Lemma~\ref{le:conc_event} and its ensuing discussion in Section~\ref{sec:main_pf_ub},
with probability at least $1-\delta$, the time complexity is upper bounded by 
\begin{align}
	& \frac{2 (L\cdot\bar{T}_{(w)} - L + 1)}{\mu_K} 
	+  \frac{2v_K^2 }{\mu_K^2}\log \left(\frac{2}{\delta}\right).
	 \nonumber 
\end{align}

\textbf{Next, we consider how the values of $w^*$ and $w'$ affect the bound.}
 According to Lemma~\ref{lemma:bound_obs} and Theorem~ \ref{thm:sec_moment_to_left_sub_gauss},
\begin{align*}
    \mu_k \ge \min\{ k/2, 1/(2w^*) \},\quad
    v_k = \min \{  k, \sqrt{2}/w' \}.
\end{align*}
We discuss two cases separately:

\noindent \underline{Case 1}: $0<w^*\le 1/K$: $0< w' < w^* \le 1/K$, $v_K =K $, $\mu_K \ge K/2$. The upper bound becomes:
\begin{align*}
    \frac{4 (L\cdot\bar{T}_{(w)} - L + 1)}{ K} 
	+  \frac{ 8  K^2 }{ K^2}\log \left(\frac{2}{\delta}\right)
	=
	O\left(
         \frac{  L }{   K(w^*-w')^2 }  \log \left[ \frac{L}{ \delta } \log_2 \left( \frac{ L }{ \delta (w^*-w')^2 } \right) \right]
        \right).
\end{align*}

\noindent \underline{Case 2}: $1/K < w^* \le 1$:  $v_k \le \sqrt{2}/w'$, $\mu_k \ge 1/(2w^*)$. The bound becomes
\begin{align*}
    & \frac{2 (L\cdot\bar{T}_{(w)} - L + 1)}{  1/(2w^*) } 
	+  \frac{ 4 /w'^2 }{ 1/(2w^*)^2}\log \left(\frac{2}{\delta}\right) 
	=
	O\left(
         \frac{ w^*L }{   (w^*-w')^2 }  \log \left[ \frac{L}{ \delta } \log_2 \left( \frac{ L }{ \delta (w^*-w')^2 } \right) \right]
         +
         \left( \frac{w^{*}}{ w' } \right)^2 \cdot \log \left( \frac{1}{\delta} \right)
        \right).
\end{align*}

\end{proof}

\subsection{Proof of Proposition~\ref{prop:conf_upbd_two_prob_exp}}
\label{pf:prop_conf_upbd_two_prob_exp}

\propConfUpbdTwoprobExp*

\begin{proof}
     For any $0<\delta'\le \delta$, we revisit the proof of Proposition~\ref{prop:conf_upbd_two_prob}. 
     Firstly, Lemma~\ref{le:conc_event} implies that $\bbP \left( \bigcap_{i=1}^L \calE(\epsilon,\delta' ) \right) \ge 1-\delta'/2$.
    Assume $ \bigcap_{i=1}^L \calE(\epsilon,\delta' ) $ holds from now on. 
    Secondly, Lemma~\ref{le:suff_obs_to_identif} implies that the agent can identify any item correctly after 
\begin{align*}
    \barT_{ (w,\delta' ) } =
    1 + \left\lfloor \frac{216}{  (w^*-w')^2 }  \log \left( \frac{24L}{ \delta' } \log_2 \left( \frac{ 648*24L }{ \delta' (w^*-w')^2 } \right) \right)  \right\rfloor
    \le 
    \frac{1320 }{  (w^*-w')^2 }  \log \left[ \frac{ L}{ \delta' } \log  \left( \frac{  L }{ \delta' (w^*-w')^2 } \right) \right]
\end{align*}
    observations. 
Then with analysis similar to Appendix~\ref{pf:prop_conf_upbd_two_prob}, we can upper bound the time complexity of Algorithm~\ref{alg:bai_cas_conf} with probability $1-\delta'$.

\noindent \underline{Case 1}: $0<w^*\le 1/K$: 
with probability at least $1-\delta'$, the time complexity is upper bounded by 
\begin{align*}
    & \frac{4 L\cdot\bar{T}_{(w,\delta')} }{ K} 
	+  8 \log \left(\frac{2}{\delta'}\right)
	\le 
	\frac{5288L }{ K (w^*-w')^2 }  \log \left[ \frac{ L}{ \delta' } \log  \left( \frac{  L }{ \delta' (w^*-w')^2 } \right) \right] \\
	& \le 
	\frac{ 10576L }{ K (w^*-w')^2 }  \log \left[ L \log  \left( \frac{  L }{   (w^*-w')^2 } \right) \right] \log (\frac{1}{\delta'})
	:= -A \log \delta'.
\end{align*}
Then for any $0<\delta' \le \delta$, $ \Pr(  \calT > -A \log \delta' )< \delta'$.
Meanwhile, Tonelli's Theorem implies that
\begin{align*}
    \bbE \calT 
    = \bbE \left[ \int_0^\calT 1 ~\rmd x  \right]
    = \bbE \left[ \int_0^{+\infty} \bbI (\calT>x) ~\rmd x  \right]
    = \int_0^{+\infty}\bbE \left[  \bbI (\calT>x)  \right] ~\rmd x
    = \int_0^{+\infty} \bbP(\calT>x) ~\rmd x.
\end{align*}
Since $x= -A \log \delta$ implies $\delta = e^{-x/A}$ and
    $\int_0^{+\infty} e^{-x/A} ~\rmd x
    =  A e^{-x/A}|_{x=+\infty}^0 = A$,
\begin{align*}
    & \bbE \calT \le 
    \int_0^{ -A \log \delta } 1 ~\rmd x
    +
    \int_{ -A \log \delta }^{+\infty} \bbP(\calT>x) ~\rmd x  
    \le 
    -A \log \delta + 
    \int_{ 0 }^{+\infty} \bbP(\calT>x) ~\rmd x   
    = 
    -A \log \delta + A
    \\
    & \le 
     \frac{ 21152 L }{ K (w^*-w')^2 }  \log \left[ L \log  \left( \frac{  L }{   (w^*-w')^2 } \right) \right] \log (\frac{1}{\delta})
		.
\end{align*}

\noindent \underline{Case 2}: $1/2 \le w' < 1$ or $w^*/w' \le 2$:  with a similar analysis, for any $0<\delta \le \delta'$, with
\begin{align*}
    & 4 w^* L \bar{T}_{(w,\delta')}  
	+  16 \left( \frac{w^{*}}{ w' } \right)^2 \log \left(\frac{2}{\delta'}\right) 
	\le 
	\frac{ 5280w^*L }{  (w^*-w')^2 }  \log \left[ \frac{ L}{ \delta' } \log  \left( \frac{  L }{ \delta' (w^*-w')^2 } \right) \right] 
	+ 
	64 \log \left( \frac{1}{\delta'} \right)\\
	& \le 
	\frac{ 10624w^*L }{  (w^*-w')^2 }  \log \left[ L \log  \left( \frac{  L }{  (w^*-w')^2 } \right) \right] \log \left(  \frac{1}{\delta'} \right)
	:= -A \log \delta'
\end{align*}
$\Pr( \calT > -A\log \delta' ) < \delta'$. Lastly,
\begin{align*}
    \bbE \calT \le 
    \frac{ 21248 w^*L }{  (w^*-w')^2 }  \log \left[ L \log  \left( \frac{  L }{  (w^*-w')^2 } \right) \right] \log \left(  \frac{1}{\delta} \right).
\end{align*}

\end{proof}

\subsection{Proof of Corollary~\ref{cor:conf_lowbd_two_prob} }
\label{pf:cor_conf_lowbd_two_prob}
\corConfLowbdTwoprob*

\begin{proof}
	First, by setting $w(i)=w^*$ for all $1\le i\le K$ and $w(j)=w'$ for all $k<j\le L$, the result in Theorem~\ref{thm:lb_fix_conf_dep} becomes
	\begin{align*}
		&  
		 \frac{\mathrm{KL}(1-\delta,\delta)}{\tilde{\mu_K}}
		 \cdot
		 \left[
		   \frac{K}{ \mathrm{KL} ( w^*, w' ) }
		   + \frac{L-K}{ \mathrm{KL} (  w', w^* ) }
	\right] 
		\ge 
		 \frac{\log (1/2.4\delta)}{\tilde{\mu_K}}
		 \cdot
		 \left[
		   \frac{K}{ \mathrm{KL} ( w^*, w' ) }
		   + \frac{L-K}{ \mathrm{KL} (  w', w^* ) }
	\right] .
	\end{align*}

	Next, according to Pinsker's and reverse Pinsker's inequality for any two distributions $ P$ and $ Q$ defined in the same finite space $X$ we have
\begin{align*}
	\delta(P,Q)^2 \le \frac{1}{2} \mathrm{KL}(P,Q)
	\le \frac{1}{\alpha_Q} \delta(P,Q)^2
\end{align*}
where
$
	\delta(P,Q) = \sup\{ | P(A)-Q(A)| \big| A \subset X \} \text{ and } \alpha_Q = \min_{x\in X: Q(x)>0} Q(x)
$.
In our case, set $\delta(w^*,w') = (w^*-w')^2$ and $\alpha = \min \{ w',w^*,1-w^*,1-w' \} = \min \{ w' ,1-w^* \}$, we have
\begin{align*}
	& (w^*-w')^2 \le \frac{1}{2} \mathrm{KL}(w^*,w')
	\le \frac{1}{\alpha } (w^*-w')^2
	= \frac{ 1 }{ \min \{ w' ,1-w^* \} } (w^*-w')^2,\\
	& (w^*-w')^2 \le \frac{1}{2} \mathrm{KL}(w',w^*)
	\le \frac{1}{\alpha } (w^*-w')^2
	= \frac{ 1 }{ \min \{ w' ,1-w^* \} } (w^*-w')^2.
\end{align*}
Further since $\tilde{\mu}_K \le 1/w'$ as stated by Lemma~\ref{lemma:bound_obs}, the lower bound becomes
\begin{align*}
	\Omega\left( 
	\min \{ w' ,1-w^* \} \cdot 
	\frac{L w' }{ ( w^*-w' )^2 } 
	\log \left[ \frac{1}{ \delta} \right] 
	\right)
	.
\end{align*}
\end{proof}

\subsection{Proof of Theorem~\ref{thm:prop_obs}}
\label{pf:thm_prop_obs}

\thmObsProp*
\begin{proof}
	(i) 
	Consider any ordered set $I=( i_1^{I}, \ldots, i_k^{I} )$.  To show ${\mu}_k( \bmu,I_{ \textrm{dec} }) \le {\mu}_k(  \bmu,I)\le {\mu}_k(\bmu,I_{ \textrm{inc} })$, it is sufficient to show the following:
	\begin{itemize}
		\item[$(*)$:] If there exists $1\le m < k$ such that $u_{ i_m^I }<u_{i_{m+1}^I }$, we can change their positions to get $I'$ and have ${\mu}_k(\bmu,I) > \mu_k(\bmu,I')$.
	\end{itemize}
The proof of $(*)$ is as follows:
	\begin{align*}
		\text{if}~ 1\le m < k-1, & \\ 		
		\mu_k(\bmu,I) - \mu_k(\bmu,I') 
		& = m\cdot \prod_{j=1}^{m-1} (1-u_{i_j^I })(u_{ i_m^I }-u_{i_{m+1}^I } ) 
		+ (m+1)\cdot \prod_{ j=1}^{m-1} 
		  (1-u_{i_j^I} )[ u_{i_{m+1}^I }(1-u_{i_m^I }) - u_{i_m^I }(1-u_{i_{m+1}^I }) ] \\
		&  = -\prod_{j=1}^{m-1} (1-u_{i_j^I })(u_{ i_m^I }-u_{i_{m+1}^I } ) > 0 ; \\
		\text{if}~ m = k-1, \hspace{1.8em}& \\
		 \mu_k(\bmu,I) - \mu_k(\bmu,I') 
		& = m\cdot \prod_{j=1}^{m-1} (1-u_{i_j^I })(u_{ i_m^I }-u_{i_{m+1}^I } ) 
		+ (m+1)\cdot \prod_{ j=1}^{m-1} 
		  (1-u_{i_j^I} )[  (1-u_{i_m^I }) -  (1-u_{i_{m+1}^I }) ] \\
		& = -\prod_{j=1}^{m-1} (1-u_{i_j^I })(u_{ i_m^I }-u_{i_{m+1}^I } ) > 0 .
	\end{align*}

(ii) 
To show the monotonicity, it is sufficient to show the following:
\begin{itemize}
	\item[$(\#)$:] Set two sets of click probabilities $u$, $v$ such that $v_{i_m^I}>u_{i_m^I}$ for some $1\le m \le k$ and $v_{i_j^I}=u_{i_j^I}$ for $j \neq m$. Then we have $\mu_k(\bmu,I) \ge \mu_k(\bmv,I)$.
\end{itemize}
Here is the proof of $(\#)$.
If $m=k$, then obviously we have $ {\mu}_k(\bmu,I)={\mu}_k(\bmv,I)$. If $1\le m <k$, we exchange positions of the $m$-th and $(m+1)$-th item to get a new ordered set $I_1$, then
\begin{align*}
	& \mu_k(\bmu,I) - \mu_k(\bmu,I_1) 
		= -\prod_{j=1}^{m-1} (1-u_{i_j^I })(u_{ i_m^I }-u_{i_{m+1}^I } ), \ \
	 \mu_k(\bmv,I) - \mu_k(\bmv,I_1) 
		= -\prod_{j=1}^{m-1} (1-u_{i_j^I })(v_{ i_m^I }-u_{i_{m+1}^I } ).
\end{align*}
Hence
\begin{align*}
	 \mu_k(\bmu,I) - \mu_k(\bmv,I)
	& = [ \mu_k(\bmu,I) - \mu_k(\bmu,I_1) ] 
	 	- [ \mu_k(\bmv,I) - \mu_k(\bmv,I_1) ] + \mu_k(\bmu,I_1) - \mu_k(\bmv,I_1)\\
	& = -\prod_{j=1}^{m-1} (1-u_{i_j^I })(u_{ i_m^I }-v_{i_{m}^I } ) + \mu_k(\bmu,I_1) - \mu_k(\bmv,I_1) \\
	& > \mu_k(\bmu,I_1) - \mu_k(\bmv,I_1).
\end{align*}
If $m+1<k$, note that the only difference between $(\bmu, I_1)$ and $(\bmv, I_1)$ now lies in the click probability of the $(m+1)$-th item. In detail,
\begin{align*}
	u_{i_{m+1}^{I_1} } = u_{i_m^I},\ 
	v_{i_{m+1}^{I_1} } = v_{i_m^I}
	~\text{ and } ~
	u_{i_j^{I_1}} = v_{i_j^{I_1}},\ \forall j\ne m+1.
\end{align*}
We exchange positions of the $(m+1)$-th and $(m+2)$-th item in $I_1$ to get a new ordered set $I_2$. Similarly we have
\begin{align*}
	 \mu_k(\bmu,I_1) - \mu_k(\bmv,I_1)
	& = [ \mu_k(\bmu,I_1) - \mu_k(\bmu,I_2) ] 
	 	- [ \mu_k(\bmv,I_1) - \mu_k(\bmv,I_2) ] + \mu_k(\bmu,I_2) - \mu_k(\bmv,I_2)\\
	& = -\prod_{j=1}^{m} (1-u_{i_j^{I_1} })(u_{ i_{m+1}^{I_1} }-v_{i_{m+1}^{I_1} } ) + \mu_k( \bmu,I_2) - \mu_k( \bmv,I_2) \\
	& = -\prod_{j=1}^{m} (1-u_{i_j^{I_1} })(u_{ i_{m}^{I} }-v_{i_{m}^{I} } ) + \mu_k( \bmu,I_2) - \mu_k( \bmv,I_2) \\
	& > \mu_k( \bmu,I_2) - \mu_k( \bmv,I_2).
\end{align*}
We can continue this operation for $n=k-m$ times and get $I_n$. Iteratively, we have $ \mu_k( \bmu,I) - \mu_k( \bmv,I) \ge \mu_k( \bmu,I_n) - \mu_k( \bmv,I_n) $. Besides, the only difference between $( \bmu, I_n)$ and $( \bmv, I_n)$ now lies in the click probability of the $k$-th item:
\begin{align*}
	u_{i_{k}^{I_n} } = u_{i_m^I},\ 
	v_{i_{k}^{I_n} } = v_{i_m^I}
	~\text{ and } ~
	u_{i_j^{I_n}} = v_{i_j^{I_n}},\ \forall j\ne k.
\end{align*}
Since $\mu_k( \bmu,I_n) = \mu_k( \bmv,I_n)$, $\mu_k( \bmu,I) \ge \mu_k( \bmv,I)$.

\end{proof}

\subsection{Proof of Lemma~\ref{lemma:bound_obs} }
\label{pf:lemma_bound_obs}
\lemmaBoundObs*

\begin{proof}

\textbf{Lower bound.} According to Lemma~\ref{thm:prop_obs}, the expectation of observations attains its minimum when we pull an ordered set $\{ 1, 2, \ldots,k \}$, and attains its maximum when we pull an ordered set $\{ L-k+1, L-k+2, \ldots,L \}$. In other words, depending on the instance, the expectation of observations can be lower bounded as follows:
\begin{align*}
	\mu_k = \mu(k,w)
	=  \sum_{i=1}^{k-1}  i \cdot w(i) \cdot \prod_{ j=1 }^{i-1} [1-w(j)]  + k\cdot \prod_{j=1}^{k-1} [1-w(j)] .
\end{align*}

Moreover, since the lower bound ${\mu}_k$ is larger than the expectation of observations when $w(i)=w^*$ for all $1\le i\le k$ or we pull item $1$ for $K$ times~(note that this is not allowed in Algorithm~\ref{alg:bai_cas_conf}), we can utilize only $w^*$ to lower bound the expectation:
\begin{align}
    \mu_k & \ge \sum_{i=1}^{k-1}  i \cdot w^* (1-w^*) ^{i-1}  + k  (1-w^*)^{k-1} 
    := g(w^*)  
    \nonumber
\end{align}
then
\begin{align*}
	g(w) & = w + 2w(1-w) + \ldots + (k-1)w(1-w)^{k-2} + k(1-w)^{k-1} \\
	(1-w) \cdot g(w) & = w(1-w) + 2w(1-w)^2 + \ldots + (k-1)w(1-w)^{k-1} + k(1-w)^{k} \\
	w\cdot g(w) & = w + w(1-w) + w(1-w)^2 + \ldots + w(1-w)^{k-2} + [ k - (k-1)w](1-w)^{k-1} - k(1-w)^{k} \\
	w\cdot g(w) & = w + w(1-w) + w(1-w)^2 + \ldots + w(1-w)^{k-2} + ( k - kw+w -k + kw )(1-w)^{k-1}  \\
	w\cdot g(w) & = w \cdot \frac{1 - (1-w)^k}{w} \\ % - k(1-w)^k + k(1-w)(1-w)^{k-1} \\
	g(w) & = \frac{1 - (1-w)^k}{w} . %- k(1-w)^{k-1} .
\end{align*}
Let $w^* = k^{-\beta} \in [0,1]$, then $\beta\ge 0$. 
Since $(1-1/x)^x$ is a nondecreasing function of $x$ and $\lim_{x\rightarrow \infty}(1-1/x)^x = 1/e$, $k^{1-\beta}\ge 0$,
\begin{align*}
     g(w^*) = \frac{1 - (1-w^*)^ k }{w^*}
        = \frac{1 - (1-k^{-\beta} )^{ k^\beta\cdot k^{1-\beta}} }{ k^{-\beta} }
        \ge  k^\beta \cdot \left( 1 - e^{ -k^{1-\beta}}  \right).
\end{align*}

If $\beta\ge 1$, let $f(x)=e^{-x}$, then $f^{(n)}(x)=(-1)^n\cdot e^{-x}$. For any $x\ge 0$,  there exists $y\in[0,x]$ such that
\begin{align*}
    f(x) = f(0) + f'(0)\cdot x + \frac{1}{2!} f^{(2)}(0)\cdot x^2 + \frac{1}{3!} f^{(3)}(0)\cdot y^3
    = 1 -x + \frac{x^2}{2} - \frac{y^3}{3} \le 1 -x + \frac{x^2}{2}.
\end{align*}
This leads to $1-e^{-x} \ge x (1-x/2)$ and
    \begin{align*}
        g(w^*) 
        \ge k^\beta \cdot k^{1-\beta}( 1 - k^{1-\beta}/2 )
        \ge k(1-1/2)
        = k/2.
    \end{align*}
Otherwise, $0\le \beta <1$. Since
\begin{align*}
    \beta \nearrow ~\Rightarrow~ 1-\beta \searrow ~\Rightarrow~ k^{1-\beta} \searrow ~\Rightarrow~
    {-k^{1-\beta}} \nearrow ~\Rightarrow~ e^{-k^{1-\beta}} \nearrow ~\Rightarrow~
    1 - e^{-k^{1-\beta} } \searrow
    ~,
\end{align*}
$1 - e^{-k^{1-\beta} }$ decreases as $\beta$ increases. Then,
    \begin{align*}
       g(w^*)  
       \ge k^\beta \cdot ( 1 - e^{-k^{1-1} } )
       = k^\beta \cdot ( 1- e^{-1} )
        \ge k^\beta \cdot (1-1/2)
        = k^\beta/2.
    \end{align*}
Altogether, $\mu_k \ge \min\{ k/2, k^\beta/2 \} = \min\{  k/2, 1/(2w^*)  \}$.

\textbf{Upper bound.} Similarly we can see that the expectation of observations attains its maximum when we pull an ordered set $\{ L, L-1, \ldots, L-k+1 \}$, and therefore upper bounded by
\begin{align*}
	\tilde{\mu}_k = \tilde{\mu}(k,w)
	=  \sum_{i=1}^{k-1}  i \cdot w(L+1-i) \cdot \prod_{ j=1 }^{i-1} [1-w(L+1-j)]  + k\cdot \prod_{j=1}^{k-1} [1-w(L+1-j)] .
\end{align*}

Furthermore, the upper bound $\tilde{\mu}_k$ is smaller than the expectation of  observations when $w(j)=w'$ for all $L-k+1 \le j\le L$ or we pull item $L$ for $K$ times~(again note that this is not allowed in Algorithm~\ref{alg:bai_cas_conf}):
\begin{align*}
    \tilde{\mu}_k \le \sum_{i=1}^{k-1}  i \cdot w' (1-w') ^{i-1}  + k  (1-w')^{k-1} 
    = g(w') \le \frac{1}{w'}.
\end{align*}

\end{proof}

\subsection{Proof of Theorem~\ref{thm:sec_moment_to_left_sub_gauss} }
\label{pf:thm_sec_moment_to_left_sub_gauss}

\thmSecmoToLeftSubgauss*

\begin{proof}
Set $\bbE X = \mu$ and $0\le X \le M$ a.s., then $M\ge 0$ and $0 \le \mu \le M$.
It  is equivalent to show that for any $v\ge \bbE X^2$, $\lambda\le 0$,
\begin{align*}
	\bbE[ \exp(\lambda X) ] \le \exp \left( \frac{v^2\lambda^2}{2} +\lambda\mu \right).
	%\label{eq:sub_exp_gauss_cond_equiv}
\end{align*}
Set
\begin{align*}
	f(\lambda) := \frac{v^2\lambda^2}{2} +\lambda\mu 
	- \log \bbE[ \exp(\lambda X) ],
\end{align*}
it is further equivalent to show $f(\lambda)\ge 0$. Then since $0\le X \le M$ a.s., for any $\lambda\le 0$, by Bounded Convergence Theorem,
\begin{align*}
    & \bbE[ \exp(\lambda X) ] \le 1, ~~
    \left| \frac{\rmd }{\rmd \lambda} \exp(\lambda X)  \right| = |X \exp(\lambda X) | \le M\ a.s., \\
    \Rightarrow ~&
    \frac{\rmd }{\rmd \lambda} \bbE[ \exp(\lambda X)  ]
    = \bbE\left[ \frac{\rmd }{\rmd \lambda} \exp(\lambda X) \right]
    = \bbE[ X \exp(\lambda X)  ] \le M,~~
     \left| \frac{\rmd }{\rmd \lambda} X\exp(\lambda X)  \right| = |X^2 \exp(\lambda X) | \le M^2\ a.s., \\
     \Rightarrow ~&
     \frac{\rmd }{\rmd \lambda} \bbE[ X\exp(\lambda X)  ]
    = \bbE\left[ \frac{\rmd }{\rmd \lambda} X\exp(\lambda X)  \right]
    = \bbE[ X^2 \exp(\lambda X)  ] \le M^2, \\
    & \left| \frac{\rmd }{\rmd \lambda} X^2\exp(\lambda X)  \right| = |X^3 \exp(\lambda X) | \le M^3\ a.s., \\
    \Rightarrow ~&
     \frac{\rmd }{\rmd \lambda} \bbE[ X^2\exp(\lambda X)  ]
    = \bbE \left[ \frac{\rmd }{\rmd \lambda} X^2\exp(\lambda X) \right]
    = \bbE[ X^3 \exp(\lambda X)  ] .
\end{align*}
Therefore,
\begin{align*}
	& f(0) = 0, \\
	& f'(\lambda) = v^2 \lambda + \mu - \frac{\frac{\rmd}{\rmd \lambda} \bbE[ \exp(\lambda X) ] }{\bbE[ \exp(\lambda X) ]}
	= v^2 \lambda + \mu - \frac{ \bbE[ X\exp(\lambda X) ] }{\bbE[ \exp(\lambda X) ]}, \\
	& f'(0) = \mu - \bbE X = 0,\\ 
	& f''(\lambda) = v^2 - \frac{ \bbE[ X^2\exp(\lambda X) ]\bbE[ \exp(\lambda X) ] - ( \bbE[ X\exp(\lambda X) ] )^2 }{ ( \bbE[ \exp(\lambda X) ] )^2 } 
	\ge v^2 - \frac{ \bbE[ X^2\exp(\lambda X) ] }{ \bbE[ \exp(\lambda X) ] } := g (\lambda),\\
    & g'(\lambda) = \frac{ - \bbE[ X^3\exp(\lambda X) ]\bbE[ \exp(\lambda X) ] + \bbE[ X^2\exp(\lambda X) ]\bbE[ X\exp(\lambda X) ] }{  \bbE[ \exp(\lambda X) ] )^2 }.
\end{align*}

Let $\mu$ be the probability measure of $X$ on $\bbR$. Since $0\le X\le M$ a.s., $\mu([0,M])=1$ and
\begin{align*}
    &  - \bbE[ X^3\exp(\lambda X) ]\bbE[ \exp(\lambda X) ] + \bbE[ X^2\exp(\lambda X) ]\bbE[ X\exp(\lambda X) ] \\
    &
    %  \int_{[0,M]} \int_{[0,M]} \int_\bbR \int_\bbR
    = - \int_{[0,M]} \int_{[0,M]} x^3 e^{\lambda x + \lambda y} ~ \rmd \mu(x)~ \rmd \mu(y) + \int_{[0,M]} \int_{[0,M]}  x^2 y e^{\lambda x + \lambda y} ~ \rmd \mu(x)~ \rmd \mu(y) \\
    &
    = \frac{1}{2} 
        \int_{[0,M]} \int_{[0,M]} e^{\lambda x + \lambda y} (-x^3 -y^3 + x^2y + xy^2)~ \rmd \mu(x)~ \rmd \mu(y) \\
    &
    = \frac{1}{2} 
        \int_{[0,M]} \int_{[0,M]} e^{\lambda x + \lambda y} [ -x^2(x-y) -y^2(y-x) ] ~ \rmd \mu(x)~ \rmd \mu(y) \\
    &
    = \frac{1}{2} 
        \int_{[0,M]} \int_{[0,M]} e^{\lambda x + \lambda y} (x-y) (y^2-x^2)  ~ \rmd \mu(x)~ \rmd \mu(y) \\
    & 
    = - \frac{1}{2} 
       \int_{[0,M]} \int_{[0,M]} e^{\lambda x + \lambda y} (x-y)^2 (x+y)~ \rmd \mu(x)~ \rmd \mu(y)
    \le 0.
\end{align*}
Since $\bbE[ \exp(\lambda X) ] )^2>0$, $g'(\lambda) \le 0$.
Hence $g(\lambda)$ is monotonically decreasing on $\bbR$. Further, for any $\lambda\le 0$, since $v^2\ge \bbE X^2$
\begin{align*}
	& 
	%e^\lambda \le 1 \Rightarrow 
	 f''(\lambda) \ge g(\lambda) \ge g(0) = v^2 - \bbE X^2 \ge 0
	~\Rightarrow~
	f'(\lambda) \text{ is monotonically increasing} \\
	~\Rightarrow~
	& f'(\lambda) \le f'(0) = 0 
	\hspace{8.7em}
	~\Rightarrow~
	f(\lambda) \text{ is monotonically decreasing }
	~\Rightarrow~
	f(\lambda) \ge f(0) = 0.
\end{align*}

\end{proof}

\textbf{Given $v^2\ge \bbE X^2$, it is more challenging to show $X$ is $v$-SG than to show $X$ is $v$-LSG.}
By revisiting the proof above, we see that given $X$ is $v$-LSG, to show $X$ is $v $-SG suffices to show $f(\lambda)\ge 0$ for any $\lambda\ge 0$.
Since it is hard to directly tell whether the inequality above holds for any $\lambda\ge 0$,
it is natural to look at how $f(\lambda)$ grows in $\bbR$.

Fix any $M_0>0$. For any $\lambda\in [0,M_0]$, again, applying the Bounded Convergence Theorem, we have
\begin{align*}
	& \frac{\rmd }{\rmd \lambda} \bbE[ \exp(\lambda X )  ]
    = \bbE\left[ \frac{\rmd }{\rmd \lambda} \exp(\lambda X ) \right],\
    \frac{\rmd }{\rmd \lambda} \bbE[  X  \exp(\lambda  X )  ]
    = \bbE\left[ \frac{\rmd }{\rmd \lambda}  X  \exp(\lambda  X  )  \right], \\
    \Rightarrow~ 
    & 
    f'(\lambda) = v^2 \lambda + \mu - \frac{ \bbE[ X \exp(\lambda X  ) ] }{\bbE[ \exp(\lambda X  ) ]},\\
    & f''(\lambda) = v^2 - \frac{ \bbE[ X^2\exp(\lambda X) ]\bbE[ \exp(\lambda X) ] - ( \bbE[ X\exp(\lambda X) ] )^2 }{ ( \bbE[ \exp(\lambda X) ] )^2 } .
\end{align*}
Since $f(0) = 0$ and $f'$ is differentiable on $\bbR$, it requires at least $r>0$ such that $f'(\lambda)\ge 0$ for any $\lambda \in [0,r]$. Furthermore, since $f' (0)=0$, one may consider showing that $f''(\lambda) \ge 0$ on $[0,r]$.

In the proof above, we define a function $g$ to show that $f''(\lambda) \ge g(\lambda) \ge 0$ on $(-\infty,0]$.
However, since $g (\lambda)\le 0$ on $[0,+\infty)$, this cannot help to show $f'' (\lambda) \ge 0$ on $[0,r]$.

The discussion above indicates that showing $X$ is $v$-SG is more challenging than showing $X$ is $v$-LSG.

\subsection{Proof of Lemma~\ref{lemma:obs_sec_moment} }
\label{pf:lemma_obs_sec_moment}
\lemmaObsSecmo*
\begin{proof}
 	Recall $w'$ is the minimum click probability. We abbreviate $X_{k;t}$ as $X$.
    Firstly, since $X\in [1,k]$, $\bbE X^2 \le k^2$. 
    Next, note that $\bbE X^2 $ increases when the click probabilities decrease or $k$ increases. Set $Y$ as a random variable drawn from a geometric distribution with parameter $w'$, then $\bbE X^2 \le \bbE Y^2$. Since $\bbE Y^2 = {2}/{w'^2}- {1}/{w’}$, $\bbE X^2 \le {2}/{w'^2}$.
\end{proof}

\subsection{Proof of Lemma~\ref{lemma:conc_one_sub_gauss_apply_obs} }
\label{pf:lemma_conc_one_sub_gauss_apply_obs}

\lemmaOneSubConcObs*

\begin{proof}

We abbreviate $X_{k;t}$ as $X_t$~(the number of observations of surviving items at step $t$ when pulling $k$ surviving items), and set $D_t= X_t  - \bbE X_t$, $\calF_t$ denote the decisions and observations up to step $t$. Besides, let $S_t$ be the set to pull at step $t$, then $S_t$ is determined by $\calF_{t-1}$, and $X_t$ depends on $S_t$. Since
\begin{align*}
	\bbE [ D_t | \calF_{t-1} ] 
	= \bbE[~ \bbE[  X_t  - \bbE X_t | S_t ] ~|~ \calF_{t-1} ~]
	= 0,
\end{align*}
$D_1, \ldots, D_t$ i s a martingale difference sequence adapted to $\calF = (\calF_t)_t$.
Besides, according to Theorem~\ref{thm:sec_moment_to_left_sub_gauss}, %and Lemma~\ref{lemma:obs_sec_moment}, 
for any $t$, any $\lambda\le 0$, $v_k^2 \ge \bbE X^2$ yields $\bbE[e^{\lambda D_t} | \calF_{t-1}] \le e^{\lambda^2v^2/2}$.
Then for any $\omega>0$,
\begin{align*}
    \Pr \left[~ \sum^n_{t=1} (X_t - \bbE X_t) \le -\omega ~\right]
    =
	\Pr \left[~ \sum^n_{t=1} D_t \le -\omega ~\right]
	\le 
	\exp\left(  -\frac{\omega^2}{ 2nv_k^2   } \right).
\end{align*}

Let the probability bound in the right hand side be $\delta$, then
\begin{align*}
	&\delta = \exp\left(
		-\frac{\omega^2}{ 2 n v_k^2 }
	\right) \ \
	\Rightarrow\ 
	\omega = \sqrt{  2 n v_k^2 \log \left( \frac{1}{\delta} \right) } .
\end{align*}
Note that $\bbE X_t \ge \mu_k$ for any $t$, 
\begin{align*}
	\delta & \ge
	\Pr\left( \sum_{t=1}^n (X_t - \bbE X_t) \le -\omega \right) 
	= 
	\Pr\left( \sum_{t=1}^n X_t \le \sum_{t=1}^n \bbE X_t  -\omega \right) \\
	& \ge 
	\Pr\left( \sum_{t=1}^n X_t \le n \mu_k -\omega \right)
	\ge 
	\Pr\left( \sum_{t=1}^n X_t \le n \mu_k - \sqrt{ 2n v_k^2 \log \left( \frac{1}{\delta} \right) }~ \right).  
\end{align*}

Next, for any $T>0$, consider
\begin{align*}
	n \mu_k - \sqrt{ 2n v_k^2 \log \left( \frac{1}{\delta} \right) } \le T.
\end{align*}
Set 
\begin{align*}
	a_0 = \frac{1}{\mu_k} \sqrt{ 2v_k^2 \log \left( \frac{1}{\delta} \right) },\
	b_0 = \frac{ T  }{\mu_k},\
	x = \sqrt{n},
\end{align*}
then $x\ge 0$ and $x^2 - a_0x - b_0 \le 0$. Note that $(p+q)^2\le 2(p^2+q^2)$,
\begin{align*}
	& x \le \frac{a_0 + \sqrt{a_0^2+4b_0} }{2} \\
	\Rightarrow \ &
	n \le \left( \frac{a_0 + \sqrt{a_0^2+4b_0} }{2} \right)^2 \le a_0^2 + 2b_0
	= \frac{2T}{\mu_k} 
	+  \frac{2v_k^2 }{\mu_k^2}\log \left( \frac{1}{\delta} \right) .
\end{align*}

\end{proof}

\subsection{Proof of Lemma~\ref{le:conc_event} }
\label{pf:le_conc_event}
\leConcEvent*

\begin{remark}[Sub-Gaussian property]\label{remark:sub_gauss}
    Define $\eta_{t}(i) = W_t(i) - w(i) $, then 
    $\eta_{t}(i)$ is $1/2$-sub-Guassian. 
\end{remark}

\begin{proof}[Proof of Remark~\ref{remark:sub_gauss}]
	Any non-negative random variable bounded in $[a,b] $ a.s. is sub-Gaussian with parameter $(b-a)/2$. Meanwhile, $W_t(i)\in [0,1]$ yields that $\eta_{t}(i) \in [ w(i)-1, w(i) ]$. $[w(i) - (w(i)-1 )]/2=1/2$.
\end{proof}

\begin{proof}
	For all $i \in [L]$, $\calE(i,\delta) = \{ \forall t\ge 1, | \hat{w}_t(i) - w(i) | \le C_t(i,\delta) \}$. Recall that
	\begin{align*}
		& C_t(i, \delta) = \tilde{C}\left( T_t(i), \rho  \right), \quad
		\tilde{C}(\tau, \rho) = 4\sqrt{ \frac{ \log \left( \log _{2}(2 \tau ) / \rho \right)}{\tau+1} },\quad
		\rho(\delta) = \sqrt{\delta/(12L)},
	\end{align*}
	then according to Theorem~\ref{thm:conc_log},
	\begin{align*}
		& \bbP \left( {\calE}(i,\delta) \right) \ge 1- 6\rho(\delta)^2  = 1- \frac{\delta}{2L} \
		\Rightarrow \ 
		\bbP \left( \bar{\calE}(i,\delta) \right) \le \frac{\delta}{2L}, \\
		\Rightarrow \ & 
		\bbP \left( \bigcap_{i=1}^L \calE(i,\delta) \right) = 1 - \bbP \left( {\bigcup_{i=1}^L \bar{\calE}(i,\delta) }  \right)
		\ge 1 - \sum_{i=1}^L \bbP \left( \bar{\calE}(i,\delta) \right) \ge 1 - L\cdot \frac{\delta}{2L} = 1 - \delta/2.
	\end{align*}
\end{proof}

\subsection{Proof of Lemma~\ref{le:suff_obs_to_identif} }
\label{pf:le_suff_obs_to_identif}
\lemmaSuffObsIdentif*

\textbf{Preliminary}. Since we use the UCB of the empirical top-$(k_t+1)$ item to accept $\epsilon$-optimal items, hopefully it should be close to the true click probability of item $(k_t+1)$; likewise, the LCB of the empirical top-$(k_t)$ item should be close to the true click probability of item $(k_t)$. This is stated in Lemma~\ref{le:true_trap_by_emp}.

\begin{lemma}[{\citet[Lemma 3]{jun2016top}}]
\label{le:true_trap_by_emp}
	Denote by $\hati$ the index of the item with empirical mean is $i$-th largest: i.e., $\hat{w}(\hat{1} ) \geq \ldots \geq \hat{w}(\hat{L}) .$ Assume that the empirical means of the arms are controlled by $\epsilon:$ i.e., $\forall i,\left| \hat{w}(i) - w(i) \right| < \epsilon$. Then,
$$
\forall i, w(i)-\epsilon \leq \hat{w}(\hat{i}) \leq w(i)+\epsilon.
$$
\end{lemma}

After that, Lemma~\ref{le:suff_obs_to_identif} shows that the agent will correctly classify the items after a sufficient number of observations, and also show what is the sufficient number of observations for each item.

\begin{proof}
	Recall
\begin{align*}
	k_t =  K - |A_t|, \ 
    \rho(\delta') =  {\delta'}/{(12L) }, \ 
    \barT_{i,\delta'} =  
    1 + \left\lfloor \frac{216}{  \bar{\Delta}_i^2 }  \log \left( \frac{2}{ \rho(\delta') } \log_2 \left( \frac{ 648 }{ \rho(\delta')  \bar{\Delta}_i^2 } \right) \right)  \right\rfloor
    .
\end{align*} 	
	And We use $\rho$ and $\rho'$ as abbreviations for $\rho(\delta)$ and $\rho(\delta')$ respectively. 
	
    It suffices to show for the case where $A_t$ and $R_t$ are empty since otherwise the problem is equivalent to removing rejected or accepted arms from consideration and starting a new problem with $L_{\text{new}}= L - |A_t| - |R_t| $ and $K_{\text{new}} = K - |A_t|$ while maintaining the observations so far. Note that when $A_t$ is empty, $k_t = K$.
    
    First of all, $T_t(i) \ge T'_t$ implies that 
    \begin{align}
        C_t(i, \delta) = \tilde{C}\left( T_t(i), \rho  \right) \le \tilde{C}\left( T'_t, \rho \right),\
        C_t(i, \delta') = \tilde{C}\left( T_t(i), \rho'  \right) \le \tilde{C}\left( T'_t, \rho' \right).
        \label{eq:suff_obs_identif_conf_radius_gen_bd}
    \end{align}
Then since  $\bigcap_{i=1}^L \calE(i,\delta')$ holds, $| \hatw_t(i) - w(i) | \le \tilde{C}\left( T_t(i), \rho' \right) \le \tilde{C}\left( T'_t, \rho' \right)$ for all $i \in D_t$. Combining this with Lemma~\ref{le:true_trap_by_emp}, we have
    \begin{align}
    	w(i) + \tilde{C}\left( T'_t, \rho' \right) \le \hatw_t(i) \le w(i) + \tilde{C}\left( T'_t, \rho' \right)\  \forall i \in D_t .
    	\label{eq:sub_obs_identif_apply_trap_by_emp}
    \end{align}

    \textbf{We first prove that for any $i\le K'$,} 
    \begin{align*}
    	T'(t) \ge \barT_{i,\delta'} 
          ~\Rightarrow ~  L_{t}(i, \delta)> U_{t}(j^*, \delta) - \epsilon \text{ where }  j^* =   \argmax _{j \in D_{t}}^{ (k_{t}+1 )} \hatw_{t} 
          ~\Rightarrow ~  i \in A_t. 
    \end{align*}   
For clarity, we write $j^*=\widehat{K+1}$, which is the item with $(K+1)$-th largest empirical mean at the $t$-th step. 
    We assume the contrary: $L_{t}(i, \delta) \le U_{t}( \widehat{K+1}, \delta ) - \epsilon $. 
    We can apply~\eqref{eq:suff_obs_identif_conf_radius_gen_bd} and~\eqref{eq:sub_obs_identif_apply_trap_by_emp} to obtain
    \begin{align*}
        & L_t(i, \delta) \ge \hat{w}_t(i) -  \tilde{C}\left( T'_t, \rho \right)
            \ge w(i) - \tilde{C}\left( T'_t, \rho \right) - \tilde{C}\left( T'_t, \rho' \right), \\
         & 
         U_t( \widehat{K+1} ) - \epsilon 
         \le \hatw_t(\widehat{K+1} ) + \tilde{C}\left( T'_t, \rho  \right) - \epsilon
         \le w ( K+1 ) + \tilde{C}\left( T'_t, \rho  \right) + \tilde{C}\left( T'_t, \rho'  \right) - \epsilon.
    \end{align*}            
Next,
    \begin{align*}
        & w(i) - \tilde{C}\left( T'_t, \rho \right) - \tilde{C}\left( T'_t, \rho' \right) 
          \le 
          w ( K+1 ) + \tilde{C}\left( T'_t, \rho  \right) + \tilde{C}\left( T'_t, \rho'  \right) - \epsilon,  \\
        \Rightarrow~ & 
          0 \overset{(a)}{<} w(i) - w(K+1) + \epsilon \le 2 \tilde{C}\left( T'_t, \rho \right) + 2 \tilde{C}\left( T'_t, \rho' \right) \le 4 \tilde{C}\left( T'_t, \rho' \right)
            =    16 \sqrt{ \frac{ \log \left( \log _{2}(2 T'_t ) / \rho' \right)}{ T'_t  } }, \\
        \Rightarrow~ &
            T'_t \le \frac{216}{ ([  w(i) - w(K+1) + \epsilon ]^2 }  \log \left( \log _{2}(2 T'_t ) / \rho' \right).
    \end{align*}
    Part (a) of the second line above follows from: (i) if $i \le K$, $w(i) - w(K+1) + \epsilon = \Delta_i + \epsilon >0$; (ii) else, $K < i \le K'$, since $w(i) \ge w(K) -\epsilon$, we have $w(i) - w(K+1) + \epsilon = w(i) - w(K) + w(K) - w(K+1) + \epsilon = \Delta_K - \Delta_i + \epsilon \ge \Delta_K >0$.
    Then invert to the third line using
    \begin{align*}
        \tau \le c \log \left(  \frac{ \log_2 2\tau }{ \rho' } \right)
        \Rightarrow 
        \tau \le c \log \left( \frac{2}{ \rho'} \log_2 \left( \frac{3c}{\rho'} \right) \right)
    \end{align*}
    with $c = 216 [  w(i) - w(K+1) + \epsilon ]^{-2} $ to have
    \begin{align*}
        T'_t 
        & \le \frac{216}{ [  w(i) - w(K+1) + \epsilon ]^2 }  \log \left( \frac{2}{ \rho'} \log_2 \left( \frac{ 648 }{ \rho'  [  w(i) - w(K+1) + \epsilon ]^2 } \right) \right) \\
        &< 1 + \left\lfloor \frac{216}{ [  w(i) - w(K+1) + \epsilon ]^2 }  \log \left( \frac{2}{ \rho'} \log_2 \left( \frac{ 648 }{ \rho' [  w(i) - w(K+1) + \epsilon ]^2 } \right) \right)  \right\rfloor  
            = \barT_{i,\delta'}.
    \end{align*}
    
    Therefore, $\barT'_t  \ge \bar{T}_{i,\delta'}$ implies that $L_{t}(i, \delta)> U_{t}(j^*, \delta) - \epsilon$ where $ j^* =   \argmax _{j \in D_{t}}^{ (k_{t}+1 )} \hatw_{t}  $. Then $i \in A_t$ is accepted.

	\textbf{Subsequently, we prove that for any $i>K'$,}
	\begin{align*}
		T'(t) \ge \barT_{i,\delta'} 
          ~\Rightarrow ~ U_{t}(i, \delta)<   L_{t}(j', \delta) - \epsilon \text{ where }   j' = \argmax _{j \in D_{t}}^{ (k_{t} )} \hatw_{t} 
          ~\Rightarrow ~  i \in R_t.
	\end{align*} 
Again for brevity, we write $\hatK=j'$, the item with $K$-th largest empirical mean at the $t$-th step.   
    We assume the contrary: $U_{t}(i, \delta )  \ge L_{t}( \hatK, \delta ) -\epsilon$. 
    Again applying~\eqref{eq:suff_obs_identif_conf_radius_gen_bd} and~\eqref{eq:sub_obs_identif_apply_trap_by_emp}, we have
    \begin{align*}
        & U_t(i,\delta) \le \hat{w}_t(i) + \tilde{C}\left( T'_t, \rho \right)
            \le w(i) + \tilde{C}\left( T'_t, \rho \right) + \tilde{C}\left( T'_t, \rho' \right), \\
         & 
         L_t( \hatK, \delta ) - \epsilon \ge \hatw_t(\hatK) - \tilde{C}\left( T'_t, \rho  \right) - \epsilon
         \ge 
         w( K) - \tilde{C}\left( T'_t, \rho  \right) - \tilde{C}\left( T'_t, \rho'  \right) - \epsilon.
    \end{align*}            
Next,
    \begin{align*}
        & w(i) + \tilde{C}\left( T'_t, \rho \right) + \tilde{C}\left( T'_t, \rho' \right) 
          \ge w( K) - \tilde{C}\left( T'_t, \rho  \right) - \tilde{C}\left( T'_t, \rho'  \right) - \epsilon,  \\
        \Rightarrow~ & 
            0< w(K) - w(i) - \epsilon \le 2 \tilde{C}\left( T'_t, \rho \right) + 2 \tilde{C}\left( T'_t, \rho' \right) \le 4 \tilde{C}\left( T'_t, \rho' \right)
            =    16 \sqrt{ \frac{ \log \left( \log _{2}(2 T'_t ) / \rho' \right)}{ T'_t  } }.
    \end{align*}
    Similar to the first case, with
    \begin{align*}
    	\barT_{i,\delta'}
    	=
    	1 + \left\lfloor \frac{216}{ ( w(K)-w(i) - \epsilon )^2 }  \log \left( \frac{2}{ \rho'} \log_2 \left( \frac{ 648 }{ \rho'  ( w(K)-w(i) - \epsilon )^2 } \right) \right)  \right\rfloor  
    \end{align*}
    we obtain that $\barT'_t  \ge \bar{T}_{i,\delta'} $ implies $U_{t}(i, \delta)<   L_{t}(j', \delta) - \epsilon$ where $j' = \argmax _{j \in D_{t}}^{ (k_{t} )} \hatw_{t} $. Then $i \in R_t$ is rejected.
    
\end{proof}

\subsection{Proof of Lemma~\ref{le:stop_eliminate_no} }
\label{pf:le_stop_eliminate_no}

\leStopElimintateNo*

\begin{proof}
	Assume $\bigcap_{i=1}^L \calE(i,\delta)$ holds. 
	
\underline{\bf Case (i)}: $K'=K$. 
	In the worst case, the algorithm does not terminate before identifying the $(L-1)$-th one. In this case, after identifying the $(L-1)$-th one with sufficient observations, either the accept set or the reject set is full, i.e., $|A_t|=K$ or $|R_t|=L-K$, the the agent can just place the remaining item in the unfilled set. 
	
	Hence, the algorithm terminates after sufficiently observing and identifying at most $L-1 = L- \max\{ K'+K, 1\}$ items.
	
\underline{\bf Case (ii)}: $K'>K$. 
	The algorithm classify all items correctly according to Lemma~\ref{le:suff_obs_to_identif}. since the number of $\epsilon$-optimal items is 
 $K'=\max \{ i: w(i)\ge w(K)-\epsilon \}\ge K$
, the number of suboptimal items is $L-K'\le L-K$. Hence, $|R_t|\le L-K'$. Besides, $|A_t|\le K$ according to the design of the algorithm. Therefore,
\begin{align*}
	|A_t| + |R_t| \le L-K'+K.
\end{align*}
In other words, the algorithm terminates after sufficiently observing and identifying at most $L-K'+K= L- \max\{ K'+K, 1\}$ items.
\end{proof}

\subsection{Proof of Lemma~\ref{lemma:lb_KLdecomp}}
\label{pf:lemma_lb_KL_decomp}
\lemmaLbKLdecomp*

To manifest the difference between instance $\ell$ and other instances, with $w^{(0)}(i) = w(i)$ for all $i\in[L]$ we write
\begin{itemize}
	\setlength{\itemindent}{0pt}
	\item $\{w^{(0)}(1), w^{(0)}(2), \ldots, w^{(0)}(L)\}$ under instance $0$;
	\item $\{w^{(0)}(1), w^{(0)}(2), \ldots, w^{(0)}(\ell-1),w^{(\ell)}(\ell), w^{(0)}(\ell+1), \ldots, w^{(0)}(L)\}$ under instance $\ell$.
\end{itemize}

We combine Lemma~\ref{lemma:KLdecompOR} and a result from \citet{kaufmann2016complexity} to relate the time complexity and KL divergence together.

\begin{lemma}[{\citep[Lemma 19]{kaufmann2016complexity}}]
\label{lemma:lb_kl_event_kau}
	Let $\calT$ be any almost surely finite stopping time with respect to $\calF_t$. For every event $\calE \in \calF_\calT$, instance $1\le \ell \le L$,
	\begin{align*}
		\mathrm{KL}\left(\{ S^{\pi, 0}_t, \bmO^{\pi, 0}_t \}^\calT_{t=1}, \{ S^{\pi, \ell}_t, \bmO^{\pi, \ell}_t \}^\calT_{t=1}\right)
		\ge 
		\mathrm{KL}( \bbP_0(\calE), \bbP_\ell(\calE) ).
	\end{align*}

\end{lemma} 

{\bf Notations.} Before presenting the proof, we remind the reader of the definition of the KL divergence~\citep{cover2012elements}. For two discrete random variables $X$ and $Y$ with common support~${\cal A}$, 
$$\mathrm{KL}(X,Y) = \sum_{ x\in{\cal A}} P_X(x) \log \frac{ P_X (x) }{ P_Y(x) }$$
 denotes the KL divergence between probability mass functions of $X$ and $Y$. Next, we also use $\mathrm{KL}(P_X \| P_Y)$ to also signify this KL divergence. Lastly, when $a$ and $b$ are two real numbers between $0$ and $1$, $\mathrm{KL}(a,b)= \mathrm{KL} \left( \text{Bern}(a) \|\text{Bern}(b) \right)$, i.e., $\mathrm{KL}(a,b)$ denotes the KL divergence between $\text{Bern}(a)$ and $\text{Bern}(b)$.

\begin{proof}

For any certain $s_t$, we can observe that the KL divergence $\mathrm{KL}\left( P_{ \bmO_t^{\pi,0 }  \mid S_t^{\pi,0}}(\cdot \mid s_t)  \,\Big\|\,  P_{ \bmO_t^{\pi, \ell }  \mid S_t^{\pi,\ell}}(\cdot \mid s_t)  \right)$ should grow
with the KL divergence of observed items. Further, for each observed item $i$, there is a KL divergence of $\mathrm{KL}\left( w^{(0)}(i), w^{(\ell)}(i) \right)$. 
Whenever $S^{\pi, 0}_t = s_t$, we have
	\begin{align*}
		& \mathrm{KL}\left( P_{ \bmO_t^{\pi,0 }  \mid S_t^{\pi,0}}(\cdot \mid s_t)  \,\Big\|\,  P_{ \bmO_t^{\pi, \ell }  \mid S_t^{\pi,\ell}}(\cdot \mid s_t)  \right)
%			 = \bbE_0 \left[ \Upsilon_t^\ell(S_t^{\pi,0}) \right] \cdot \mathrm{KL}\left(\frac{1-\epsilon}{K}, \frac{1 + \epsilon}{K}\right). \\
	 = \sum_{i\in s_t} \bbE_0 \left[ \mathsf{1}\{ i \text{ is observed at time } t \} \right] \cdot \mathrm{KL}\left( w^{(0)}(i), w^{(\ell)}(i) \right).
	\end{align*}
Then according to Lemma~\ref{lemma:KLdecompOR},
	\begin{align*}
	& \mathrm{KL}\left(\{ S^{\pi, 0}_t, \bmO^{\pi, 0}_t \}^\calT_{t=1}, \{ S^{\pi, \ell}_t, \bmO^{\pi, \ell}_t \}^T_{t=1}\right) \\
	& = \sum_{t=1}^\calT \sum_{s_t\in [L]^{(K)}}\Pr[S^{\pi, 0}_t = s_t]\cdot \mathrm{KL}\left( P_{ \bmO_t^{\pi,0 }  \mid S_t^{\pi,0}}(\cdot \mid s_t)  \,\Big\|\,  P_{ \bmO_t^{\pi, \ell }  \mid S_t^{\pi,\ell}}(\cdot \mid s_t)  \right) \\ 
	& = \sum_{t=1}^\calT \sum_{s_t\in [L]^{(K)}}\Pr[S^{\pi, 0}_t = s_t]\cdot 
		\sum_{i\in s_t} \bbE_0 \left[ \mathsf{1}\{ i \text{ is observed at time } t \} \right] \cdot \mathrm{KL}\left( w^{(0)}(i), w^{(\ell)}(i) \right) \\ 
	& = \sum_{i=1}^L \sum_{t=1}^\calT \sum_{s_t\in [L]^{(K)}}
		\bbE_0 \left[ \mathsf{1}\{ S^{\pi, 0}_t = s_t, \ i \in s_t, \ i \text{ is observed at time } t \} \right] \cdot \mathrm{KL}\left( w^{(0)}(i), w^{(\ell)}(i) \right) \\ 
	& = \sum_{i=1}^L \bbE [T_\calT(i)] \cdot \mathrm{KL}\left( w^{(0)}(i), w^{(\ell)}(i) \right) \\
	& = \bbE [T_\calT(\ell)] \cdot \mathrm{KL}\left( w^{(0)}(\ell), w^{(\ell)}(\ell) \right).
\end{align*}

\end{proof}

\subsection{ Proof of Theorem~\ref{thm:fix_conf_up_bd} }
\label{appdix:conf_up_bd_last_step}

\textbf{Preliminary.} 
Recall that $\bar{\Delta}_{\sigma(1)} \ge \bar{\Delta}_{\sigma(2)} \ge \ldots \ge \bar{\Delta}_{\sigma(L)}$, and $T_t(i)$ counts the number of observations of item $i$ up to the $t$-th step. The worst case is that the algorithm eliminates $\sigma(1), \sigma(2), \ldots$ in order, and the algorithm eliminates at most $1$ item at one time step. Besides, the design of Algorithm~\ref{alg:bai_cas_conf} implies that  
\begin{align}
	T_t(j)-1 \le T_t(i) \le T_t(j)+1, ~\forall i \neq j\in D_t.
	\label{eq:design_of_alg_T_t}
\end{align}

In the following discussion,we assume $\bigcap_{i=1}^L \calE(i,\delta)$ holds and $K'<2K-1$~(discussion on $K'\ge 2K-1$ is in Appendix~\ref{appdix:influen_epsilon}). 
Note that Lemma~\ref{le:conc_event} implies that
$\bbP \left( \bigcap_{i=1}^L \calE(i,\delta) \right) \ge 1- \delta/2$.
Besides, we write $\mu(k,w )$ as $\mu_k$, $v(k,w )$ as $v_k$, $\bar{T}_{i,\delta}$ as $\bar{T}_{i}$, $\rho(\delta)$ as $\rho$ for simplicity.

\textbf{Bound the number of observations per phrase.}
Observe that there are less than $K$ surviving items remaining in the survival set $D_t$ at some steps before the algorithm terminates, we separate the steps into several phrases:

(i) {During the first phrase}, the agent eliminates $L-K+1$ items within $t_1$ steps. According to Lemma~\ref{le:suff_obs_to_identif} and Line~\eqref{eq:design_of_alg_T_t}, 
\begin{align*}
	& T_{t_1}(\sigma(j)) \le \bar{T}_{\sigma(j)}, \hspace{6em}
	\forall 1 \le j \le L-K+1; \\
	& T_{t_1}(\sigma(i)) \le \bar{T}_{\sigma(L-K+1)} + 1,\hspace{2em}
	\forall L-K+1< i \le L.
\end{align*}
Then the total number of observations of surviving items in $D_t$ within this phrase can be bounded as follows:
\begin{align*}
	\sum_{i=1}^{L-K+1} \bar{T}_{\sigma(i)} + \sum_{i=L-K+2}^L T_{t_1}(\sigma(i))
	\le \sum_{i=1}^{L-K} \bar{T}_{\sigma(i)} + K\bar{T}_{\sigma(L-K+1)} +(K-1) := \tilde{T}_1.
\end{align*}

(ii) {During the $k$-th phrase} for any $2\le k \le K-\max\{K'-K,1\} $, the agent eliminates the $L-K+k$-th item within $t_k$ steps.
Again apply Lemma~\ref{le:suff_obs_to_identif} and Line~\eqref{eq:design_of_alg_T_t}: 
\begin{align*}
	& T_{\sum_{j=1}^{k } t_j }(\sigma(L-K+k))
	  \le \bar{T}_{\sigma(L-K+k )} ; \\
	& T_{\sum_{j=1}^{k } t_j }(\sigma(i))
	  \le \bar{T}_{\sigma(L-K+k )} +1, \hspace{3.1em}
	  \forall L-K+k+1 \le i\le L;\\
	& T_{\sum_{j=1}^{k-1} t_j }(\sigma(i))
	  \ge \bar{T}_{\sigma(L-K+k-1)} -1,\hspace{2em}
	  \forall L-K+k \le i\le L.
\end{align*}
Then the total number of observations of surviving items in $D_t$ within this phrase can also be bounded:
\begin{align*}
	& 
	\bar{T}_{\sigma(L-K+k)} + \sum_{i=L-K+k+1}^L T_{\sum_{i=j}^k t_j}(\sigma(i)) - \sum_{i=L-K+k}^L T_{\sum_{j=1}^{k-1} t_j }(\sigma(i))
	\\
	& 
	\le (K-k+1)[\bar{T}_{\sigma(L-K+k)} - \bar{T}_{\sigma(L-K+k-1)}] + 2(K-k)+1 := \tilde{T}_k.
\end{align*}

\textbf{Bound the number of time steps per phrase.}
Recall that the $k$-th~($1\le k \le K-\max\{K'-K,1\} $) phrase consist of $t_k$ time steps. Let $Z_k$ be the total number of observations within the $t_k$ steps. 
Lemma~\ref{lemma:conc_one_sub_gauss_apply_obs}
indicates that 
\begin{align*}
	\bbP( Z_k \ge t_k\mu_{K+1-k}\omega_{K+1-k} ) 
	\ge 1 - \delta_k
	~~~~~\text{ with }~~~
	\omega_{K+1-k} = -  \sqrt{-2t_k v_{K+1-k}^2 \log \delta_k}.
\end{align*}
Then according to Lemma~\ref{lemma:conc_one_sub_gauss_apply_obs}, for any $k~(1\le k \le K-\max\{K'-K,1\} )$, with probability at least $1-\delta_k$,
\begin{align*}
	t_k \le  \frac{2\tilde{T}_k }{\mu_{K-k+1} } 
	- \frac{2v_{K+1-k}^2 }{\mu_{K-k+1}^2} \cdot \log \delta_k .
\end{align*}

\textbf{Bound the time complexity.}
Altogether, we would have $\sum_{k=1}^{ K-\max\{K'-K,1\} } t_k$ as the time complexity. 
Besides, we bound the total error incurred by partial observation by $\delta/2$.
In other words,
\begin{align}
	\calT
	\le 
	\sum_{k=1}^{ K-\max\{K'-K,1\} }  
	\left(
	         - \frac{2 v_{K-k+1}^2 }{\mu_{K-k+1}^2} \cdot \log \delta_k 
	    + 2\sum_{k=1}^{2K-K' } \frac{\tilde{T}_k}{\mu_{K-k+1} }
	\right)
~\text{ where } \sum_{k=1}^{ K-\max\{K'-K,1\} } \delta_k \le \delta/2. \nonumber
%\label{eq:opt_prob_conf_up_bd}
\end{align}
Depending on the value of $K'-K$, there are two cases:
\begin{enumerate}
	\item[] {\bf Case 1:} $K'-K\ge 1 $ , i.e., $K-\max\{ K'-K,1 \} =2K-K'$;
	\item[] {\bf Case 2:} $K'=K$ , i.e., $ K-\max\{ K'-K,1 \} =K-1$.
\end{enumerate}
For brevity, we only go through the remaining analysis for the first case, the analysis for the second one is similar.

Since the second term in the bound on $\calT$ merely depends on the problem, we turn to analyze the first term. Since the first term holds for any values of $\delta_k$'s
such that $\sum_{k=1}^{2K-K'} \delta_k \le \delta/2$, we  minimize the first term with the method of Lagrange multiplier. Set $c_k = \frac{2 v_{K-k+1}^2 }{\mu_{K-k+1}^2}$, the problem turns to
\begin{align*}
	(\blacktriangle) = \max_{\delta_k: 1\le k\le 2K-K' } \sum_{k=1}^{2K-K' } c_k \log \delta_k  
	\quad \text{ s.t. }\quad  
	\sum_{k=1}^{2K-K'} \delta_k \le \delta/2. 
\end{align*}
Let
\begin{align*}
	L\left( \{\delta_k\}_{k=1}^{2K-K' }, \{c_k\}_{k=1}^{2K-K'-1}, \lambda \right)
	 = \sum_{k=1}^{2K-K' } c_k \log \delta_k + \lambda \left(  \sum_{k=1}^{2K-K' } \delta_k - \delta/2 \right),
\end{align*}
then for all $1\le k \le 2K-K' $,
\begin{align*}
	& \frac{ \partial L }{\partial \delta_k} = \frac{c_k}{\delta_k} + \lambda=0 \
	\Rightarrow \ 
	\delta_k^* = \frac{ c_k \delta }{ 2\sum_{j=1}^{2K-K' } c_j }. 
\end{align*}
$(\blacktriangle)$ attains its maximum when $\delta_k = \delta_k^*$ for all $1\le k \le 2K-K' $. Hence
\begin{align*}
	\sum_{k=1}^{2K-K'} t_k 
	& \le 
	  - \sum_{k=1}^{2K-K' }  
         c_k \log \delta_k^*
	    + 2\sum_{k=1}^{2K-K' } \frac{\tilde{T}_k}{\mu_{K-k+1} } \\
	& = 
	  \sum_{k=1}^{2K-K' }  
         c_k \log \left( \frac{ 2\sum_{j=1}^{2K-K' } c_j }{ c_k \delta } \right)
	    + 2\sum_{k=1}^{2K-K' } \frac{\tilde{T}_k}{\mu_{K-k+1} } \\
	& = 
	\underbrace{
	  \sum_{k=1}^{2K-K' }  
         c_k \log \left( \frac{2 }{\delta} \sum_{j=1}^{2K-K' } c_j \right)
	 }_{(\spadesuit)}
	 +
	 \underbrace{
	    \sum_{k=1}^{2K-K' }  
         c_k \log \left( \frac{ 1 }{ c_k } \right)
	 }_{(\heartsuit)}
	+
	\underbrace{ 
	2\sum_{k=1}^{2K-K' } \frac{\tilde{T}_k}{\mu_{K-k+1} }
	}_{(\clubsuit)} .
\end{align*}

Now we bound $(\spadesuit)$, $(\heartsuit)$, $(\clubsuit)$ individually.

\textbf{Bounding $(\spadesuit)$ }: note that $\mu_{K+1-k} \ge 2$ for all $1\le k \le 2K-K' $, $K'\ge K$ and 
$c_k = \frac{2 v_{K-k+1}^2 }{\mu_{K-k+1}^2} $,
\begin{align*}
	(\spadesuit) 
	& = \sum_{k=1}^{2K-K' }  
         c_k \log \left( \frac{2 }{\delta} \sum_{j=1}^{2K-K' } c_j \right) 
    = \sum_{k=1}^{2K-K' }  
         \frac{2 v_{K-k+1}^2 }{\mu_{K-k+1}^2} 
         \log \left( \frac{2 }{\delta} \sum_{j=1}^{2K-K' } 
          \frac{2 v_{K-j+1}^2 }{\mu_{K-j+1}^2} 
          \right) 
          .
\end{align*}

\textbf{Bounding $(\heartsuit)$ }: Let $g(x) = \frac{\log x }{x}$ for $x>0$, then $g'(x) = \frac{ 1 - \log x }{x^2}$. Since $g'(x) > 0 $ when $x \in (0,e)$, $g'(e)=0$, $g'(x)<0$ when $x\in (e,+\infty)$, $g(x)$ is increasing on $(0,e)$, is decreasing on $(e,+\infty)$ and attains its global maximum $g(e) = \frac{1}{e}$ at $x=e$. Hence,
\begin{align*}
	(\heartsuit) = \sum_{k=1}^{2K-K' } g \left( \frac{1}{c_k} \right)
		\le \frac{2K-K' }{e} \le K .
\end{align*}

\textbf{Bounding $(\clubsuit)$ }: 
We first rewrite this term according to the definition of $\tilde{T}_k$'s: 
\begin{align*}
	& \tilde{T}_1 = \sum_{i=1}^{L-K} \bar{T}_{\sigma(i)} + K\bar{T}_{\sigma(L-K+1)} +(K-1), \\
	& \tilde{T}_k = (K-k+1)[\bar{T}_{\sigma(L-K+k)} - \bar{T}_{\sigma(L-K+k-1)}] + 2(K-k)+1 \ \quad  \forall 2\le k \le 2K-K'-1, \\
	\Rightarrow~ &
	 (\clubsuit)
    \le
         \frac{2}{\mu_K} \left[ \sum_{i=1}^{L-K} \bar{T}_{\sigma(i)} + K\bar{T}_{\sigma(L-K+1)} + K \right]
        +
          \sum_{k=2}^{2K-K' }\frac{ 2(K-k+1) }{\mu_{K-k+1}} \left[ \bar{T}_{\sigma(L-K+k)} - \bar{T}_{\sigma(L-K+k-1)} +3 \right] \\
    \Rightarrow~ &
	 (\clubsuit)/4
    \le
         \frac{1}{\mu_K} \left[ \sum_{i=1}^{L-K} \bar{T}_{\sigma(i)} + K\bar{T}_{\sigma(L-K+1)} \right]
        +
          \sum_{k=2}^{2K-K' }\frac{ K-k+1 }{\mu_{K-k+1}} \left[ \bar{T}_{\sigma(L-K+k)} - \bar{T}_{\sigma(L-K+k-1)}  \right]  .
\end{align*}

Next, since $ \mu_k \ge \min\{ k/2, 1/(2w^*) \}$ as shown in Lemma~\ref{lemma:bound_obs}, when $K-k+1 \le 1/w^*$,
\begin{align*}
    k \ge K+1- \lfloor 1/w^* \rfloor, \ 
    \mu_{K-k+1} \ge \frac{K-k+1}{2},\ 
    \frac{K-k+1}{ \mu_{K-k+1}} \le 2.
\end{align*}  
Hence with $K_0 = \max \{ \min \{ 2K-K', K - \lfloor 1/w^* \rfloor \} ,1 \}$,
\begin{align*}
    & \frac{K}{\mu_K} \bar{T}_{\sigma(L-K+1)} + 
    \sum_{k=2}^{K_1 }\frac{ K-k+1 }{\mu_{K-k+1}} \left[ \bar{T}_{\sigma(L-K+k)} - \bar{T}_{\sigma(L-K+k-1)} \right]  
    \\
    & 
    = 
         \frac{ K\bar{T}_{\sigma(L-K+1)} }{  \mu_K }
        +
          \sum_{k=2}^{K_0}\frac{  (K-k+1)\bar{T}_{\sigma(L-K+k)}  }{ \mu_{K-k+1} } 
          -
          \sum_{k=1}^{ K_0-1 }\frac{ (K-k) \bar{T}_{\sigma(L-K+k)}  }{ \mu_{K-k}}     \nonumber \\
    & = 
         \frac{ (K -K_0+1)\bar{T}_{\sigma(L-K+K_0 ) } }{ \mu_{K-K_0+1} }
        +
          \sum_{k=1}^{ K_0-1 }\bar{T}_{\sigma(L-K+k)}
          \left( \frac{  K-k+1  }{ \mu_{K-k+1} } - \frac{  K-k  }{ \mu_{K-k}}  \right),
\end{align*}
and
\begin{align*}
     \sum_{k=K_0+1}^{ 2K-K'  }\frac{ K-k+1 }{\mu_{K-k+1}} \left[ \bar{T}_{\sigma(L-K+k)} - \bar{T}_{\sigma(L-K+k-1)} \right] 
     & \le 
     2 \sum_{k=K_0+1}^{ 2K-K'  }  \bar{T}_{\sigma(L-K+k)} - \bar{T}_{\sigma(L-K+k-1)} \\
     &= 2 \bar{T}_{\sigma(L+K-K')} - 2 \bar{T}_{ \sigma(L-K+K_0) }.
\end{align*}
Further,
\begin{align*}
   & (\clubsuit)/4
    \le 
      \frac{1}{\mu_K} \sum_{i=1}^{L-K} \bar{T}_{\sigma(i)} 
        +
           \sum_{k=1}^{ K_0-1 }\bar{T}_{\sigma(L-K+k)}
          \left( \frac{  K-k+1  }{ \mu_{K-k+1} } - \frac{  K-k  }{ \mu_{K-k}}  \right)
         + 
         \left( \frac{  K -K_0+1 }{ \mu_{K-K_0+1} } -2 \right) \bar{T}_{\sigma(L-K+K_0 ) }
         +
         2 \bar{T}_{\sigma(L+K-K')}.
\end{align*}

\textbf{Summation of $(\spadesuit)$, $(\heartsuit)$, $(\clubsuit)$.}
Recall $\rho = {\delta}/{(12L) }$ and
\begin{align*}
    &    \barT_i = 1 + \left\lfloor \frac{216}{  \bar{\Delta}_i^2 }  \log \left( \frac{2}{ \rho} \log_2 \left( \frac{ 648 }{ \rho \bar{\Delta}_i^2 } \right) \right)  \right\rfloor.
\end{align*}
The time complexity is upper bounded by
\begin{align*}
	& c_1
		\sum_{k=1}^{2K-K' }  
         \frac{ v_{K-k+1}^2 }{\mu_{K-k+1}^2} 
         \log \left( \frac{ 1 }{\delta} \sum_{j=1}^{2K-K' } 
          \frac{ v_{K-j+1}^2 }{\mu_{K-j+1}^2} 
          \right) 
          +  
		c_2 \frac{1}{\mu_K} \sum_{i=1}^{L-K} \bar{T}_{\sigma(i)} 
        +
        c_3   \sum_{k=2}^{2K-K' }\frac{ K-k+1 }{\mu_{K-k+1}} \left[ \bar{T}_{\sigma(L-K+k)} - \bar{T}_{\sigma(L-K+k-1)} \right]
\end{align*}
where
\begin{align*}
	& \frac{1}{\mu_K} \sum_{i=1}^{L-K} \bar{T}_{\sigma(i)}  = 
	O\left(
	  \frac{ 1 }{ \mu_K }\sum_{i=1}^{L-K} \bar{\Delta}_{\sigma(i)}^{-2} \log \left[ \frac{ L}{\delta } \log \left( \frac{ L }{ \delta \bar{\Delta}_{ \sigma(i) }^{2} } \right)\right] 
	 \right), \\
%%%%%
	& \sum_{k=2}^{2K-K' }\frac{ K-k+1 }{\mu_{K-k+1}} \left[ \bar{T}_{\sigma(L-K+k)} - \bar{T}_{\sigma(L-K+k-1)} \right] \\
	& =
	\sum_{k=1}^{ K_0-1 }\bar{T}_{\sigma(L-K+k)}
          \left( \frac{  K-k+1  }{ \mu_{K-k+1} } - \frac{  K-k  }{ \mu_{K-k}}  \right)
        + 
         \left( \frac{  K -K_0+1 }{ \mu_{K-K_0+1} } -2 \right) \bar{T}_{\sigma(L-K+K_0 ) }
         +
         2 \bar{T}_{\sigma(L+K-K')}  \\
    & =
    c_4 \sum_{k=1}^{K_0-1 }
          \bar{\Delta}_{\sigma(L-K+k)}^{-2} \log \left[ \frac{ L}{\delta } \log \left( \frac{ L }{ \delta \bar{\Delta}_{ \sigma(L-K+k) }^{2} } \right)\right]
          \cdot
           \left( \frac{  K+1-k  }{\mu_{K+1-k}} - \frac{  K-k  }{\mu_{K-k}}  \right)
    \\ & \hspace{1em}
    + c_5 \left(  \frac{ K-K_0+1  }{ \mu_{K-K_0+1} } - 2\right)
          \bar{\Delta}_{\sigma( L-K+K_0 )}^{-2}
         \log \left[ \frac{ L}{\delta } \log \left( \frac{ L }{ \delta \bar{\Delta}_{ \sigma(L-K+K_0) }^{2} } \right)\right] 
    + c_6 \bar{\Delta}_{\sigma( L+K-K' )}^{-2} 
         \log \left[ \frac{ L}{\delta } \log \left( \frac{ L }{ \delta \bar{\Delta}_{ \sigma(L+K-K') }^{2} } \right)\right] .
\end{align*}

\subsection{Proof of Theorem~\ref{thm:lb_fix_conf_dep}}
\label{appdix:conf_low_bd_last_step}

Recall that $\bmO^\pi_t$ is a vector in $\{0, 1, \star\}^K$, where $0, 1, \star$ represents observing no click, observing a click and no observation respectively. For example, when $S^\pi_t = (2, 1, 5, 4)$ and $\bmO^\pi_t = (0, 0, 1, \star)$, items $2,1,5,4$ are listed in the displayed order; items $2,1$ are not clicked, item 5 is clicked, and the response to item 4 is not observed. By the definition of the cascading model, the outcome  $\bmO^\pi_t = (0, 0, 1, \star)$ is in general a (possibly emtpy) string of $0$s, followed by a $1$~(if the realized reward is $1$), and then followed by a possibly empty string of $\star$s.
Clearly, $S^{\pi, \ell}_t$, $\bmO^{\pi, \ell}_t$  are random variables with distribution depending on $w^{(\ell)}$ (hence these random variables distribute differently for different $\ell)$, albeit a possibly complicated dependence on $w^{(\ell)}$.

With the analysis in Section~\ref{sec:main_pf_lb}, according to Lemma~\ref{lemma:lb_KLdecomp}
and the definition of the instance $\ell$, one obtains for $i\in \{ 1,\ldots,K\}$ or $j \in \{K+1, \ldots, L\}$ respectively, 
\begin{align*}
	& \bbE[ T_\calT(i) ] \ge \frac{ \mathrm{KL}(1-\delta,\delta) }{ \mathrm{KL}\left( w (i), w (K+1)  \right) + \alpha },~~
	 \bbE[ T_\calT(j) ] \ge \frac{ \mathrm{KL}(1-\delta,\delta) }{ \mathrm{KL}\left( w (j), w (K)  \right) + \alpha }.
\end{align*}
Let $Y_t$ denote the number of observations of items at time step $t$. 
By revisiting the definition of $X_{k;t}$ in Section~\ref{sec:main_result_up}, we see that $X_{K;t}$ actually counts the observation of all pulled items at time step $t$. Hence, $Y_t \le X_{K;t}$.
Setting $\alpha\rightarrow 0$ and summing over the items yields a bound on the expected number of total observations 
$\bbE \left[\sum_{t=1}^\calT Y_t \right]  = \sum_{i=1}^L \bbE[ T_\calT(i) ] $. 
Meanwhile, an upper bound of $\bbE X_{K;t}$ as stated in Lemma~\ref{lemma:bound_obs} and tower property indicates  that
\begin{align*}
	& \bbE \left[\sum_{t=1}^\calT Y_t \right] 
	= \bbE\left[~ \bbE\left[ \sum_{t=1}^T Y_t ~|~ \calT=T \right]  ~\right]   
	\le \bbE\left[~ \bbE\left[ \sum_{t=1}^T \tilde{\mu}_K ~\Big|~ \calT=T \right]  ~\right] 
	= \bbE \left[~ \tilde{\mu}_K \cdot \calT ~\right]
	= \tilde{\mu}_K \cdot \bbE[\calT].
\end{align*}
Note that $\rmKL(x,1-x) \ge \log(1/2.4x)$ for any $x\in[0,1]$, we complete the proof of Theorem~\ref{thm:lb_fix_conf_dep}.

\subsection{Proof of Proposition~\ref{prop:fix_conf_up_bd_many_eps_opt} }
\label{pf:fix_conf_up_bd_many_eps_opt}

\propFixConfUpBdManyEpsOpt*

\begin{proof}
		Consider $K' \ge 2K-1$, i.e, $K'-K \ge K-1$.
According to Lemma~\ref{le:stop_eliminate_no}, there are at least $K'-K+1 \ge K$ items in the survival set $D_t$ before the algorithm terminates, so the algorithm pulls $K$ items from the surviving set $D_t$ at each time step. 
And for simplicity, we again write $\mu(k,w )$ as $\mu_k$, $v(k,w )$ as $v_k$, $\bar{T}_{i,\delta}$ as $\bar{T}_{i}$, $\rho(\delta)$ as $\rho$. 

Recall Lemma~\ref{lemma:conc_one_sub_gauss_apply_obs}, we
 set $\delta_0 = \delta/2$, $k=K$, $n=t_0'$, 
 $\rho'= - \sqrt{-2t_0' v_K^2 \log (\delta/2)}$.
  Then the total number of observations during $t_0'$ steps should be larger than $t_0' \mu_K + \rho'$ with probability at least $1-\delta/2$. And since the number of observations can be upper bounded, we consider
\begin{align*}
	t_0' \mu_K + \rho'	
	& \le \sum_{i=1}^{L-K'+K} \bar{T}_{\sigma(i)} + \sum_{i=L-K'+K+1}^L T_{t'_0}(\sigma(i)) 
	  \le \sum_{i=1}^{L-K'+K} \bar{T}_{\sigma(i)} + (K'-K) \left( \bar{T}_{\sigma(L-K'+K)} +1 \right) \\
	& = \sum_{i=1}^{L-K'+K-1} \bar{T}_{\sigma(i)} + (K'-K+1) \bar{T}_{\sigma(L-K'+K)} + (K'-K)
	  := \tilde{T}_0.
\end{align*} 

Lastly, with Lemma~\ref{lemma:conc_one_sub_gauss_apply_obs}, \ref{le:conc_event} and \ref{le:suff_obs_to_identif}, we obtain that with probability at least $1-\delta$, Algorithm~\ref{alg:bai_cas_conf} stops after at most
\begin{align*}
	\frac{2\tilde{T}_0}{\mu_K} 
	+  \frac{2v_K^2 }{\mu_K^2}\log \left(\frac{2}{\delta}\right)
	= N_2' + N_1'
\end{align*}
steps.

\end{proof}

\newpage
\section{Additional numerical results}
\label{appdix_exp}

\subsection{Order of pulled items}
\label{appdix_exp:order_pull_item}

    \begin{figure}[!h]
        \centering
        \begin{tabular}{c}
        \makebox[\textwidth]{$L=64, ~K=16, ~\delta=0.1, ~\epsilon=0$} \\
        \includegraphics[width = .9\textwidth]{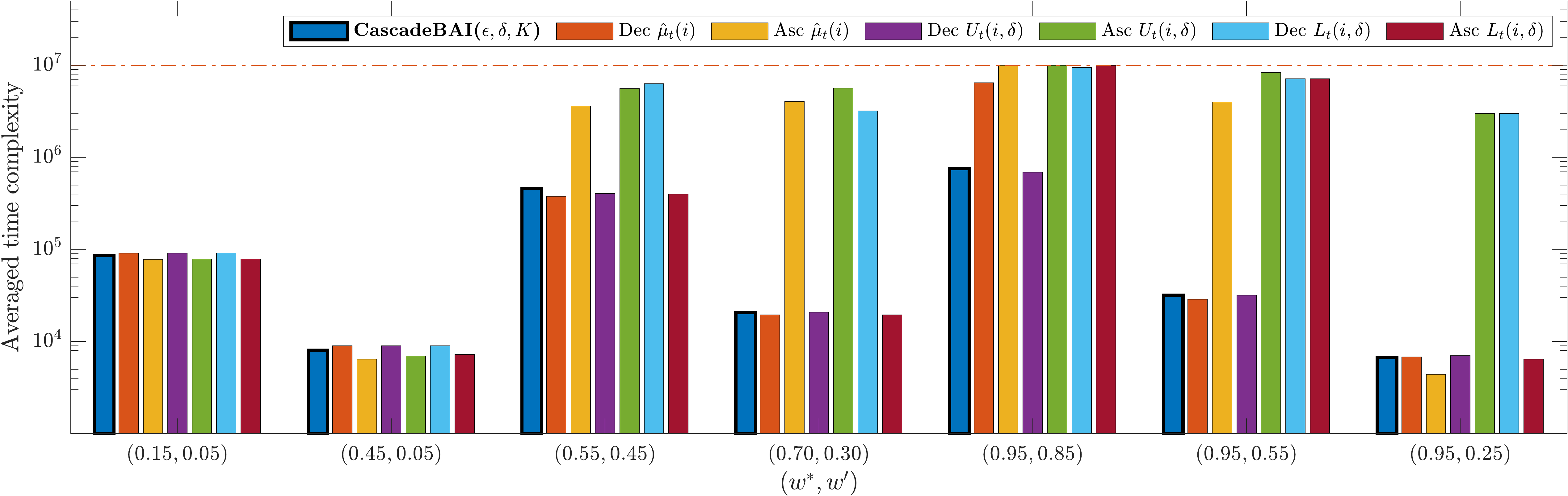} \\ 
        \makebox[\textwidth]{$L=64, ~K=8, ~\delta=0.1, ~\epsilon=0$} \\
        \includegraphics[width = .9\textwidth]{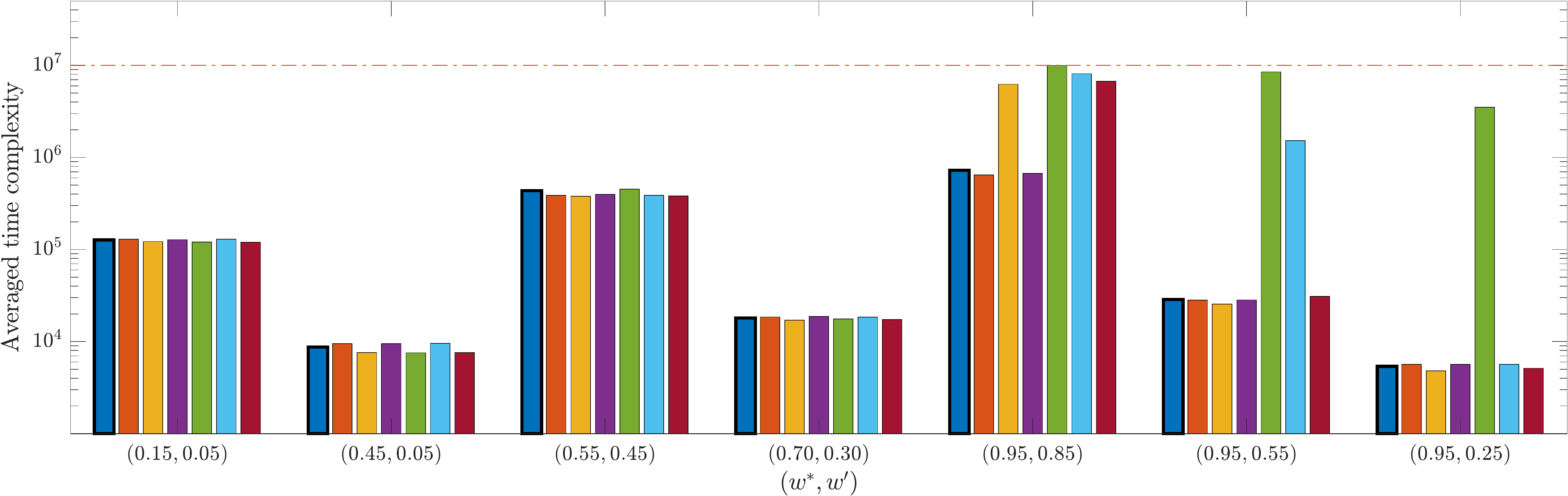} \\ 
        \makebox[\textwidth]{$L=128, ~K=16, ~\delta=0.1, ~\epsilon=0$} \\
        \includegraphics[width = .9\textwidth]{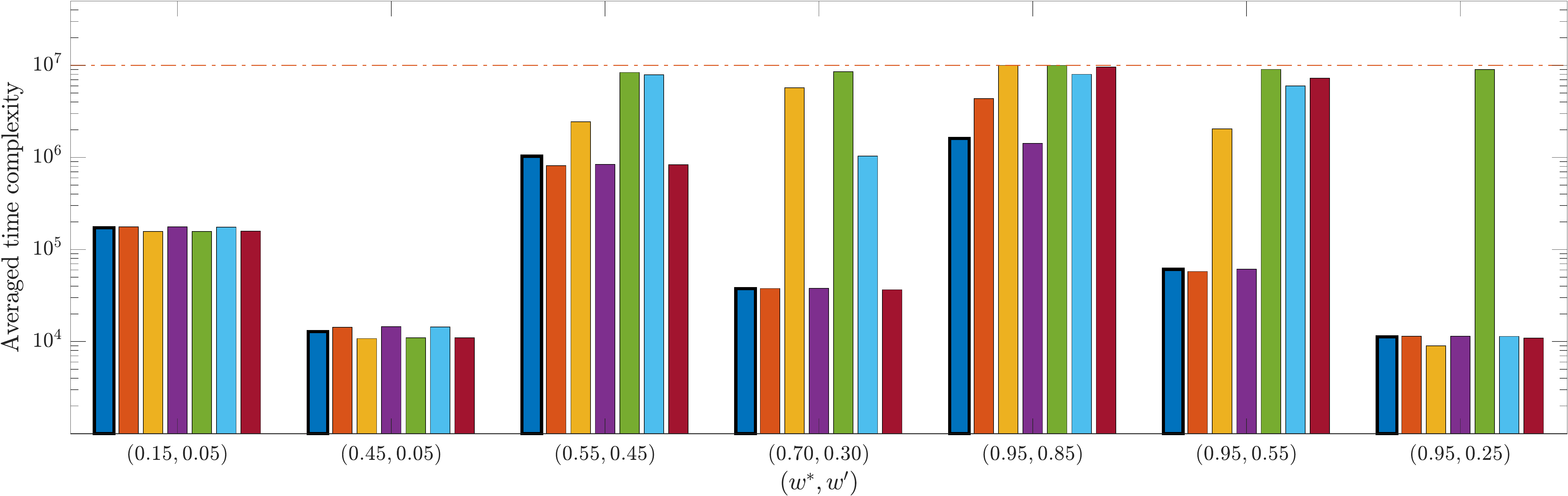} \\ 
        \makebox[\textwidth]{$L=128, ~K=8, ~\delta=0.1, ~\epsilon=0$} \\
        \includegraphics[width = .9\textwidth]{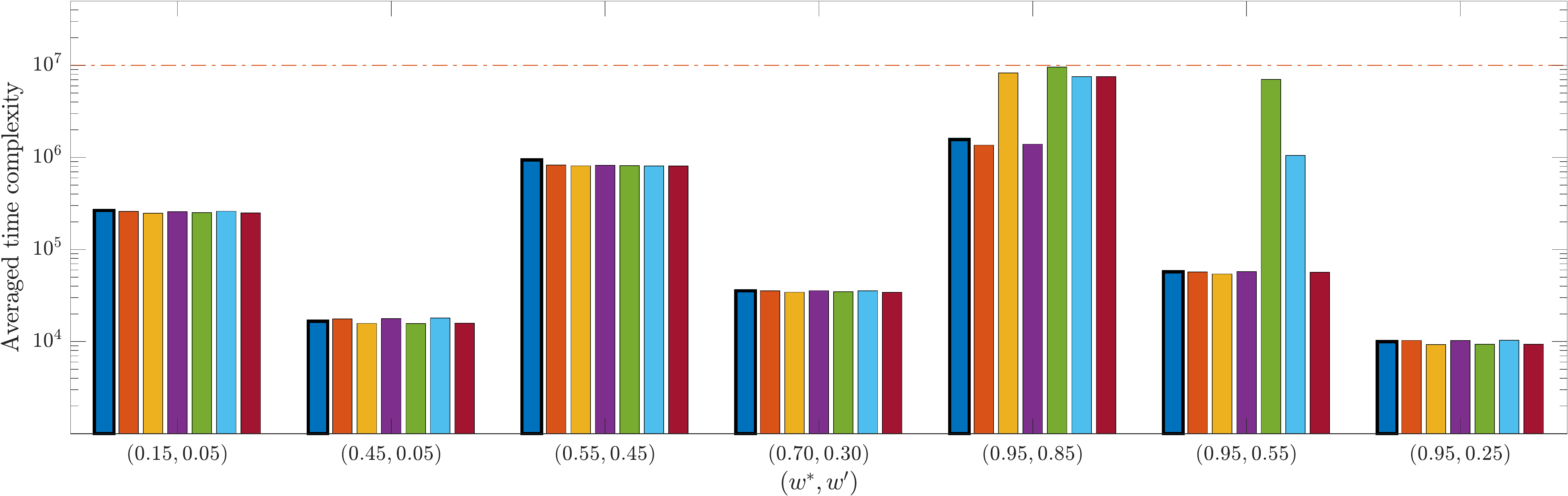} \\ 
       \end{tabular} \vspace{-1em}
    \end{figure}

    \begin{figure}[!h]
        \centering
        \begin{tabular}{c}        
        \makebox[\textwidth]{$L=64, ~K=16, ~\delta=0.1, ~\epsilon=0.05$} \\
        \includegraphics[width = .9\textwidth]{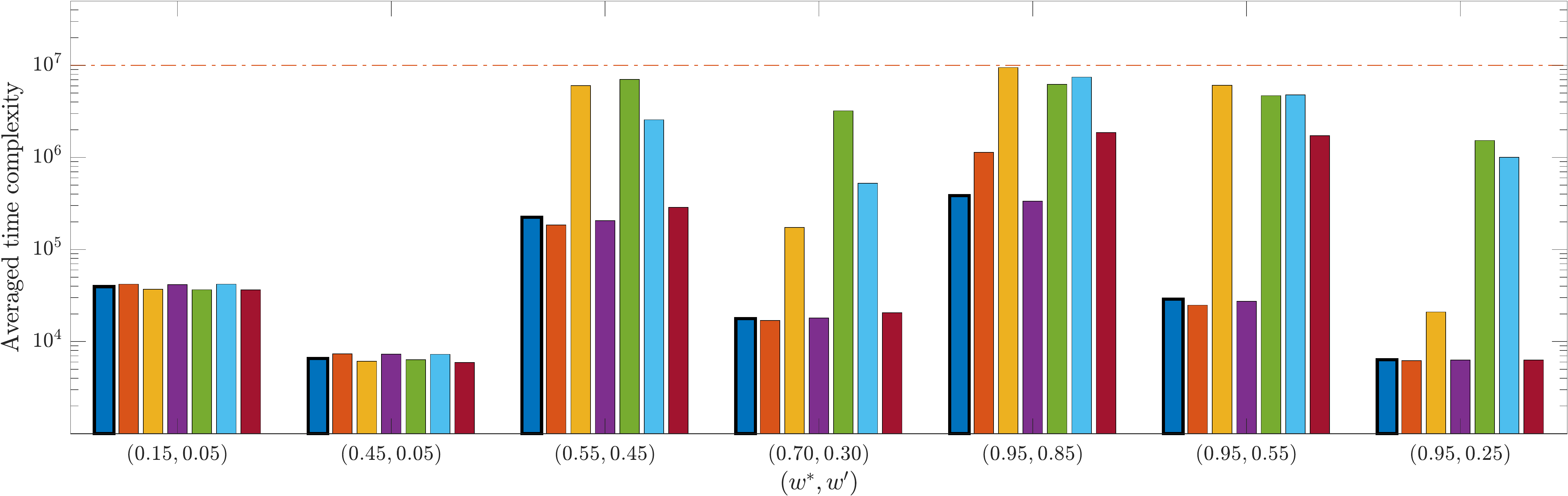} \\ 
        \makebox[\textwidth]{$L=64, ~K=8, ~\delta=0.1, ~\epsilon=0.05$} \\
        \includegraphics[width = .9\textwidth]{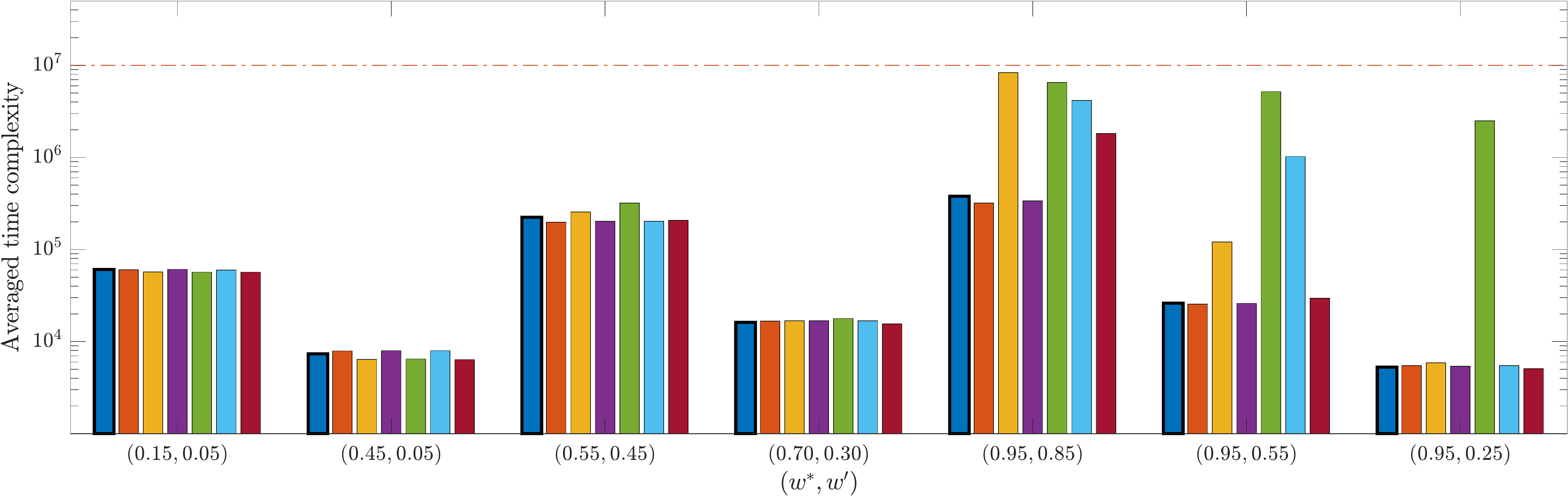} \\ 
        \makebox[\textwidth]{$L=128, ~K=16, ~\delta=0.1, ~\epsilon=0.05$} \\
        \includegraphics[width = .9\textwidth]{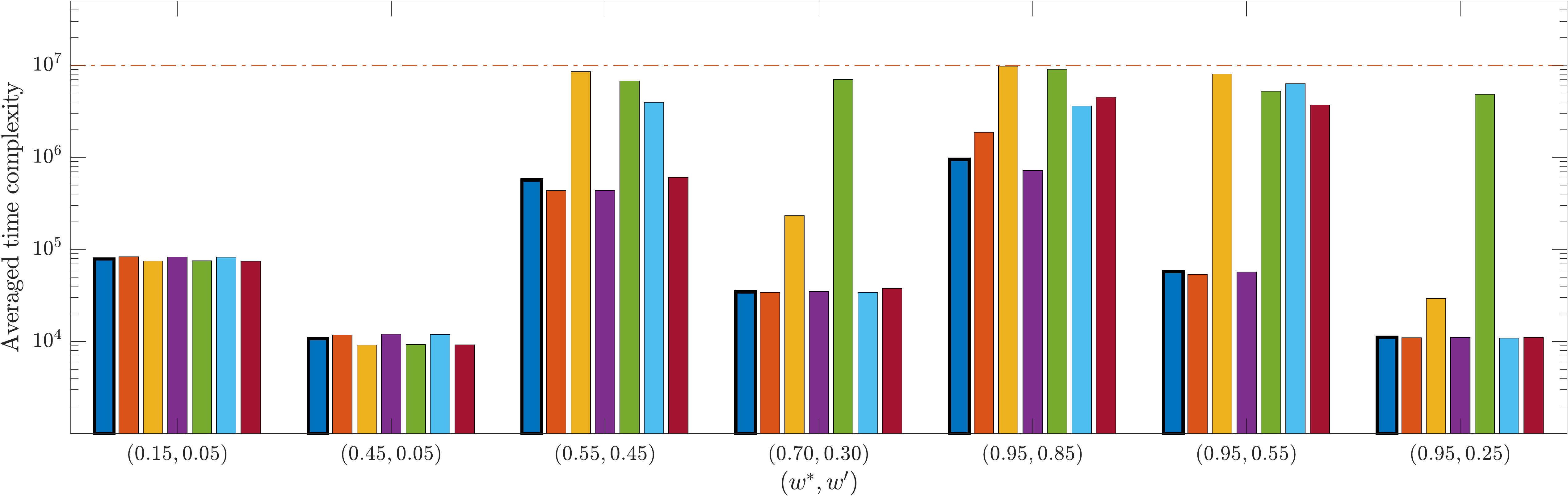} \\ 
        \makebox[\textwidth]{$L=128, ~K=8, ~\delta=0.1, ~\epsilon=0.05$} \\
        \includegraphics[width = .9\textwidth]{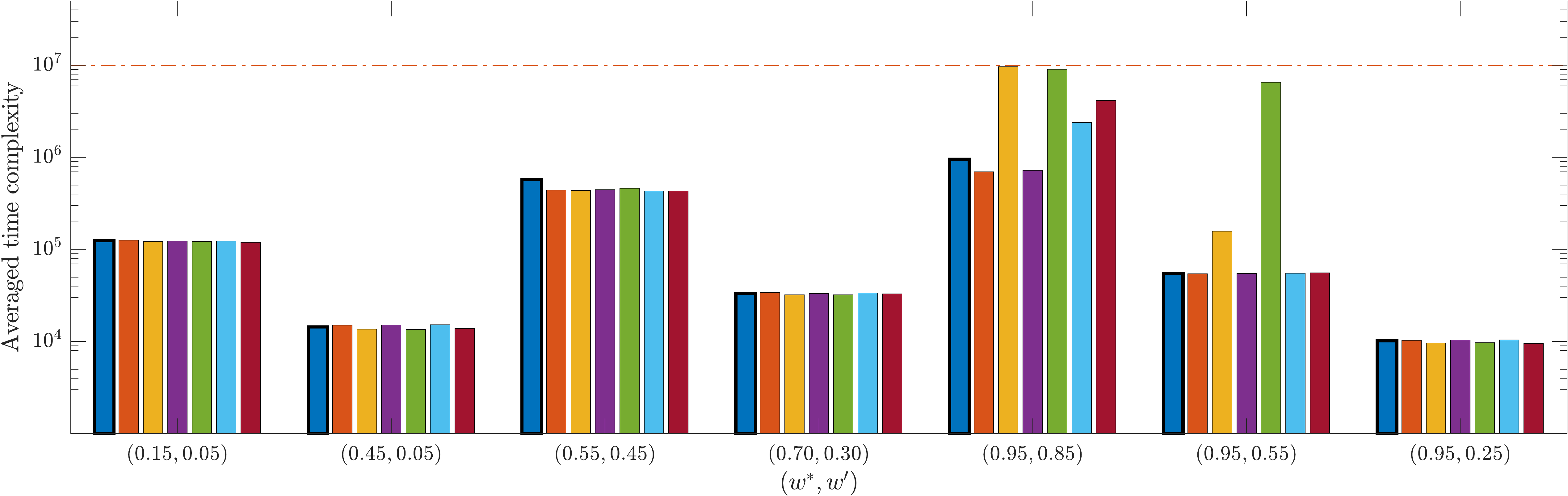} \\ 
       \end{tabular} 
    \end{figure}         

    \begin{figure}[!h]
        \centering
        \begin{tabular}{c}          
        
        \makebox[\textwidth]{$L=64, ~K=16, ~\delta=0.05, ~\epsilon=0$} \\
        \includegraphics[width = .9\textwidth]{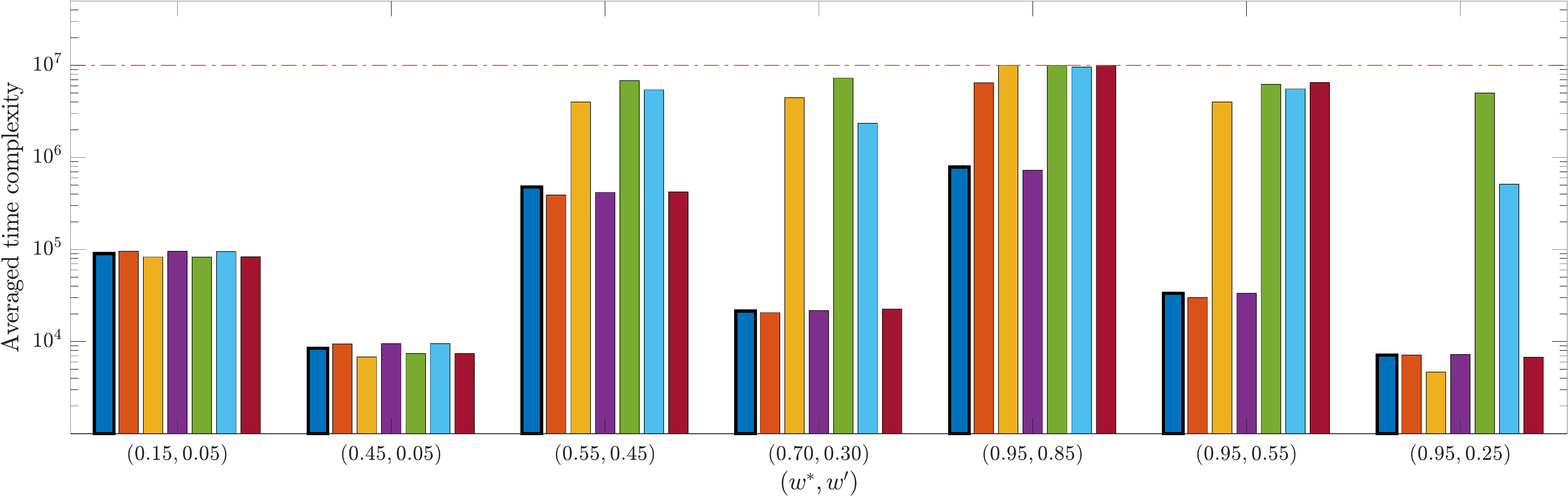} \\ 
        \makebox[\textwidth]{$L=64, ~K=8, ~\delta=0.05, ~\epsilon=0$} \\
        \includegraphics[width = .9\textwidth]{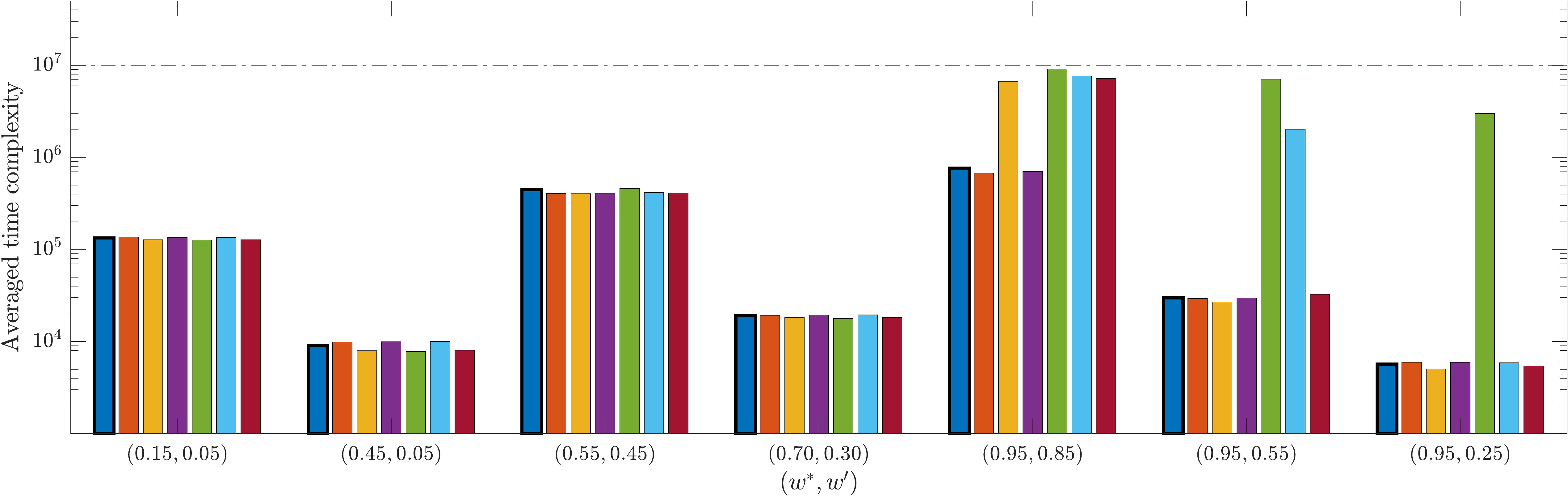} \\ 
        \makebox[\textwidth]{$L=128, ~K=16, ~\delta=0.05, ~\epsilon=0$} \\
        \includegraphics[width = .9\textwidth]{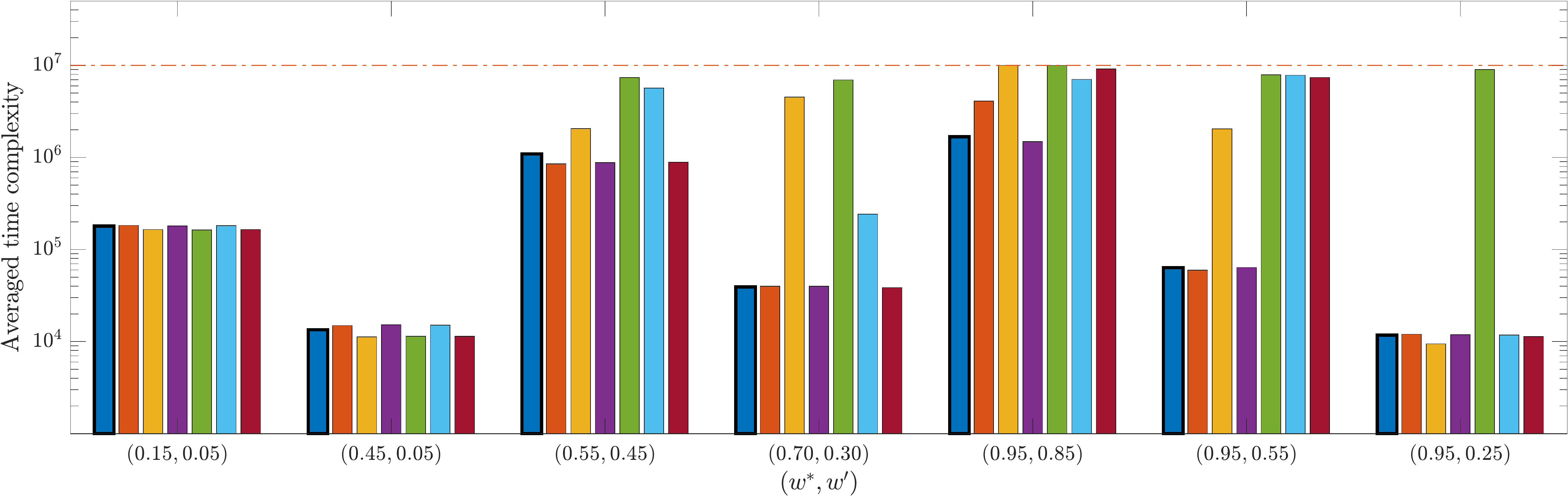} \\ 
        \makebox[\textwidth]{$L=128, ~K=8, ~\delta=0.05, ~\epsilon=0$} \\
        \includegraphics[width = .9\textwidth]{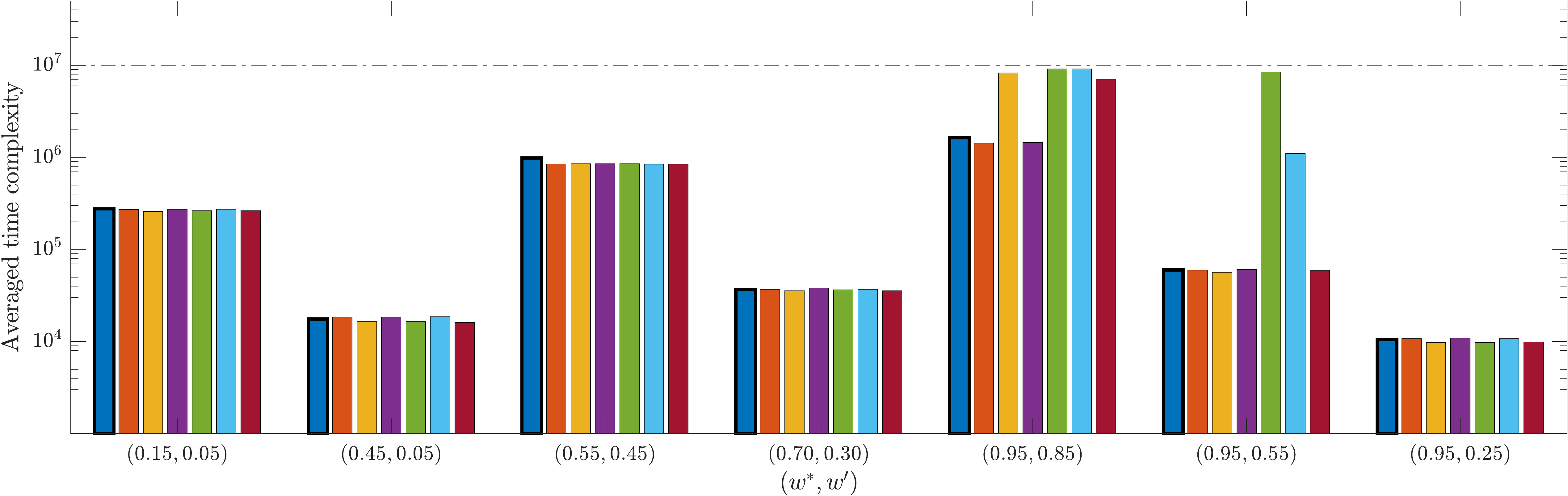} \\ 
       \end{tabular} 
    \end{figure}         

    \begin{figure}[!h]
        \centering
        \begin{tabular}{c}          
        \makebox[\textwidth]{$L=64, ~K=16, ~\delta=0.05, ~\epsilon=0.05$} \\
        \includegraphics[width = .9\textwidth]{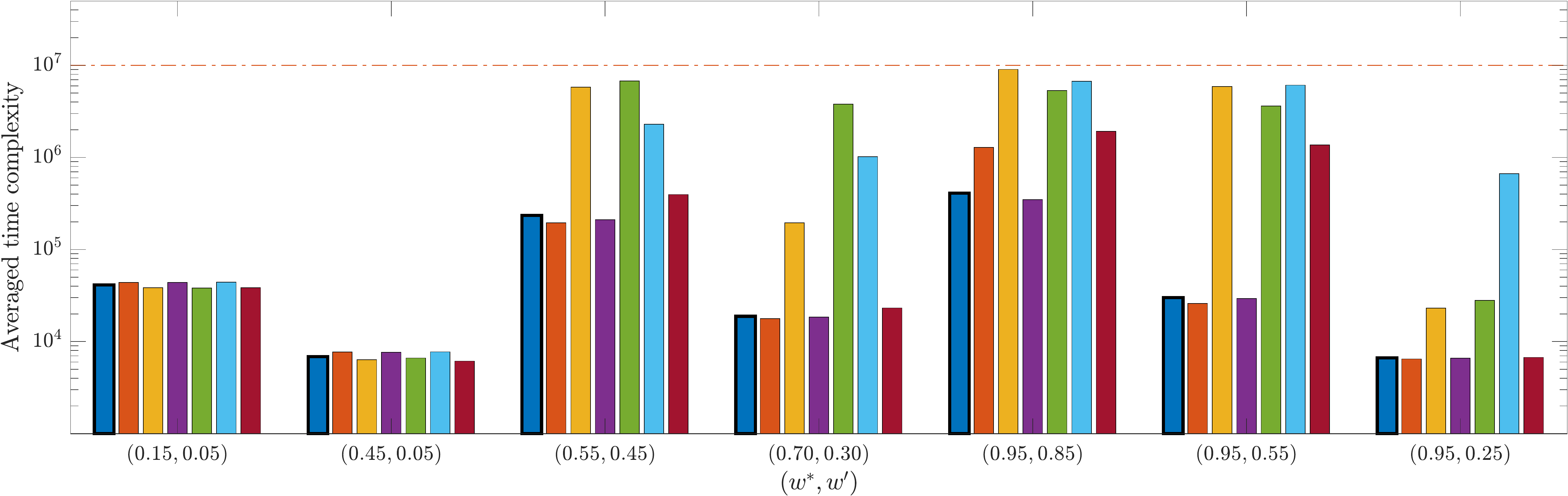} \\ 
        \makebox[\textwidth]{$L=64, ~K=8, ~\delta=0.05, ~\epsilon=0.05$} \\
        \includegraphics[width = .9\textwidth]{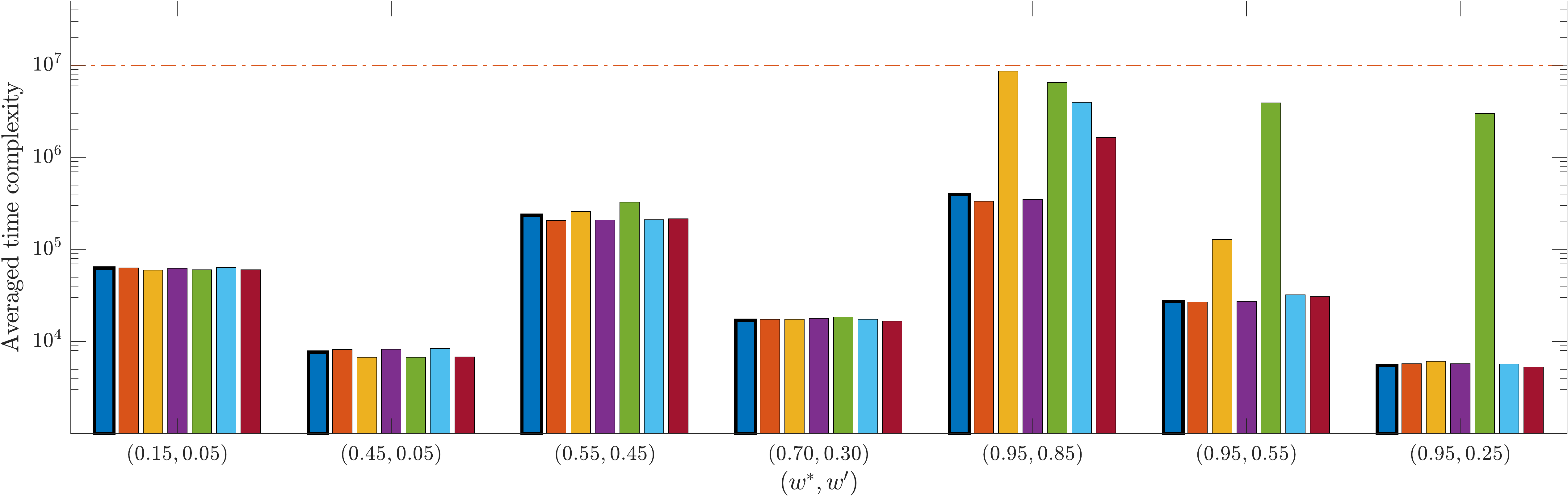} \\ 
        \makebox[\textwidth]{$L=128, ~K=16, ~\delta=0.05, ~\epsilon=0.05$} \\
        \includegraphics[width = .9\textwidth]{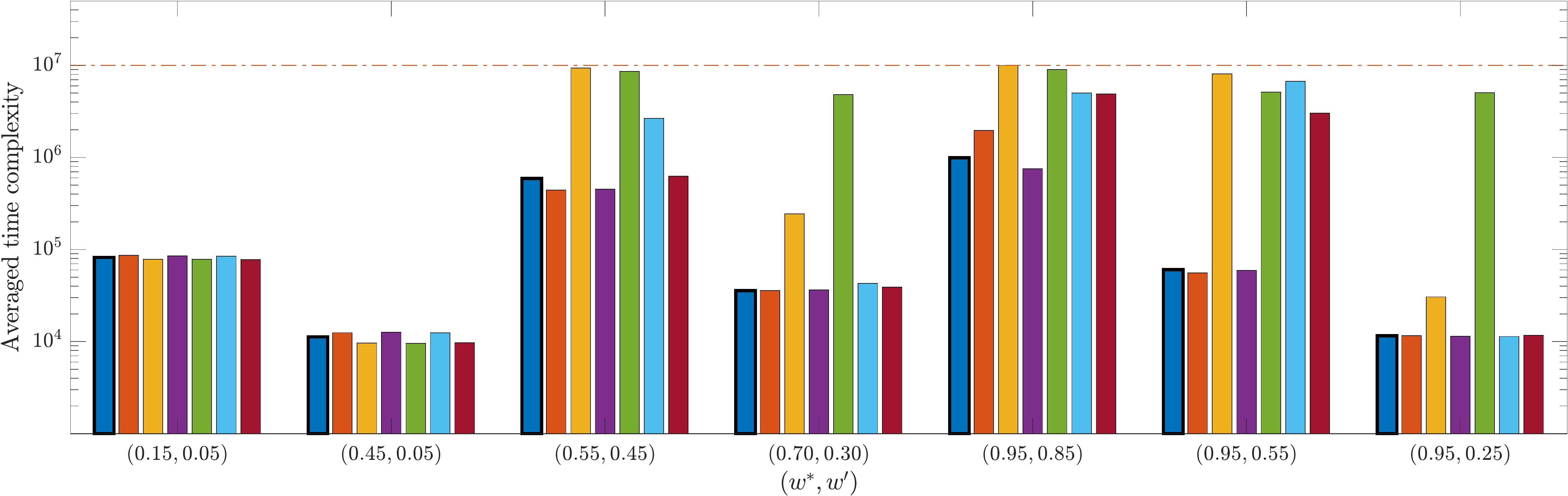} \\ 
        \makebox[\textwidth]{$L=128, ~K=8, ~\delta=0.05, ~\epsilon=0.05$} \\
        \includegraphics[width = .9\textwidth]{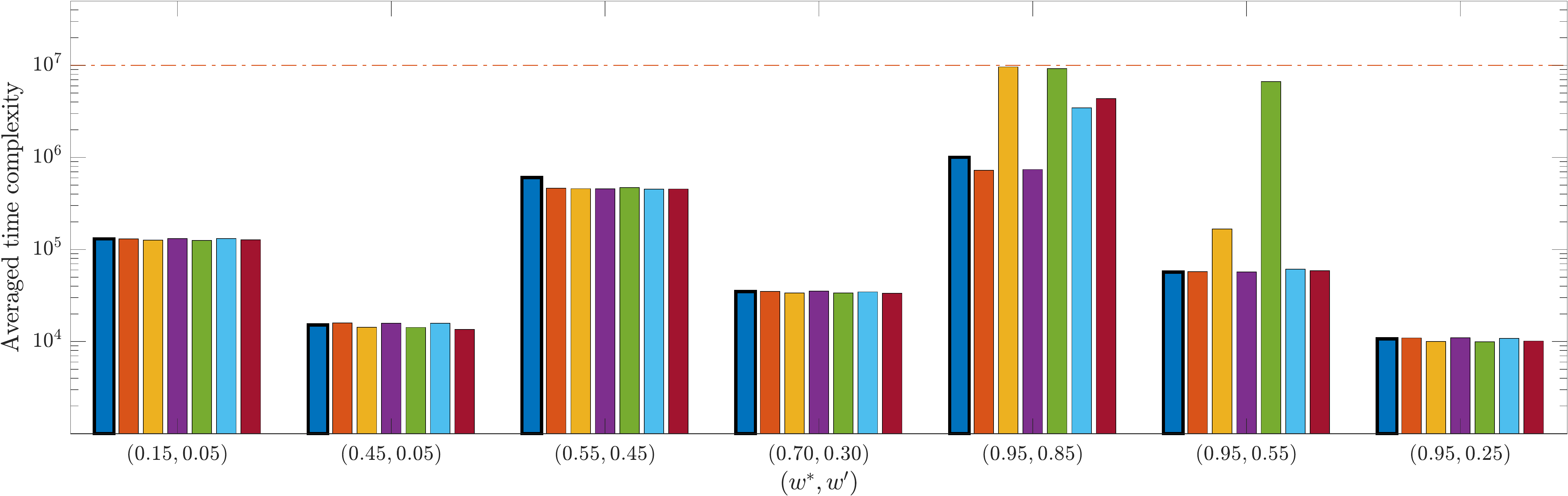}
       \end{tabular} 
		\caption{Average time complexity incurred by different sorting order of $S_t$: ascending order of $T_i(t)$~(Algorithm~\ref{alg:bai_cas_conf}), ascending/descending order of $\hat{\mu}_t(i)/ U_t(i) / L_t(i)$ 
		in the cascading bandits. 
       }\label{pic:compare_order_all} 
    \end{figure}

\clearpage    
After a large amount of observations, it is likely that the empirical mean $\hatw_t(i)$ approaches the true weight $w(i)$, and $w(i)$ lies between the confidence bounds $U_t(i,\delta)$ and $L_t(i,\delta)$ with high probability. 
Therefore, one may consider to sort $S_t$ in the descending or ascending order of $\hatw_t(i)$'s, $U_t(i,\delta)$'s or $L_t(i,\delta)$'s~(the difference to Algorithm~\ref{alg:bai_cas_conf} reveals in Line 5--9). 
Diving into the numerical results, we found an algorithm always manages to find an $\epsilon$-optimal arm provided that it is not terminated by the limit of $10^7$ steps.
Hence, we focus on the comparison of averaged stopping time.

In Figure~\ref{pic:compare_order_all}, we can see that sorting $S_t$ in the ascending order of $\hat{\mu}_t(i)$ or $U_t(i)$, especially the latter one, incurs an apparently larger averaged stopping time than other methods in most cases. Next, the descending order of $\hat{\mu}_t(i)$ does not work well in some cases. 
Thirdly, the ascending order of  $L_t(i)$ performs almost the same as our algorithm in most cases but there are several cases where it performs much worse and does not terminate even after $10^7$ iterations.
Lastly, the descending order of $U_t(i)$ works almost as well as Algorithm~\ref{alg:bai_cas_conf} empirically but is in lack of theoretical guarantee on time complexity. 
Meanwhile, the standard deviation of the stopping time of our algorithm is negligible comparing to the average value. For instance, in the left-most case of Figure~\ref{pic:compare_order}, the standard deviation is about $22318.54$ when the average is about $754140.65$.

\subsection{ Further empirical evidence }

\label{appdix_exp:plot_K_fit}

\begin{figure}[!h]
        \centering
        \begin{tabular}{c}
        \includegraphics[width = .45\textwidth]{K-L128-eps0-delta01---wSet1-Kfit-eps-converted-to.pdf} 
        \includegraphics[width = .45\textwidth]{K-L128-eps0-delta01---wSet3-Kfit-eps-converted-to.pdf}   \\
        \makebox[.45\textwidth]{(a) $w^* = \frac{1}{K}, w'=\frac{1}{K^2}$} 
        \makebox[.45\textwidth]{(b) $w^* = 1-\frac{1}{K^2}, w'=1-\frac{1}{K}$}
        \vspace{1.2em} \\
        
        \includegraphics[width = .45\textwidth]{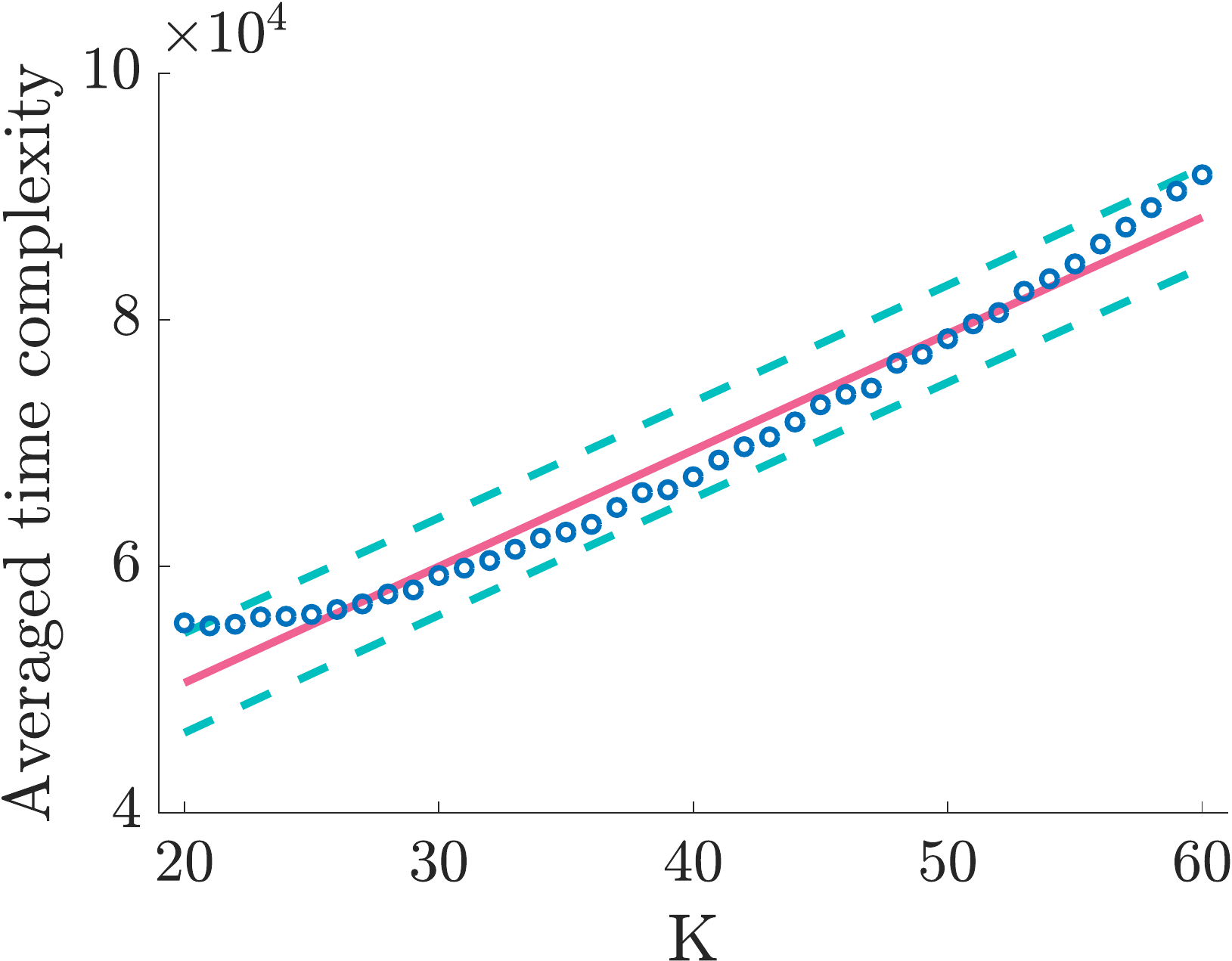} 
        \includegraphics[width = .45\textwidth]{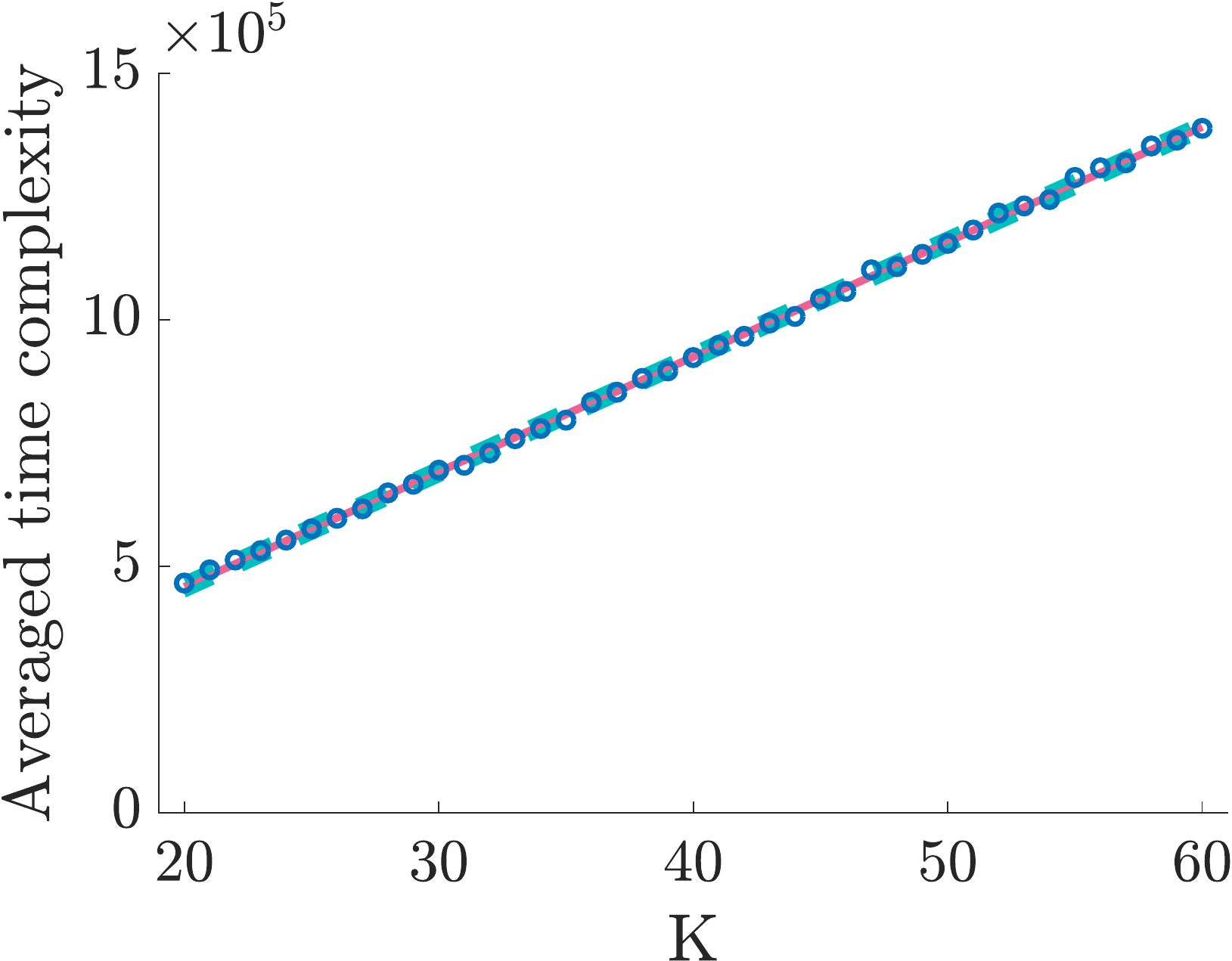}  \\
        \makebox[.45\textwidth]{(c) $w^* = \frac{1}{\sqrt{K}}, w'=\frac{1}{K }$} 
        \makebox[.45\textwidth]{(d) $w^* = 1-\frac{1}{K}, w'=1-\frac{1}{ \sqrt{K} }$}
        \vspace{1.2em} \\

        \includegraphics[width = .45\textwidth]{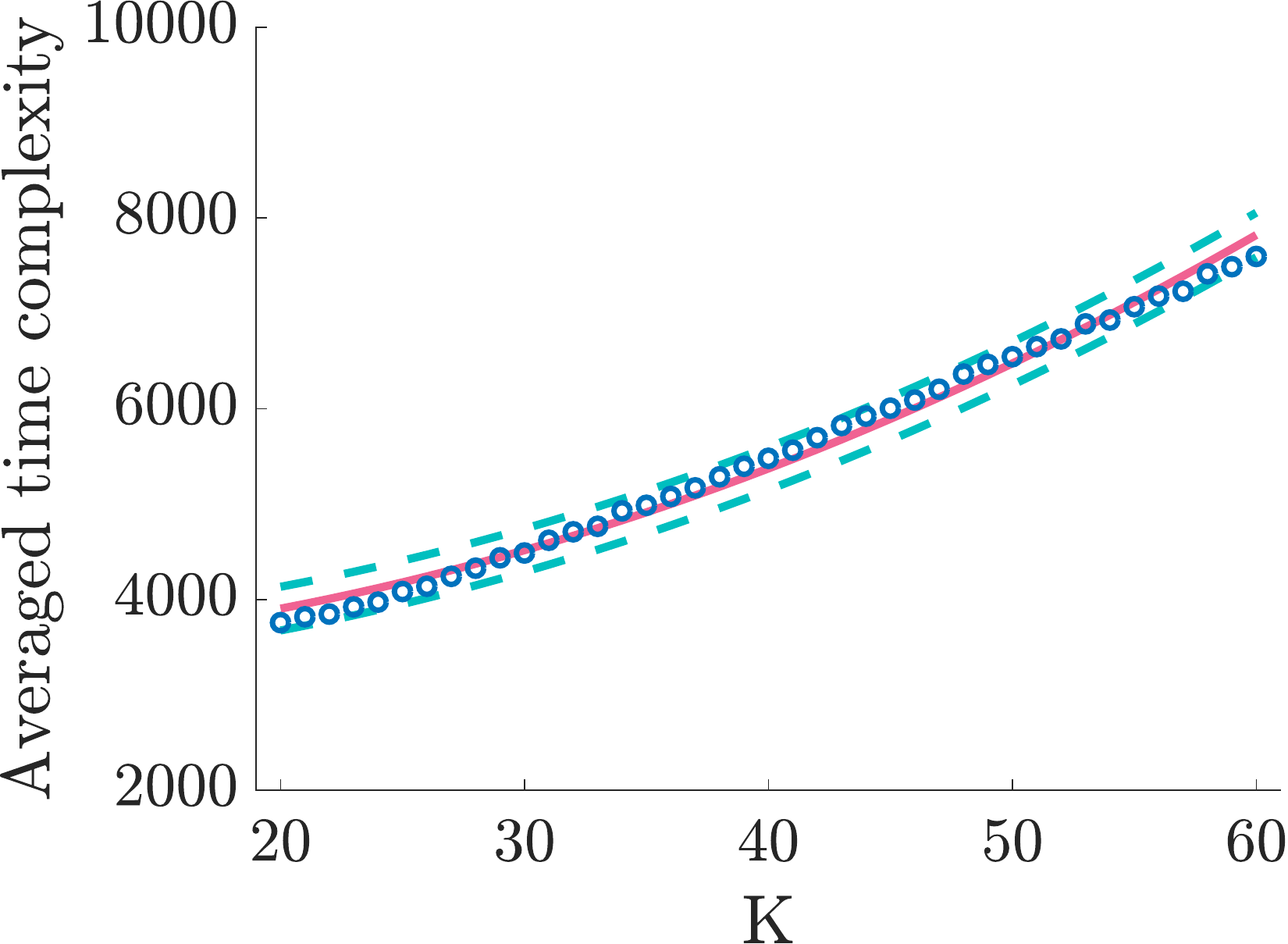}\\
        \makebox[.45\textwidth]{(e) $w^* = 1-\frac{1}{K}, w'= \frac{1}{ K }$}
        \vspace{1.2em} 
       \end{tabular} 
       \caption{Fit the averaged stopping time with functions of $K$ for each case in order. Fix $L=128$, $\delta=0.1$, $\epsilon=0$.  
       Blue dots are the averaged stopping time, red line is the fitted curve, and cyan dashed lines show the $95\%$ confidence interval.}  
\label{pic:K_fit}
\end{figure}

\begin{table*}[ht]
  \centering
  \caption{Fitted results of upper bounds on the stopping time $\calT$ of Algorithm~\ref{alg:bai_cas_conf} with 
  $\epsilon=0$~(Proposition~\ref{prop:conf_upbd_two_prob}).} 
	\begin{tabular}{cccccccc}
	  No. & $w^*$     & $w'$ & Fitting model & $c_1$ & $c_2$ &  $R^2$-statistic & $p$-value \\
	  \hline \\[-.8em]
	  1 & ${1}/{K}$ & ${1}/{K^2}$ &  $c_1 K + c_2$  
	  	& $23802.95$ 	& $67400.19$ 	& $0.9988$ 	& $1.39\times 10^{-58}$ 
        \\[.5em]
      2 & $1-{1}/{K^2 }$ & $ 1-{1}/{K }$ &  $c_1 K^2 + c_2$  
      	& $21615.50$ 	& $2007597.07$ 	& $0.9987$ 	& $1.29\times 10^{-57}$
        \\[.5em]
      3 & ${1}/{\sqrt{K} }$ & ${1}/{K }$ &  $c_1 K + c_2$ 
      	& $944.82$ 	& $31626.49$ 	& $0.9729$ 	& $3.58\times 10^{-32}$ 
        \\[.5em]
      4 & $1-{1}/{K }$ & $ 1-{1}/{\sqrt{K} }$ &  $c_1 K + c_2$   
      	& $23343.29$ 	& $8823.27$ 	& $0.9995$ 	& $3.00\times 10^{-65}$ 
        \\[.5em]
      5 & $1-{1}/{K }$ & $ {1}/{K }$ &  $c_1 K^2 + c_2$
      	& $1.22$ 	& $3414.56$ 	& $0.9917$ 	& $3.03\times 10^{-42}$ 
	    \end{tabular}%
	  \label{tab:K_set_approx_result}%
	\end{table*}

As shown in Table~\ref{tab:K_set_approx_result}, $p$-value is the probability that we reject the assumption of our fitting model versus a constant model~\citep{glantz1990primer}. Hence, the small $p$-values indicates that our fitting models are reasonable.
Next, all $c_1$'s are positive, implying all averaged stopping time grows with $K$, which corroborates our theoretical results.

\end{document}